%% file: main.tex
\documentclass{article}
\usepackage{color}
\usepackage[dvipsnames]{xcolor}
\usepackage{fullpage}

\usepackage{booktabs}
\usepackage{fullpage}

\usepackage{hyperref}
\definecolor{myblue}{rgb}{0,0.2,0.8}
\hypersetup{ %
pdftitle={},
pdfauthor={},
pdfsubject={},
pdfkeywords={},
pdfborder=0 0 0,
pdfpagemode=UseNone,
colorlinks=true,
linkcolor=myblue,
citecolor=myblue,
filecolor=myblue,
urlcolor=myblue,
}

\usepackage[utf8]{inputenc}
\usepackage{microtype}

\usepackage{graphicx}
\usepackage{natbib}
\usepackage{amsmath}
\usepackage{amssymb}
\usepackage{amsthm}
\usepackage{bbm}
\usepackage{color}
\usepackage{scalerel}
\usepackage[utf8]{inputenc} 
\usepackage[T1]{fontenc}    
\usepackage{url}            
\usepackage{booktabs}       
\usepackage{amsfonts}       
\usepackage{nicefrac}       
\usepackage{microtype}      
\usepackage[english]{babel}
\usepackage{algorithm} 
\usepackage{algorithmic}
\usepackage[algo2e]{algorithm2e} 
\usepackage{multicol}
\usepackage{wrapfig}
\usepackage{comment}

\newtheorem{thm}{Theorem}[section]
\newtheorem{lem}{Lemma}[section]
\newtheorem{cor}{Corollary}[section]
\newtheorem{prop}{Proposition}[section]
\newtheorem{asmp}{Assumption}[section]
\newtheorem{defn}{Definition}[section]

\newtheorem{condition}{Condition}[section]
\newtheorem{theorem}{Theorem}[section]

\newtheorem{lemma}[theorem]{Lemma}

\newcommand{\mat}{\mathbf}
\newcommand{\vect}{\mathbf}

\def\R{\mathbb{R}}

\def\vw{\mathbf{w}}
\def\va{\mathbf{a}}
\def\vv{\mathbf{v}}

\usepackage{array}
\newcolumntype{P}[1]{>{\centering\arraybackslash}p{#1}}

\def\vx{\mathbf{x}}
\def\vy{\mathbf{y}}
\def\vw{\mathbf{w}}
\def\vW{\mathbf{W}}

\def\v0{\mathbf{0}}
\def\va{\mathbf{a}}
\def\vu{\mathbf{u}}

\def\vv{\mathbf{v}}
\def\vX{\mathbf{X}}
\def\vH{\mathbf{H}}
\def\vU{\mathbf{U}}

\def\vV{\mathbf{V}}
\def\vSi{\mathbf{\Sigma}}

\def\vu{\mathbf{u}}

\def\vE{\mathbf{E}}

\def\norm#1{\|#1\|}

\newcommand{\abs}[1]{\left|#1\right|}
\newcommand{\expect}{\mathbb{E}}

\newcommand{\indict}{\mathbb{I}}

\newcommand{\relu}[1]{\sigma\left(#1\right)}

\title{AdaLoss: A computationally-efficient and provably convergent adaptive gradient method}
\author{%
  Xiaoxia Wu\thanks{work done while XW was a student at UT Austin.}\\
 University of Chicago\\
  \texttt{xiaoxiawu@uchicago.edu} \\
  \and
  Yuege Xie\\
 University of Texas at Austin\\
  \texttt{yuege@oden.utexas.edu} \\
  \and
Simon Du \\
  University of Washington \\
  \texttt{ssdu@cs.washington.edu} \\
  \and
 Rachel Ward \\
 University of Texas at Austin\\
  \texttt{rward@math.utexas.edu} \\
}
\date{}

\begin{document}

\maketitle
\begin{abstract}
\label{sec:abs}
\input{abs.tex}

\end{abstract}
\section{Introduction}
\label{sec:intro}
\input{intro_adagrad}
\section{AdaLoss Stepsize}
\input{motivation.tex}
\label{sec:gd}
\section{AdaLoss in Linear Regression}
\input{linear.tex}
\section{AdaLoss in Two-Layer Networks}
\label{sec:pre}
\input{setup.tex}
\input{gradient_descent.tex}
\label{sec:adaloss}
\input{adagrad.tex}

\section{Apply AdaLoss to Adam} \label{sec:AdamLoss}
\input{aaaexperiments.tex}

%

\section*{Broader  Impact}   
\label{sec:conclue}
 Our theoretical results make a step forward in explaining the linear convergence rate and zero training error using adaptive gradient methods as observed in neural network training in practice. Our new technique for developing a linear convergence proof (Theorem \ref{thm:linearRegression}  and Theorem \ref{thm:main_adagradloss}) might be used to improve the recent sub-linear convergence results of Adam-type methods \citep{kingma2014adam,chen2018on, ma2018adaptive,  zhou2018convergence,zou2018sufficient,defossez2020convergence}. Based on a theoretical understanding of the complexity bound of adaptive gradient methods and the relationship between loss and gradient, we proposed a  provably convergent adaptive gradient method (Adaloss). It is computationally-efficient and could potentially be a useful optimization method for large-scale data training. In particular, it can be applied in natural language processing and reinforcement learning domains where tuning hyper-parameters is very expensive, and thus making a potentially positive impact on society.



\subsubsection*{Acknowledgments}
\label{sec:ack}
This work was partially done while XW, SD and RW  were participating  the program on “Foundation of Data Science” (Fall 2018) in the Simons Institute for the Theory of
Computing at UC Berkeley. YX was partially funded by AFOSR MURI FA9550-19-1-0005 and NSF HDR-1934932.  We would also thank Facebook AI Research for
partial support of RW’s Research. RW was supported in part by AFOSR MURI FA9550-19-1-0005, NSF DMS 1952735, and NSF 2019844.
\bibliography{simonduref}
\bibliographystyle{plainnat}

\include{appendix-sum}

\end{document}

%% file: abs.tex

We propose a computationally-friendly adaptive learning rate schedule, ``AdaLoss", which directly uses the information of the loss function to adjust the stepsize in gradient descent methods. We prove that this schedule enjoys linear convergence in  linear regression.
Moreover, we provide a linear convergence guarantee over the non-convex regime, in the context of two-layer over-parameterized neural networks. If the width of the first-hidden layer in the two-layer networks is sufficiently large (polynomially), then AdaLoss converges robustly \emph{to the global minimum} in polynomial time. We numerically verify the theoretical results and extend the scope of the numerical experiments by considering applications in LSTM models for text clarification and policy gradients for control problems.

%% file: intro_adagrad.tex
Gradient-based methods are widely used in optimizing neural networks.
One crucial component in gradient methods is the learning rate (a.k.a. step size) hyper-parameter, which determines the convergence speed of the optimization procedure. An optimal learning rate can speed up the convergence but only up to a certain threshold value; once it exceeds this threshold value, the optimization algorithm may no longer converge.  This is by now well-understood for convex problems; excellent works on this topic include \citep{nash1991numerical}, \citep{bertsekas1999nonlinear}, \cite{nesterov2005smooth}, \cite{haykin2005cognitive}, \cite{bubeck2015convex}, and the recent review for large-scale stochastic optimization to \cite{bottou2018optimization}. 

While determining the optimal step size is theoretically important for identifying the optimal convergence rate, the optimal learning rate often depends on certain unknown parameters of the problem. For example, for a convex and $L$-smooth objective function, the optimal learning rate is $O(1/L)$ where $L$ is often unknown to practitioners. To solve this problem, adaptive methods~\cite{duchi2011adaptive,mcmahan2010adaptive} are proposed since they can change the learning rate on-the-fly according to gradient information received along the way. 
Though these methods often introduce additional hyper-parameters compared to gradient descent (GD) methods with well-tuned stepsizes, the adaptive methods are provably robust to their hyper-parameters in the sense that they still converge at suboptimal parameter specifications, but modulo (slightly) slower convergence rate \cite{levy2017online,ward2018adagrad}.
Hence, adaptive gradient methods are widely used by practitioners to save a large amount of human effort and computer power in manually tuning the hyper-parameters.

Among many variants of adaptive gradient methods, one that  requires a minimal amount of hyper-parameter tuning is \emph{AdaGrad-Norm} \cite{ward2018adagrad}, which has the following update
\begin{align}
      \vw_{j+1} & = \vw_{j} -{\eta}\nabla f_{\xi}(\vw_j)/{b_{j+1}} \quad\text{with}\quad  b_{j+1}^2 = b_j^2 + \|\nabla f_{\xi}(\vw_j)\|^2; \label{eq:adagrad-norm}
 \end{align}
above, $\vw_j$ is the target solution to the problem of minimizing the finite-sum objective function $F(\vw):= \sum_{i=1}^n f_{i}(\vw)$, and  $\nabla f_{\xi}(\vw_j)$ is the stochastic (sub)-gradient that depends on the random index $\xi \sim \text{Unif}\{1,2,\dots \}$  satisfying the conditional equality $\mathbb{E}_{\xi} [\nabla f_{\xi}(\vw_j)| \vw_j] = \nabla F(\vw_j)$.
However, computing the norm of the (sub)-gradient $\nabla f_{\xi}(\vw_j)\in \mathbb{R}^d$ in high dimensional space, particularly in settings which arise in training deep neural networks, is not at all practical. Inspired by a Lipschitz relationship between the objective function $F$ and its gradient for a certain class of objective functions (see details in Section 2), we propose the following scheme for $b_k$ which is significantly more  computationally tractable compared to computing the norm of the gradient: 
\begin{align*}
  \textbf{AdaLoss} \quad  b_{j+1}^2 &= b_j^2 + \alpha | f_{\xi}(\vw_j)-c|
 \end{align*}
 where $\alpha>0$ and $c$ are the tuning parameters.   With this update, we theoretically  show that AdaLoss converges with an upper bound that is tighter than AdaGrad-Norm under certain conditions.  

Theoretical investigations into adaptive gradient methods for optimizing neural networks are scarce. Existing analyses only deal with general (non)-convex and smooth functions, and thus, only concern convergence to first-order stationary points \cite{li2018convergence, chen2018on}. However, it is sensible to instead target global convergence guarantees for adaptive gradient methods in this setting in light of a series of recent breakthrough papers showing that (stochastic) GD can converge to the global minima of over-parameterized neural networks~\cite{du2018gradient,du2018deep,li2018learning,allen2018convergence,zou2018stochastic}. By adapting their analysis, we are able to answer the following open question:
\begin{center}
\vspace{-0.1cm}
\emph{What is the iteration complexity of adaptive gradient methods in over-parameterized networks?  } 
\vspace{-0.1cm}
\end{center}
In addition, we note that these papers require the step size to be sufficiently small to guarantee global convergence.
In practice, these optimization algorithms can use a much larger learning rate while still converging to the global minimum.
Thus, we make an effort to answer
\begin{center}
\vspace{-0.15cm}
\emph{What is the optimal stepsize in optimizing neural networks?}
\vspace{-0.15cm}
\end{center}

 \paragraph{Contributions.}

First, we study AdaGrad-Norm (\cite{ward2018adagrad}) in the linear regression setting and significantly improve the constants in the convergence bounds --  $\mathcal{O}\left( {L^2}/{ \epsilon }\right)$ (\cite{ward2018adagrad},\cite{xie2019linear})\footnote{The rate is for Case (2) of Theorem 3 in  \cite{xie2019linear} and of Theorem 2.2 in  \cite{ward2018adagrad}. $L$ is $\bar{\lambda}_1$ in Theorem \ref{thm:linearRegression}.}  in the deterministic gradient descent setting-- to a near-constant dependence  $\mathcal{O}\left(\log(L/\epsilon)\right)$ (Theorem \ref{thm:linearRegression}). 
 
Second, we develop an adaptive gradient method called \emph{AdaLoss} that can be viewed as a variant of the ``norm" version of AdaGrad but with better computational efficiency and easier implementation. We provide theoretical evidence that AdaLoss converges at the same rate as AdaGrad-Norm \emph{but with a better convergence constant} in the setting of linear regression (Corollary \ref{cor:linearRegression}).

Third, for an overparameterized two-layer neural network, we show the learning rate of GD can be improved to the rate of  $\mathcal{O}(1/\norm{\mat{H}^{\infty}})$ (Theorem \ref{thm:main_gd}) where $\mat{H}^{\infty}$ is a Gram matrix which only depends on the data.\footnote{Note that this upper bound is independent of the number of parameters.
	As a result, using this stepsize, we show GD enjoys a faster convergence rate.
	This choice of stepsize directly leads to an improved convergence rate compared to \cite{du2018gradient}. }
We further prove AdaLoss converges to the global minimum in polynomial time and does so robustly, in the sense that \emph{for any choice of hyper-parameters} used, our method is guaranteed to converge to the global minimum in polynomial time (Theorem \ref{thm:main_adagradloss}).  The choice of hyper-parameters only affects the rate but not the convergence.
 In particular, we provide explicit expressions for the polynomial dependencies in the parameters required to achieve global convergence.\footnote{Note that this section has greatly subsumes \cite{wu2019global}. However, Theorem \ref{thm:main_adagradloss} is a much improved version compared to Theorem 4.1 in \cite{wu2019global}. This is due to our new inspiration from Theorem \ref{thm:linearRegression}.} 
 
 We numerically verify our theorems in both linear regression and a two-layer neural network (Figure \ref{fig:adaloss1},  \ref{fig:adaloss2} and \ref{fig:adaloss}).  To demonstrate the easy implementation and extension of our algorithm for practical purposes, we perform experiments in a text classification example using LSTM models, as well as for a control problem using policy gradient methods (Section \ref{sec:AdamLoss}).
\paragraph{Related Work.}
Closely related work to ours can be divided into two categories as follows.

\textit{Adaptive Gradient Methods.} 
Adaptive Gradient (AdaGrad) Methods, introduced independently by \citep{duchi2011adaptive} and \citep{mcmahan2010adaptive}, are now widely used in practice for online learning due in part to their robustness to the choice of stepsize. The first convergence guarantees proved in \citep{duchi2011adaptive} were for the setting of online convex optimization  where the loss function may change from iteration to iteration.  Later convergence results for variants of AdaGrad were proved in \citep{levy2017online} and \citep{mukkamala2017variants} for offline convex and strongly convex settings.   In the non-convex and smooth setting, \citep{ward2018adagrad} and  \citep{li2018convergence} prove that the ``norm" version of AdaGrad converges to a stationary point at rate $O\left({1}/{\varepsilon^2}\right)$ for stochastic GD and at rate $O\left( {1}/{\varepsilon}\right)$ for batch GD.  Many modifications to AdaGrad have been proposed, namely, RMSprop \citep{hinton2012neural}, AdaDelta \citep{zeiler2012adadelta}, Adam \citep{kingma2014adam}, AdaFTRL\citep{orabona2015scale}, SGD-BB\citep{tan2016barzilai}, AcceleGrad \citep{OnlineLevy2018}, Yogi \citep{aheer2018adaptive1}, Padam \citep{chen2018closing}, to name a few.  More recently, accelerated adaptive gradient methods have also been proven to converge to stationary points~\citep{barakat2018convergence, chen2018on, ma2018adaptive,  zhou2018convergence,zou2018sufficient}. 

\textit{Global Convergence for Neural Networks.}
\label{sec:rel}
A series of papers showed that gradient-based methods  provably reduce to zero training error  for over-parameterized neural networks~\cite{du2018gradient,du2018deep,li2018learning,allen2018convergence,zou2018stochastic}. In this paper, we study the setting considered in \cite{du2018gradient} which showed that for learning rate $\eta = O(\lambda_{\min}(\mat{H}^{\infty})/n^2)$, GD finds an $\varepsilon$-suboptimal global minimum in $O\left(\log({1}/{\epsilon})\right)/({\eta\lambda_{\min}(\mat{H}^{\infty})})$ iterations for the two-layer over-parameterized ReLU-activated neural network.  As a by-product of the analysis in this paper, we show that the learning rate can be improved to $\eta = O(1/\norm{\mat{H}^{\infty}})$ which results in faster convergence.
We believe that the proof techniques developed in this paper can be extended to deep neural networks, following the recent works ~\citep{du2018deep,allen2018convergence,zou2018stochastic}.


\paragraph{Notation} Throughout, $\| \cdot \|$ denotes the Euclidean norm if it applies to a vector  and the maximum eigenvalue if it applies to a matrix. 
We use $N(\vect{0},\mat{I})$ to denote a standard Gaussian distribution where $\mat{I}$ denotes the identity matrix and $U(S)$ denotes the uniform distribution over a set $S$.
 We use $[n] := \{ 0,1, \dots, n\}$, and we write $[\vx]_i$ to denote the entry of the $i$-th dimension of the vector $\vx$.
\vspace{-0.2cm}

%% file: motivation.tex
Let  $\{Z_1,\dots, Z_n\}$ be empirical samples drawn uniformly from an unknown underlying distribution $\mathcal{S}$. Define  $f_{i}(\vw) = f(\vw, Z_i): \mathbb{R}^{d} \rightarrow \mathbb{R}, i = 1, 2, \ldots, n$.  Consider minimizing the empirical risk defined as finite sum of $f_i(\vw)$ over $i\in[n]$. The standard algorithm is stochastic gradient descent (SGD) with an appropriate step-size \cite{bottou2018optimization}. Stepsize tuning for optimization problems, including  training neural networks, is generally challenging because the convergence of the algorithm is very sensitive to the stepsize: too small values of the stepsize mean slow progress while too large values lead to the divergence of the algorithm. 


To find a suitable learning rate schedule, one could use the information on past and present gradient norms as described in equation \eqref{eq:adagrad-norm}, and the convergence rate for SGD is $\mathcal{O}\left( {1}/{\varepsilon^2}\right)$, the same order as for well-tuned stepsize \cite{levy2017online,li2018convergence,ward2018adagrad}. However, in high dimensional statistics, particularly in the widespread application of deep neural networks, computing the norm of the  (sub)-gradient $\nabla f_{i}(\vw_j)\in \mathbb{R}^d$ for $i\in[n]$ at every iteration $j$ is impractical. To tackle the problem, we recall the popular setting of linear regression and two-layer network regression  \cite{du2018deep} where assuming at optimal $\nabla f_{i}^*=0$, 
$$ \|\nabla f_{i}(\vw)\|^2 =  \|\nabla f_{i}(\vw) -\nabla f_{i}^* \|^2 \leq C| f_{i}(\vw) -f_{i}^*|.$$
The norm of the gradient is bounded by the difference between $f_{i}(\vw_j)$ and $f_{i}^* $. The optimal value $ f_{i}^* $ is a fixed number, which could possibly be known as prior or estimated  under some conditions. For instance, for an over-determined linear regression problem or over-parameterized neural networks, we know that $f_i^*=0$. For the sake of  the generality  of our proposed algorithm, we replace $f^*$ with a constant $c$.
Based on the above observation, we propose the update  in Algorithm 1.

Our focus is $b_{k+1}$, a parameter that is changing at every iteration according to the loss value of previous computational outputs. 
There are four positive hyper-parameters, $b_0, \eta, \alpha, c$, in  the algorithm. 
$\eta$ is for ensuring homogeneity and that the units match. 
$b_0$ is the initialization of a monotonically increasing sequence $\{b_k\}_{k=1}^{\infty}$. 
The parameter $\alpha$ is to control the rate of updating $\{b_k\}_{k=1}^\infty$ and the constant $c$  is a surrogate for the ground truth value $f^*$ ($c=0$ if $f^*=0$). 
\begin{algorithm}[H]\label{alg:adaa}
  \caption{AdaLoss Algorithm}
\begin{algorithmic}[1]
  \STATE {\bfseries Input:}   Initialize $\vw_0 \in \mathbb{R}^d, b_{0}>0, c>0, j \leftarrow 0$,  and the total iterations $T$.
	\FOR{$j=1,2,3, \ldots T$}
    \STATE Generate a random index $\xi_{j}$ 
    \STATE $b_{j+1}^2 \leftarrow  b_{j}^2 +  \alpha| f_{\xi_{j}}(\vw_{j} )-c|$ 
 	 \STATE   $ \vw_{j+1} \leftarrow \vw_{j} - \frac{\eta }{b_{j+1}}  \nabla f_{\xi_{j}}(\vw_{j})$  
     \ENDFOR
\end{algorithmic}
  \end{algorithm} 

The algorithm makes a significant improvement in \emph{computational efficiency} by using the direct feedback of the (stochastic) loss. 
For the above algorithm, $\xi_j \sim \text{Unif}\{1,2,\dots, n\}$  satisfies the conditional equality $\mathbb{E}_{\xi_j} [\nabla f_{\xi_j}(\vw_j)| \vw_j] = \nabla F(\vw_j)$.  As a nod to the use of the information of the stochastic loss for the stepsize schedule, we call this method \textbf{adaptive loss} (AdaLoss).  In the following sections, we present our analysis of this algorithm on linear regression and two-layer over-parameterized neural networks.

%% file: linear.tex
  Consider the linear regression:
\begin{align}
\min_{\vw \in \mathbb{R}^{d} } \frac{1}{2}\|\vX\vw -\vy\|^2, \quad \vy\in \mathbb{R}^{n}   \text{ and }\vX\in \mathbb{R}^{n\times d}  \label{eq:linear}
\vspace{-0.7cm}
\end{align}  Suppose the data matrix $\vX^\top\vX$ a 
positive definite matrix with the smallest singular value  $\bar{\lambda}_0>0$ and the largest singular value  $\bar{\lambda}_1>0$. Denote $\vV$ the unitary matrix from the singular value decomposition of $\vX^\top \vX= \vV\Sigma \vV^T$. 
Suppose we have the optimal solution  $\vX\vw^*=\vy$.
  The recent work of \cite{xie2019linear} implies that the convergence rate using the adaptive stepsize update in \eqref{eq:adagrad-norm} enjoys linear convergence. However, the linear convergence is under the condition that the effective learning rate ${2\eta}/{b_0}$ is less than the critical threshold ${1}/{\bar{\lambda}_1}$ (i.e.,$b_0\geq  {\bar{\eta\lambda}_1}/{2}$). If we initialize the effective learning rate larger than the threshold, the algorithm falls back to a sub-linear convergence rate with an order $\mathcal{O}\left( {\bar{\lambda}_1}/{\varepsilon}\right)$. Suspecting that this might be due to an artifact of the proof, we here tighten the bound that admits the linear convergence $\mathcal{O}\left(\log \left({1}/{\varepsilon}\right) \right)$ for any $b_0$ (Theorem \ref{thm:linearRegression}).



\begin{thm}
\label{thm:linearRegression}\textbf{ (Improved AdaGrad-Norm Convergence)}
Consider  the problem \eqref{eq:linear}  and
\begin{align} \label{eq:adalinear}
\vw_{t+1} =  \vw_t-\left({\eta}/{b_{t+1}}\right)\vX^T\left(\vX\vw_t-\vy\right) \quad \text{with}\quad b^2_{t+1}= b^2_{t}+\|\vX^T\left(\vX\vw_t-\vy\right)\|^2
\end{align}
We have $\| \vw_{T}-\vw^*\|^2\leq \epsilon$ for \footnote{
 			$\widetilde{\mathcal{O}}$  hide logarithmic terms.	}
 \begin{small}
\begin{align}
T=  \widetilde{\mathcal{O}}\left( \left(\frac{\|\vX(\vw_0-\vw^*)\|^2}{\eta\bar{\lambda}_1}+\frac{\eta}{s_0^2}\log \left(\frac{ \left({\eta \bar{\lambda}_1} \right)^2 - 4(b_0)^2 }{\eta^2  \bar{\lambda}_1^2 }\right) \mathbbm{1}_{\{2b_0\leq \eta \bar{\lambda}_1\}} \right)\frac{\bar{\lambda}_1\log\left({1}/{\epsilon}\right)}{\eta\bar{\lambda}_0} \right). \label{eq:b1}    
\end{align}
 \end{small}
where $s_0:= [\vV^\top \vw_{0}-\vV^\top {\vw}^{*}]_1$. Here $[\cdot]_1$ corresponds to  the dimension scaled by the largest singular value of $\vX^\top \vX$ (see (\ref{eq: scale}) in appendix), i.e. $\bar{\lambda}_1$.
\end{thm}

We state the explicit complexity $T$ in Theorem \ref{thm:linearRegression1} and the proof is in Section \ref{sec:linear-proof}.
Our theorem significantly improves the sub-linear convergence rate when $b_0\leq  {\bar{\eta\lambda}_1}/{2}$ compared to \cite{xie2019linear} and \cite{ward2018adagrad}. The  bottleneck in their theorems for small $b_0$ is that they assume the dynamics $b_t$ updated by the gradient $\|\vX^T\left(\vX\vw_t-y\right)\|^2\approx \varepsilon$ for all $j= 0,1,\ldots ,$ which results in taking as many  iterations as $N\approx ((\eta\bar{\lambda}_1)^2-4b_0^2)/{\epsilon}$ in order to get $b_N\geq {\bar{\eta\lambda}_1}/{2}$.
Instead, we  explicitly characterize  $\{b_t\}_{t\geq0}$.  That is, for each dimension $i\in [d]$,
\begin{align}
s_{t+1}^{(i)} & :=\left([\vV^\top \vw_{t+1}]_i-[\vV^\top {\vw}^{*}]_i\right)^2 \nonumber =\left( 1 - {\eta\bar{\lambda}_i}/{b_{t+1}} \right)^2([\vV^\top {\vw}_{t}]_i -[\vV^\top {\vw}^{*}]_i  )^2.
\end{align}
If $b_0\leq {\eta\bar{\lambda}_i}/{2}$, each $i$-th sequence $\{s_{t}^{(i)}\}_{t=0}^k$ is monotone increasing  up to  $b_{k} \leq \frac{\eta\bar{\lambda}_i}{2}$, thereby taking significantly fewer iterations, \emph{independent of the prescribed accuracy $\varepsilon$}, for $b_t$ to reach the critical value $\eta \bar{\lambda}_1/2$ as describe in the following lemma.

\begin{lem} \textbf{ (Exponential Increase for  $2b_0 < {\eta\bar{\lambda}_1}$)}\label{lem:increase_sub}\\
 Suppose we start with small initialization: $0<{b}_0< \bar{\lambda}_1/2$. Then there exists the first index ${{N}}$ such that ${b}_{{{N}}+1}\geq \bar{\lambda}_1/2$ and ${b}_{{N}}< \bar{\lambda}_1/2$, and  $N$  satisfies
 \begin{small}
\begin{align*}
&\textbf{(AdaGrad-Norm) }  N \leq \log  \left(  1+
\frac{ \left({\eta \bar{\lambda}_1} \right)^2 - 4(b_0)^2 }{\eta^2 \textcolor{red}{\bar{\lambda}_1^2} }
\right) / \log \left (1+   \frac{ 4}{ \eta^2 } \left([\vV^T\vw_{0}]_1 -[\vV^T\vw^{*}]_1 \right)^2 \right) +1 \\
&\textbf{(AdaLoss) } \widetilde{N}\leq  \log  \left(  1+
\frac{ \left({\eta \bar{\lambda}_1} \right)^2 - 4b_0^2 }{\eta^2 \textcolor{red}{\bar{\lambda}_1} }
\right) / \log \left (1+   \frac{ 4}{ \eta^2 \textcolor{red}{\bar{\lambda}_1} } \left([\vV^T\vw_{0}]_1 -[\vV^T\vw^{*}]_1 \right)^2 \right) +1
\end{align*}
 \end{small}
\end{lem}

Suppose $\bar{\lambda}_1>1$. For AdaLoss, we see that when the initialization $2b_0\leq \eta \bar{\lambda}_1$,  $b_t$ updated by AdaLoss is more likely to take more iterations than AdaGrad-Norm to reach a value greater than   $\eta \bar{\lambda}_1$ (see the red part in Lemma \ref{lem:increase_sub}). 
Furthermore, a more interesting finding is that AdaLoss's upper bound is smaller than AdaGrad-Norm's if  $\|\vX\left(\vw_t-\vw^*\right)\|^2\geq \|\vw_t-\vw^*\|^2$ (see Lemma \ref{lem:b-max}).  Thus, the upper bound of AdaLoss could be potentially tighter than AdaGrad-Norm when $2b_0\geq \eta \bar{\lambda}_1$, but possibly looser than  AdaGrad-Norm when $2b_0\leq \eta \bar{\lambda}_1$. To see this, we follow the same process and have the convergence of AdaLoss stated in Corollary \ref{cor:linearRegression}.
\begin{cor}
\label{cor:linearRegression}\textbf{(AdaLoss Convergence)}
Consider  the same setting as Theorem \ref{thm:linearRegression} but with the $b_t$ updated by:
$b^2_{t+1}= b^2_{t}+\|\vX\vw_t-\vy\|^2$. 
We have $\| \vw_{T}-\vw^*\|^2\leq \epsilon$ for
\begin{small}
\begin{align}
T= \widetilde{\mathcal{O}}\left(  \left(\frac{\|(\vw_0-\vw^*)\|^2}{\eta \bar{\lambda}_1}+\frac{\eta}{s_0^2}\log \left(\frac{ \left({\eta \bar{\lambda}_1} \right)^2 - 4(b_0)^2 }{\eta^2  \bar{\lambda}_1 }\right) \mathbbm{1}_{\{2b_0\leq \eta \bar{\lambda}_1\}} \right)\frac{\bar{\lambda}_1\log\left({1}/{\epsilon}\right)}{\eta \bar{\lambda}_0} \right).\label{eq:b2}
\end{align}
\end{small}
\end{cor}

 For the explicit form of $T$, see Corollary \ref{cor:linearRegression2} in the appendix. Suppose the first term in the bounds of \eqref{eq:b1} and \eqref{eq:b2} takes the lead and $\lambda_1> 1$, AdaLoss has a tighter upper bound than AdaGrad-Norm, unless $\bar{\lambda}_1\leq 1$. Hence, our proposed computationally-efficient method AdaLoss is a preferable choice in practice since $\bar{\lambda}_1$ is usually not available.\footnote{Although one cannot argue that an algorithm is better than another by comparing their worse-case upper bounds, it might give some implication of the overall performance of the two algorithms considering the derivation of their upper bounds is the same.}
 
 \begin{figure}[ht]
\begin{minipage}[c]{0.5\textwidth}
 \centering
    \includegraphics[width=1.\linewidth]{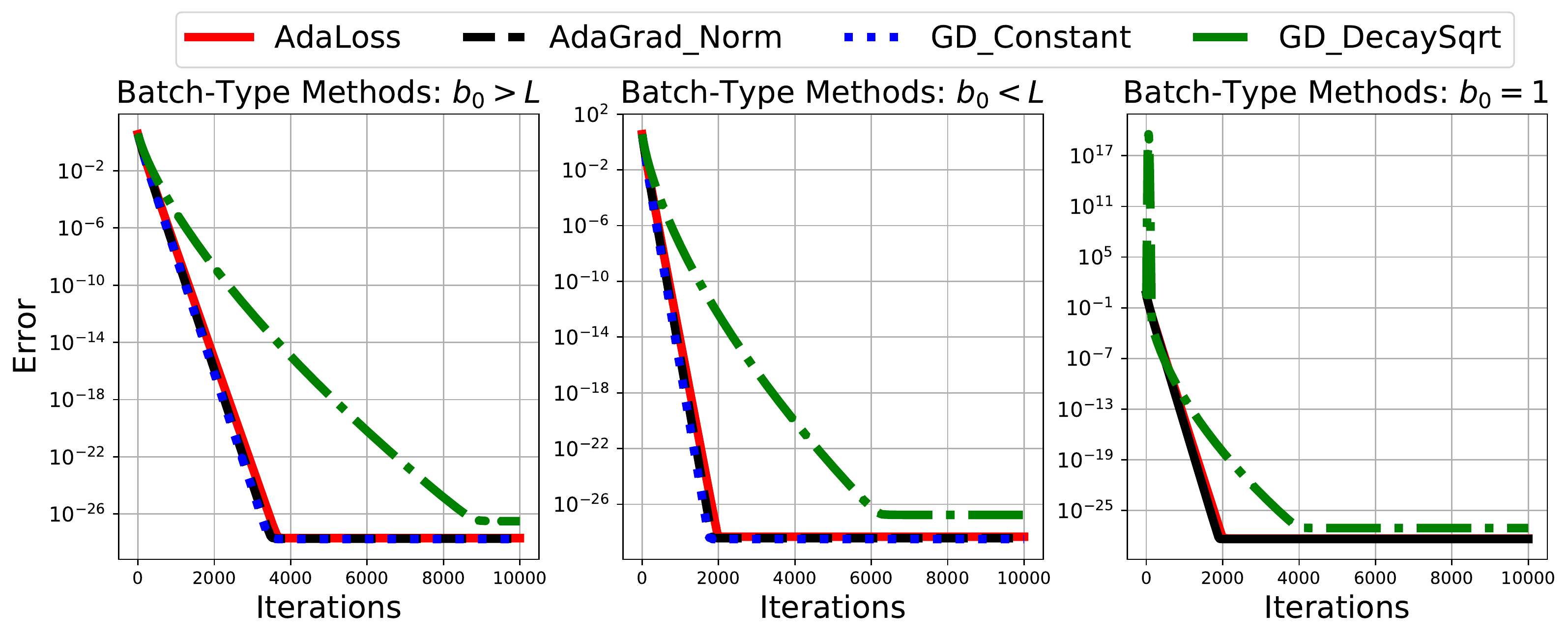}
    \end{minipage} 
    \begin{minipage}[c]{0.5\textwidth}
        \centering
             \includegraphics[width=1.\linewidth]{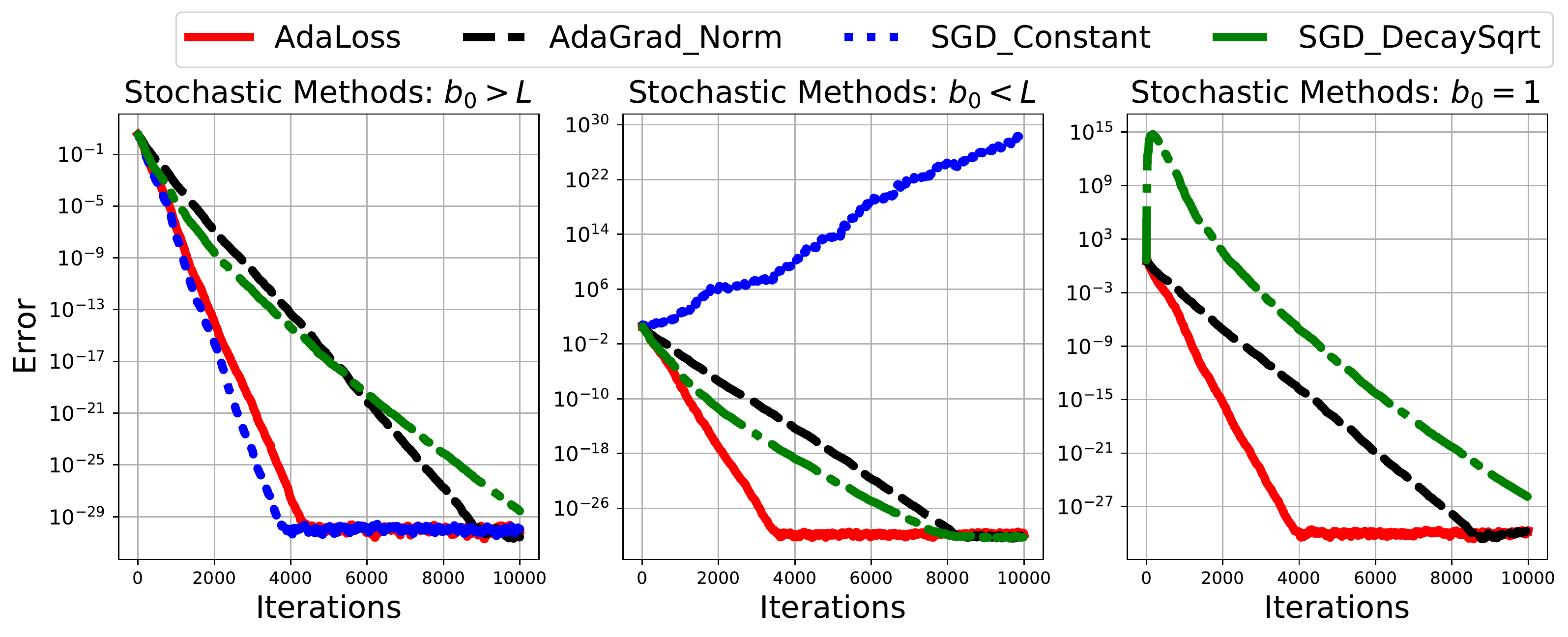}
\end{minipage}
         \caption{The left three plots are linear regression in the deterministic setting, and the right three plots in the stochastic setting. The x-axis is the number of iterations, and y-axis is the error $\|\vw_t-\vw^*\|^2$ in log scale. Each curve is an independent experiment  for algorithms:  AdaLoss (red), AdaGrad-Norm (black), (stochastic) GD with constant stepsize (blue) and SGD with square-root decaying stepsize (green).} \label{fig:adaloss1}

    \begin{minipage}[c]{0.48\textwidth}
        \centering
 \includegraphics[width=.9\linewidth]{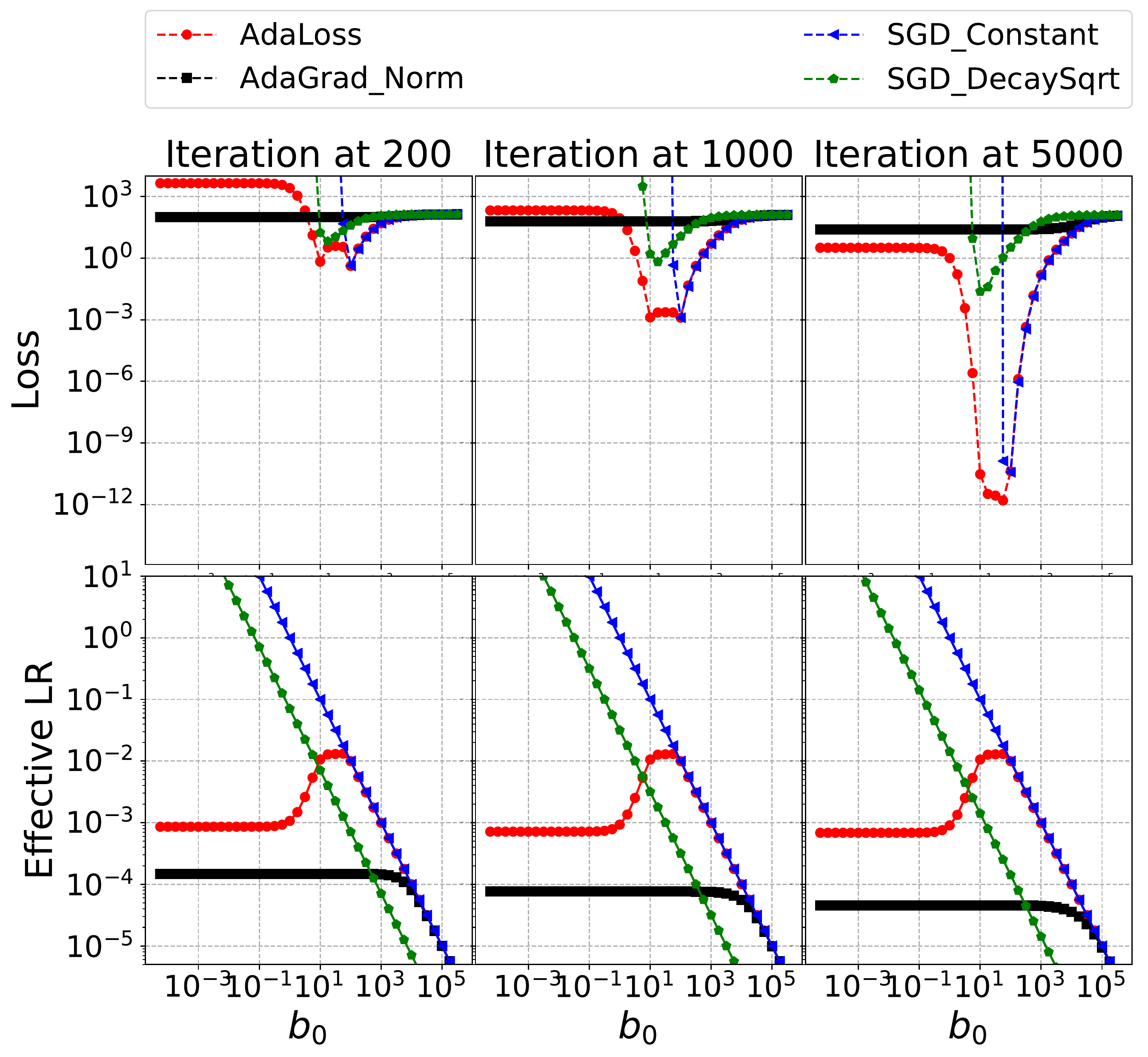}
 \caption{Linear regression in the stochastic setting. The top (bottom) 3 figures plot the average of loss (effective stepsize $1/b_t$) w.r.t. $b_0$, for iterations $t$ in $[101,200]$, $[991,1000]$ and $[4901,5000]$ respectively. }\label{fig:adaloss2} 
\end{minipage}
\hfill
    \begin{minipage}[c]{0.48\textwidth}
        \centering
    \includegraphics[width=0.9\linewidth]{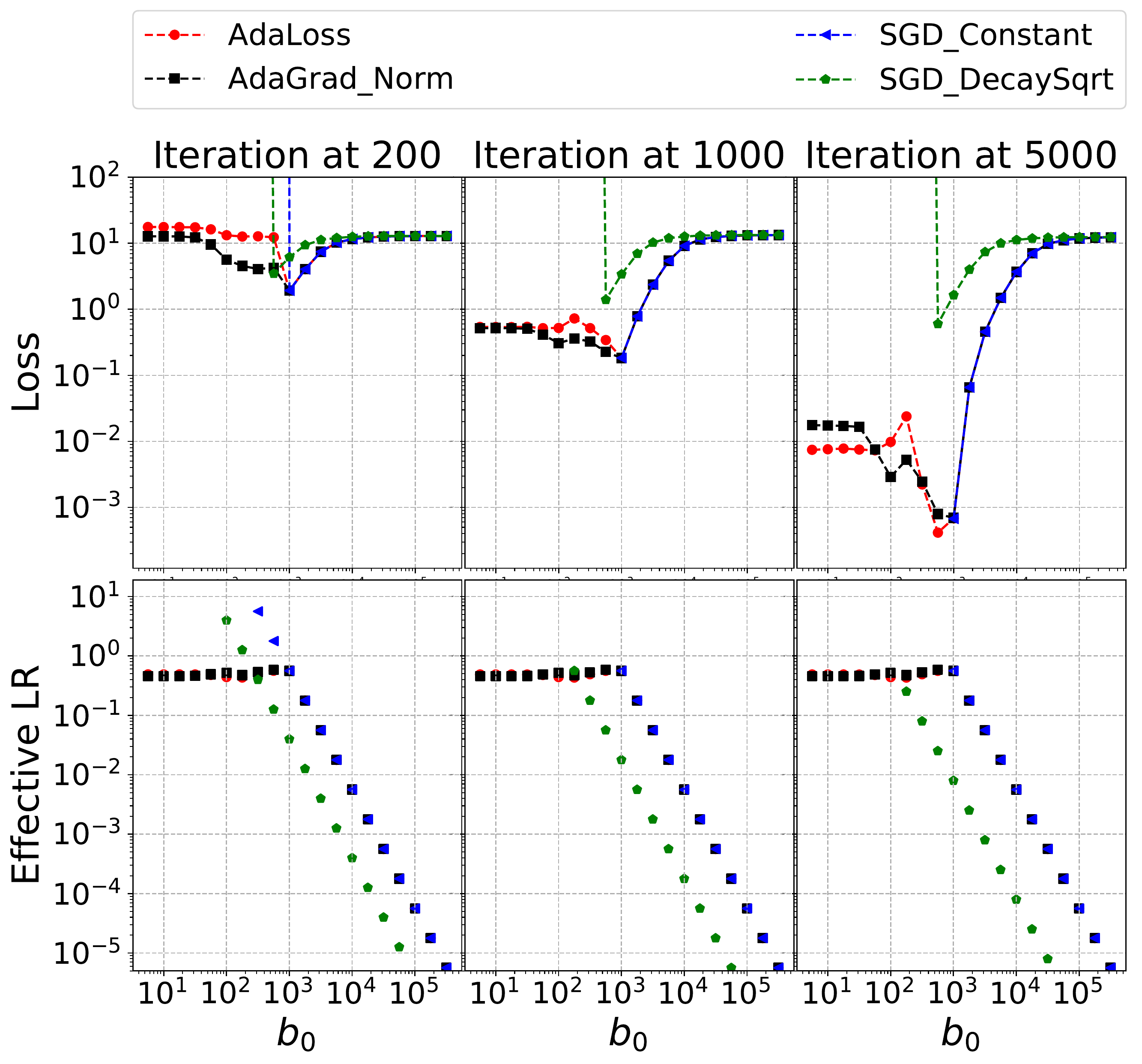}
\caption{Synthetic data (Gaussian) -- stochastic setting for two-layer neural network. The  top (bottom) 3 figures plot the average of loss (effective stepsize $1/b_t$) w.r.t.  $b_0$  for iterations $t$ in $[101,200]$, $[991,1000]$ and $[4901,5000]$. }\label{fig:adaloss}
  \end{minipage}   
  \vspace{-0.5cm}
\end{figure}

 For general functions, stochastic GD (SGD) is often limited to a sub-linear convergence rate \citep{levy2017online,zou2018stochastic,ward2018adagrad}. However, when there is no noise at the solution ($\nabla f_i (x^*)=0$ for all $i$), we prove that the limit of $b_t$ is bounded, $\lim_{t\to \infty}b_t <\infty$, which ensures the linear convergence.  Let us first state the update:
\[\vw_{t+1} =  \vw_t-\left({\eta}/{b_{t+1}}\right)\left(\vx_{\xi_t}^\top\vw_t-y_{\xi_t}\right)\vx_{\xi_t} \text{ with }  b^2_{t+1} = b^2_{t}+ R_t \]
Here, $R_t$  is defined for AdaGrad-Norm and AdaLoss respectively as
\vspace{-0.1cm}
\begin{equation}
\textbf{(AdaGrad-Norm) }R_t = \|\vx_{\xi_t}\left(\vx_{\xi_t}^\top\vw_t-y_{\xi_t}\right)\|^2 \text{ and }\textbf{ (AdaLoss) } R_t = \left(\vx_{\xi_t}^\top\vw_t-y_{\xi_t}\right)^2.  \label{alg:linear_adas}
\end{equation}

We show the linear convergence for AdaGrad-Norm by replacing general strongly convex functions \citep{xie2019linear} with linear regression (Theorem \ref{thm:sadagrad}). For AdaLoss, we follow the same process of their proof and derive the convergence in Theorem \ref{thm:ada-loss}.  Due to  page  limit,  we  put them in the appendix. The main discovery in this process is the crucial step -- inequality \eqref{eq: adaloss} -- that improves the bound using AdaLoss. The intuition is that the add-on value of AdaLoss, $(\vx_{\xi_t}^\top\vw_t-y_{\xi_t})^2$, is smaller than that of AdaGrad-Norm ($\|\vx_{\xi_t}(\vx_{\xi_t}^\top\vw_t-y_{\xi_t})\|^2$).

In Proposition \ref{prop:ada-loss-better}, we compare the upper bounds of Theorem \ref{thm:sadagrad} and Theorem \ref{thm:ada-loss}. The proposition  shows that using AdaLoss in the stochastic setting achieves a tighter convergence bound  than AdaGrad-Norm when $2b_0\geq \eta \bar{\lambda}_1$.  
\begin{prop}\label{prop:ada-loss-better} (\textbf{Stochastic AdaLoss v.s. Stochastic AdaGrad-Norm)} Consider the problem \eqref{eq:linear} where $\bar{\lambda}_1>1$ and the stochastic gradient method in \eqref{alg:linear_adas} with $2b_0\geq \eta \sup_i\|\vx_i\|$. AdaLoss improves the constant in the convergence rate of AdaGrad-Norm up to an additive factor:  $(\bar{\lambda}_1-1)\|\vw_0-\vw^*\|^2$.
\end{prop}
\vspace{-0.35cm}
\paragraph{Numerical Experiments.}\label{sec:exp-linear} 
To verify the convergence results in linear regression, we compare four algorithms: (a) AdaLoss with ${1}/{b_t}$, (b) AdaGrad-Norm  with ${1}/{b_t}$ (c) SGD-Constant  with ${1}/{b_0}$, (d) SGD-DecaySqrt  with ${1}/({b_0 + c_s\sqrt{t}})$ ($c_s$ is a constant). See Appendix \ref{sec:exp-details} for experimental details. Figure \ref{fig:adaloss1} implies that AdaGrad-Norm and AdaLoss behave similarly in the deterministic setting, while AdaLoss performs much better in the stochastic setting, particularly when $b_0 \leq L =: \sup_i\|\vx_i\|$. Figure \ref{fig:adaloss2} implies that  stochastic AdaLoss and AdaGrad-Norm are robust to a wide range of initialization of $b_0$. Comparing AdaLoss with AdaGrad-Norm, we find that when $b_0\leq 1$, AdaLoss is not better than AdaGrad-Norm  at the beginning (at least before 1000 iterations, see the first two figures at the top row), albeit the effective learning rate is much larger than AdaGrad-Norm. However, after 5000 iterations (3rd figure, 1st row), AdaLoss outperforms AdaGrad-Norm in general.

%% file: setup.tex

 
We consider the same setup as in  \cite{du2018gradient} where they assume that the data points, $\{\vect{x}_i,y_i\}_{i=1}^n$, satisfy
\begin{asmp}
	\label{asmp:norm1}
For $i \in [n]$, $\norm{\vect{x}_i} = 1$ and $\abs{y_i} = O(1)$.
\end{asmp}
The assumption on the input is only for the ease of presentation and analysis.
The second assumption on labels is satisfied in most real-world datasets. We predict labels using a  two-layer neural network 
 \begin{align}
     f(\vW,\va,\vx) = \frac{1}{\sqrt{m}} \sum_{r=1}^{m}a_r
     \sigma(\langle{\vw_r, \vx\rangle}) 
 \end{align}
where $\vx\in \mathbb{R}^d$ is the input, for any $r \in [m]$, $\vw_r \in \mathbb{R}^d$ is the weight vector of the first layer, $a_r \in \mathbb{R}$ is the output weight, and $\sigma(\cdot)$ is ReLU activation function. 
For $r \in [m]$, we initialize the first layer vector with $\vect{w}_r(0) \sim N(\vect{0},\mat{I})$ and output weight with $a_r \sim U(\left\{-1,+1\right\})$.
We fix the second layer and train the first layer with the quadratic loss.
 Define  $u_i =f(\vW,\va,\vx_i)$ as the prediction of the $i$-th example and $\vect{u} =[u_1,\ldots,u_n]^\top \in \mathbb{R}^n$.
Let $\vect{y} = [y_1,\ldots,y_n]^\top \in \mathbb{R}^{n}$, we define 
 \begin{align*}
    L(\vW) = \frac{1}{2}\|\vu-\vy\|^2 \quad \text{ or } \quad 
    L(\vW) & =\sum_{i=1}^{n}
     \frac{1}{2}(f(\vW,\va,\vx_i)-y_i)^2 \label{eq:loss_ave}
 \end{align*}
We use $k$ for indexing since $\vect{u}(k)$ is induced by $\mat{W}(k)$. According to \citep{du2018gradient}, the matrix below determines the convergence rate of GD.
\begin{defn}
\label{def:H}
The matrix  $\mat{H}^\infty \in \mathbb{R}^{n \times n}$ is defined as follows.
For $(i,j) \in [n] \times [n]$.
\begin{align}
\mat{H}_{ij}^\infty & = \expect_{\vect{w} \sim N(\vect{0},\mat{I})}\left[\vect{x}_i^\top \vect{x}_j\indict\left\{\vect{w}^\top \vect{x}_i \ge 0, \vect{w}^\top \vect{x}_j \ge 0\right\}\right] = \vect{x}_i^\top \vect{x}_j\frac{\pi - \arccos \left(\vect{x}_i^\top \vect{x}_j\right)}{2\pi}
\end{align}

\end{defn}
This matrix represents the kernel matrix induced by Gaussian initialization and ReLU activation function.
We make the following assumption on $\mat{H}^{\infty}$.
\begin{asmp}\label{asmp:lambda_0}
	The matrix  $\mat{H}^\infty \in \mathbb{R}^{n \times n}$ in Definition \ref{def:H} satisfies $\lambda_{\min}(\mat{H}^{\infty}) \triangleq \lambda_ 0 > 0$.
\end{asmp}
\citep{du2018gradient} showed that this condition holds as long as the training data is not degenerate.
We also define the following empirical version of this Gram matrix, which is used in our analysis.
For $(i,j) \in [n] \times [n]$:
$\mat{H}_{ij} = \frac{1}{m}\sum_{r=1}^{m}\vect{x}_i^\top \vect{x}_j \indict{\left\{\vect{w}_r^\top \vect{x}_i \ge 0, \vect{w}_r^\top \vect{x}_j \ge  0\right\}}. $
 

%% file: gradient_descent.tex
We first consider GD with a constant learning rate ($\eta$) 
 $\vW{(k+1)} = \vW{(k)} - \eta \frac{\partial L(\vW{(k)})}{\partial \vW }.$
 \citep{du2018gradient} showed gradient descent achieves zero training loss with learning rate $\eta = O(\lambda_0/n^2)$. 
 Based on the approach of eigenvalue decomposition in \cite{arora2019fine} (c.f. Lemma \ref{thm:convergence_rate}), we show that the maximum allowable learning rate can be improved from $O(\lambda_0/n^2)$ to $O(1/\norm{\mat{H}^{\infty}})$.
\begin{thm}\label{thm:main_gd}
(\textbf{Gradient Descent with Improved Learning Rate}) Under Assumptions~\ref{asmp:norm1} and~\ref{asmp:lambda_0}, if the number of hidden nodes $m =
\Omega\left(\frac{n^8}{\lambda_0^4 \delta^{3}}\right)
$ and  we set the stepsize 
$ \eta = \Theta \left ( \frac{1}{ \norm{\vH^{\infty}}} \right),$
 then with probability at least $1-\delta$
with respect to the random initialization, we have $L(\mat{W}(T)) \le \varepsilon$ for \footnote{
 			$\widetilde{O}$ and $\widetilde{\Omega}$ hide $\log(n),\log(1/\lambda_0), \log(1/\delta)$ terms.	}
\[ T = \widetilde{O} \left(\left( \norm{ \vH^{\infty}}/{   \lambda_0 }\right) \log\left( { 1 }/{ \varepsilon} \right) \right) \]
 	
\end{thm}

Note that since $\norm{\mat{H}^{\infty}} \le n$,  Theorem \ref{thm:main_gd} gives an $O(\lambda_0/n)$ improvement.
The improved learning rate also  gives a tighter iteration complexity bound $O \left(({ \norm{ \vH^{\infty}}}/{\lambda_0 }) \log\left({ n }/{ \varepsilon  }\right) \right)$, compared to the $O \left( ({n^2} /\lambda_0^2) \log\left( { n }/{\varepsilon }\right) \right)$ bound in \cite{du2018gradient}. Empirically, we find that if the data matrix is approximately orthogonal, then $\norm{\mat{H}^{\infty}} =O\left(1\right)$ (see Figure \ref{fig:data} in Appendix \ref{sec:basic}). 
Therefore, we show that the iteration complexity of gradient descent is nearly independent of  $n$. 

%% file: adagrad.tex
We surprisingly found that there is  a strong connection between over-parameterized neutral networks and  linear regression. We observe that  $\vH^{\infty}$ in the over-parameterized setup and $\vX^\top \vX$ in linear regression  share a strikingly similar role in the convergence. Based on this observation, we combine the induction proof of Theorem \ref{thm:main_gd} with the convergence analysis of Theorem \ref{thm:linearRegression}.  An important observation is that one needs the overparameterization level $m$ to be sufficiently large so that the adaptive learning rate can still have enough “burn in” time to reach the critical value for small initialization $b_0$, while ensuring that the iterates $\|\vW(t)-\vW(0)\|_F$ remain sufficiently small and the positiveness of the Gram matrix.  Theorem \ref{thm:main_adagradloss} characterizes the convergence rate of AdaLoss.

\begin{thm} \textbf{(AdaLoss Convergence for Two-layer Networks)}
	\label{thm:main_adagradloss}
Consider Assumptions~\ref{asmp:norm1} and~\ref{asmp:lambda_0}, and suppose the width satisfies $
	m = {\Omega} \left(\frac{n^8}{\lambda_0^4\delta^3} +  \frac{\eta^2 n^6}{ \lambda_0^2\delta^2\alpha^2} \left(\frac{\lambda_0}{\alpha^2\sqrt{n}\varepsilon}+\frac{1}{\alpha^2\varepsilon} \right)\mathbbm{1}_{b_0<\eta (C (\lambda_0+\norm{\mat H^{\infty}}))/2}\right).
	$
	Then, the update using $b_{k+1}^2 \leftarrow b_k^2 + \alpha^2\sqrt{n}\| \vy-\vu (k) \|^2$ admits  the following convergence results.\\
 \textbf{(a)}  If $ {b_0} \geq \eta C (\lambda_0+\norm{\mat H^{\infty}})/2$, then with probability $1-\delta$ with respect to the random initialization, we have  $\min_{t\in [T]}\norm{\vect y- \vect u(t)}^2 \le\varepsilon$ 
		after
 	 \begin{small}
 	 \[ T = \widetilde{O}\left( \left(\frac{b_0 }{\eta \lambda_0} + \frac{\alpha^2 n^{3/2}}{\eta^2\lambda_0^2 {\delta} } \right) \log\left(\frac{1}{\varepsilon} \right)\right).\]
 \end{small}
 \textbf{(b)} If $ 0<{b_0}\leq {\eta} C (\lambda_0+\norm{\mat H^{\infty}})/2$, then with probability $1-\delta$ with respect to the random initialization, we have  $\min_{t\in [T]}\norm{\vect y- \vect u(t)}^2 \le\varepsilon$   after
 \begin{small}
\[T =  \widetilde{O}\left( \frac{ \lambda_0+\sqrt{n} }{\alpha^2\sqrt{ n} \varepsilon} + \left(\frac{\alpha^2 n^{3/2}}{\eta^2\lambda_0^2 {\delta}} + \frac{\|\vH^\infty\|^2}{\lambda_0^2} \right) \log \left(\frac{1}{\varepsilon} \right)\right).
\]
 \end{small}
\end{thm}
To our knowledge, this is the first global convergence guarantee of any adaptive gradient method for neural networks robust to initialization of $b_0$. It improves the results in \citep{ward2018adagrad}, where AdaGrad-Norm is shown only to converge to a stationary point. 
 Besides the robustness to the hyper-parameter, two key implications in Thm 4.2: (1) Adaptive gradient methods can converge \textit{linearly} in certain two-layer networks using our new technique developed for linear regression (Theorem \ref{thm:linearRegression}); (2) But that linear convergence and  robustness comes with \textit{a cost}: the width of the hidden layer has to be much wider than $n^8$. That is, when the initialization $b_0$ satisfying Case (b), the leading rate for $m$ is its second term, i.e. ($\eta^2 n^6)/(\alpha^4\epsilon)$, which is larger than $n^8$ if $\epsilon$ is sufficiently small.

 We remark that Theorem \ref{thm:main_adagradloss} is different from Theorem 3 in \cite{xie2019linear} which achieves convergence by assuming a PL inequality for the loss function.  This condition --  PL inequality -- is not guaranteed in general. The PL inequality is satisfied in our two-layer network problem when the Gram matrix $\vH^{\infty}$ is strictly positive (see Proposition \ref{prop:a1}).  That is, in order to satisfy PL-inequality, we use induction to show that the model has to be sufficiently overparameterized, i.e., $m=O\left((poly(n^8, \alpha, \eta, \lambda_0, \delta,\varepsilon)\right)$.
\begin{figure*}[ht]
    \centering
    \includegraphics[width=0.66\linewidth]{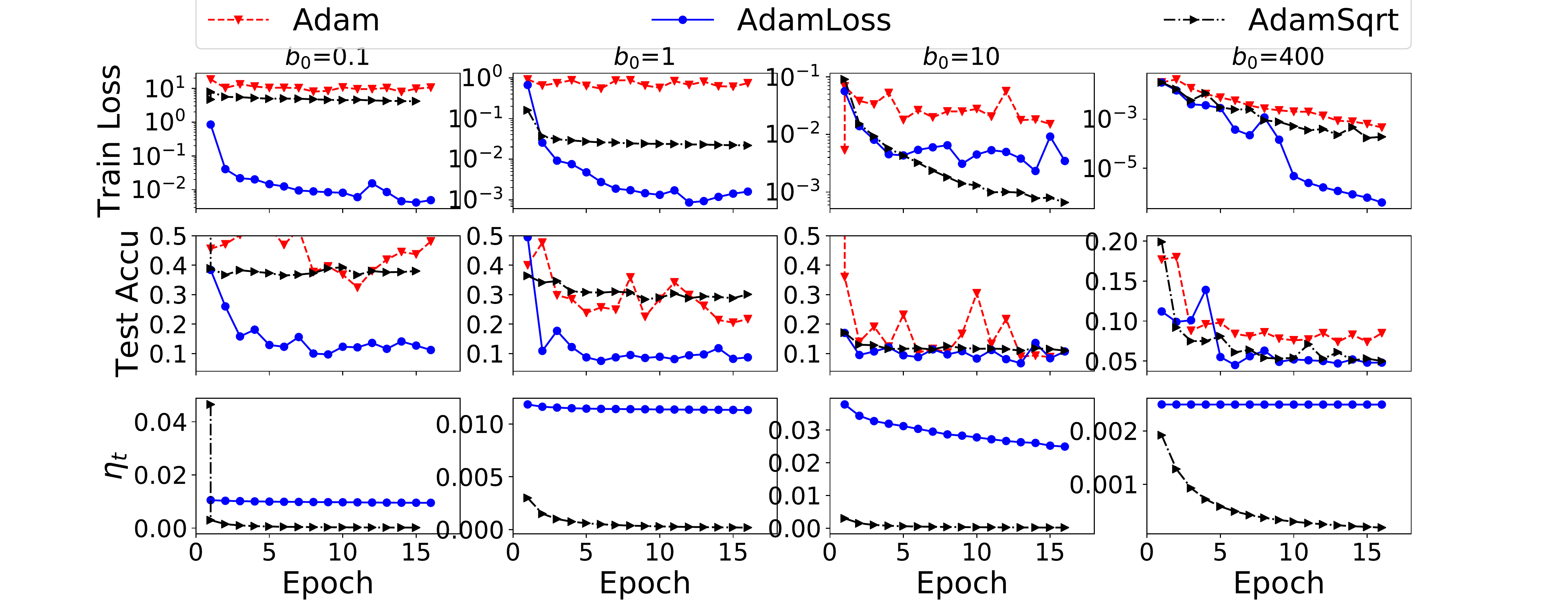}
     \includegraphics[width=0.33\textwidth]{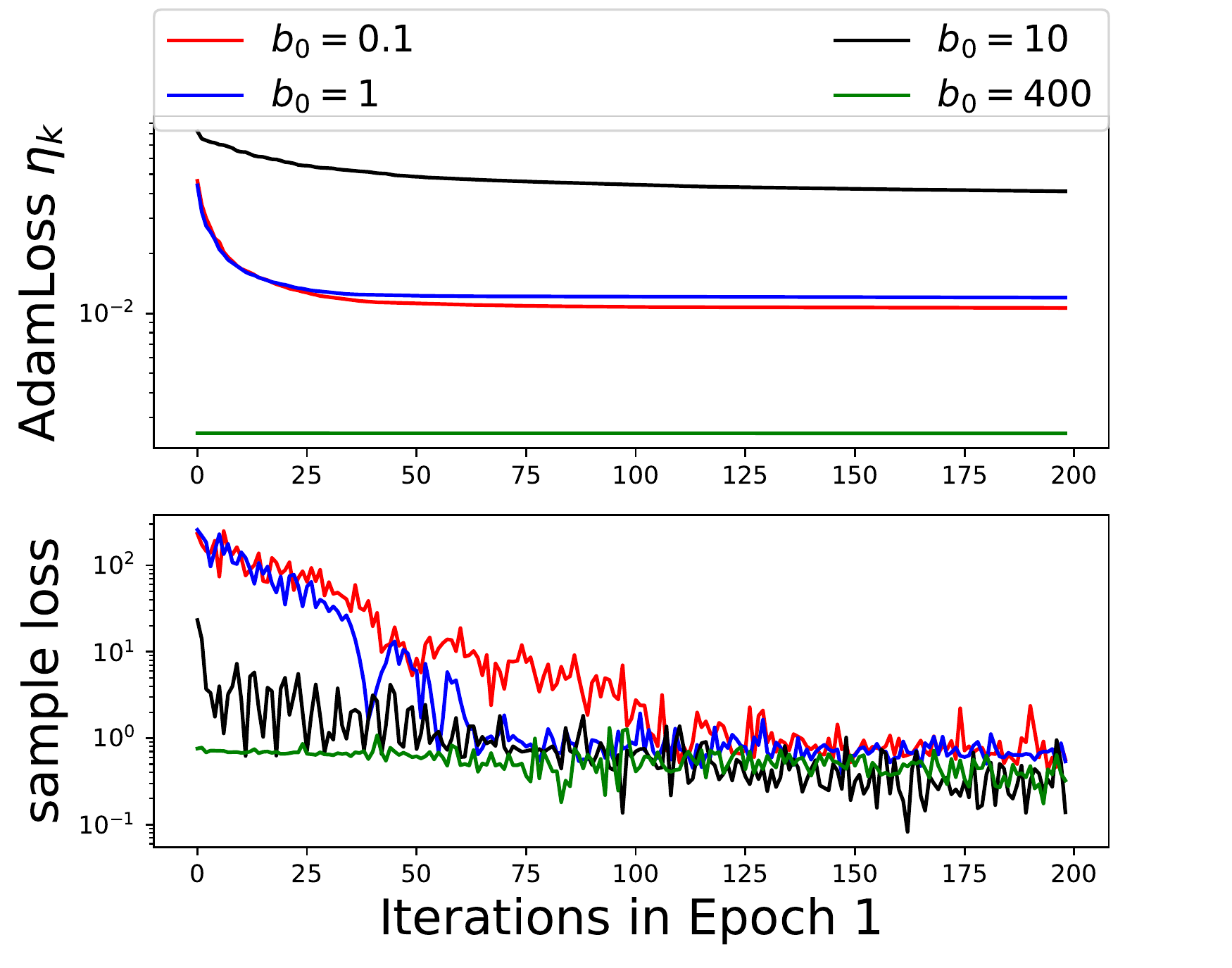}
       \caption{Text Classification - LSTM model. On the left, the top four plots are the training loss w.r.t. epoch;  the middle ones are test accuracy  w.r.t. epoch; the bottom ones stepsize $\eta$ w.r.t. epoch. On the right, the top (bottom) plot is the stepsize $\eta_t$ (the stochastic loss) w.r.t. to iterations in the 1st epoch. \label{fig:adaloss-lstm}}
\includegraphics[width=0.63\linewidth]{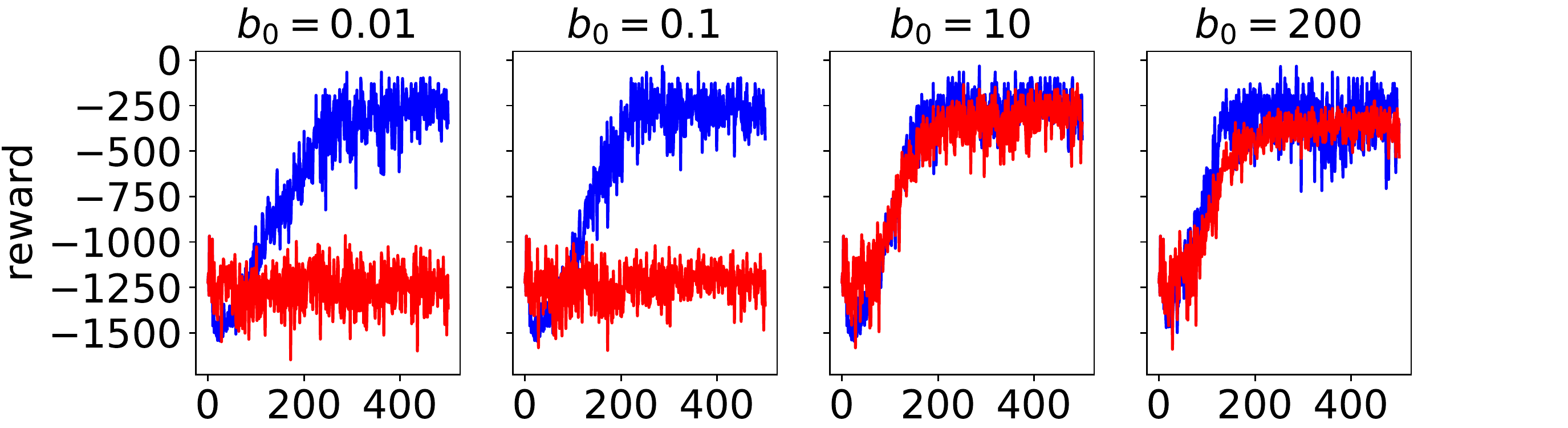}
  \includegraphics[width=0.33\textwidth]{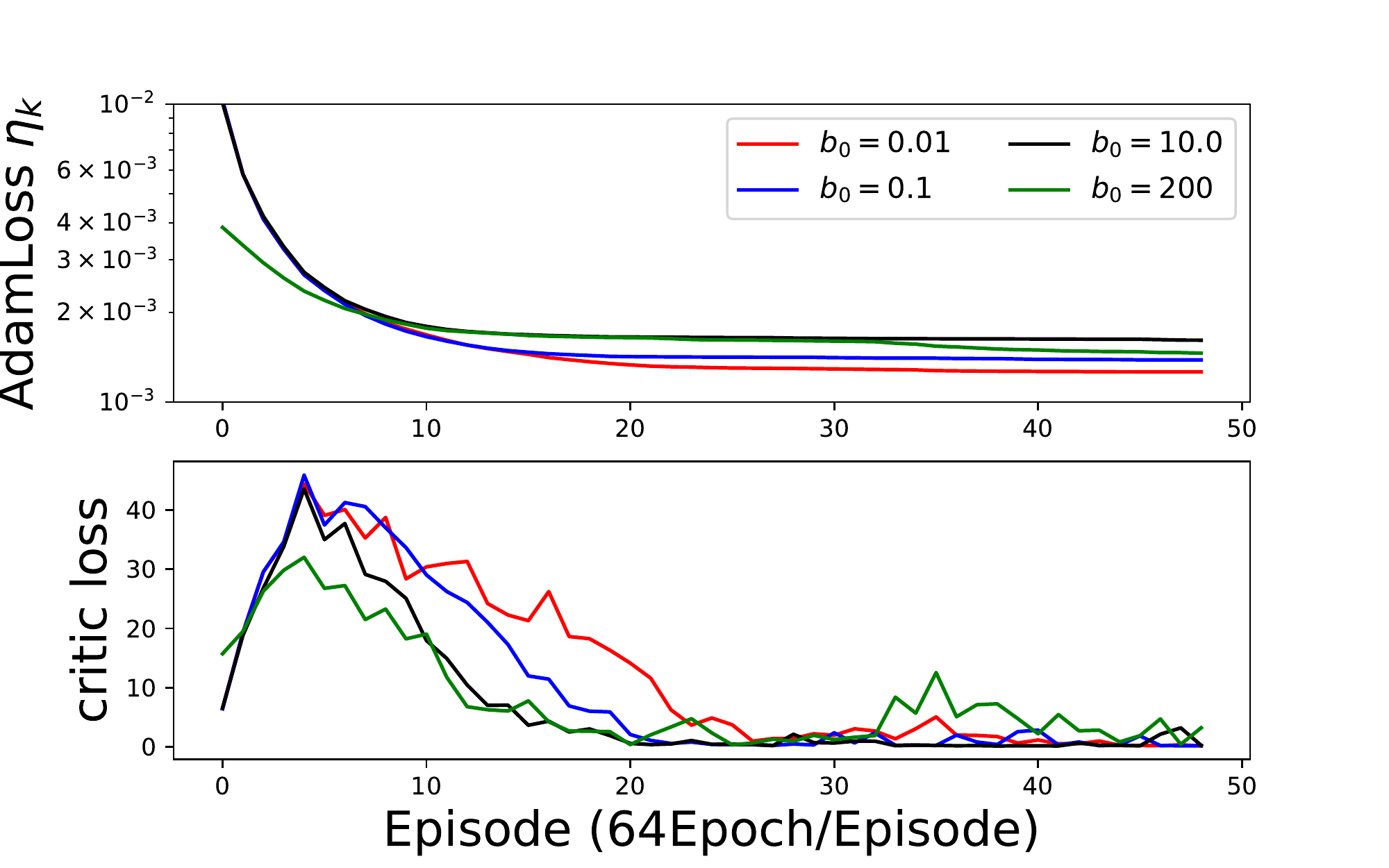}
      \caption{ Inverted Pendulum Swingup  with Actor-critic Algorithm.  On the left, the 4 plots are the rewards (scores) w.r.t. number of frames $10000$ with roll-out length $2048$. On the right, the top (bottom) figure is the stepsize $\eta_t$ (the stochastic loss) w.r.t. to total episode. \label{fig:adaloss-rl} }
\end{figure*}

Theorem~\ref{thm:main_adagradloss} applies to two cases.
In the first case, the effective learning rate at the beginning ($\eta/b_0$) is smaller than the threshold $2/ (C (\lambda_0+\norm{\mat H^{\infty}}))$ that guarantees the global convergence of gradient descent (c.f. Theorem~\ref{thm:main_gd}).
In this case, the convergence has two terms, and the first term $\frac{b_0}{\eta \lambda_0}\log\left(\frac{1}{\epsilon}\right)$ is the  standard gradient descent rate if we use $\eta/b_0$ as the learning rate.
Note this term is the same as Theorem~\ref{thm:main_gd} if $\eta/b_0 = \Theta(1/\norm{\mat{H}^\infty})$.
The second term 
comes from the upper bound of  $b_T$ in the effective learning rate $\eta/b_{T}$ (c.f. Lemma \ref{lem: small_w_for_allbsq_sub}). This case shows if $\alpha$ is sufficiently small that the second term is smaller than the first term, then we have the same rate as gradient descent.

In the second case, the initial effective learning rate, $\eta/b_0$, is greater than the threshold that guarantees the convergence of gradient descent.
Our algorithm guarantees either of the followings happens after $T$ iterations:
(1) The loss is already small, so we can stop training.
This corresponds to the first term ${(\lambda_0+\sqrt{n} )}/({\alpha^2\sqrt{ n\varepsilon}}) $.
(2) The loss is still large, which makes the effective stepsize, $\eta/b_k$, decrease with a good rate, i.e., if (2) keeps happening, the stepsize will decrease till $\eta/b_k \le 2/ (C (\lambda_0+\norm{\mat H^{\infty}}))$, and then it comes to the first case.
Note that the second term here is the same as the second term of the first case, but the third term, $\left({\|\vH^\infty\|}/{\lambda_0} \right)^{2}\log\left({1}/{\epsilon}\right)$ is slightly worse than the rate in  the gradient descent. The reason is that the loss may increase due to the large learning rate at the beginning (c.f. Lemma \ref{lem:log_growth_sub}).   


When comparing AdaGrad-Norm, one could get the same convergence rate as AdaLoss. The comparison between AdaGrad-Norm and AdaLoss are almost the same as in linear regression. The bounds of AdaGrad-Norm and AdaLoss are similar, since our analysis for both algorithms is the worst-case analysis. However, numerically, AdaLoss can behave better than AdaGrad-Norm: Figure \ref{fig:adaloss} shows that AdaLoss performs almost the same as or even better than AdaGrad-Norm with SGD.  As for extending Theorem \ref{thm:main_adagradloss} to the stochastic setting, we leave this for future work.  We devote the rest of the space to real data experiments. 



%% file: aaaexperiments.tex
    
In this section, we consider the application of AdaLoss in the practical domain. Adam \cite{kingma2014adam} has been successfully applied to many machine learning problems. However, it still requires fine-tuning the stepsize $\eta$ in Algorithm \ref{alg:adam}. Although the default value is $\eta=0.001$, one might wonder if this is the optimal value. Therefore, we apply AdaLoss to make the value $\eta$ robust to any initialization (see the blue part in Algorithm \ref{alg:adam}) and name it \textbf{AdamLoss}. We take two tasks to test the robustness of AdamLoss and compare it with the default Adam as well as AdamSqrt, where we literally let $\eta = 1/\sqrt{b_0+t}$. Note that for simplicity, we set $\alpha=1$. More experiments are provided in the appendix for different $\alpha$.

    \begin{small}
\begin{algorithm}
	\caption{Adam\textcolor{blue}{Loss}}
	\label{alg:adam}
\begin{algorithmic}[1]
		\STATE {\bfseries Input:} $x_1$, $\beta_{1} = 0.9$, $\beta_2=0.99$, and positive value  $\eta$  \textcolor{blue}{ $\alpha_0$ and $b_0$}. Set $m_{0} =v_{0} = \hat{v}_{0} = 0$
		\FOR{$t=1,2,3, \ldots T$}
			\STATE \textcolor{blue}{$b_t =b_{t-1}+ \alpha |f_t(x_t)|$}
		\STATE \textcolor{blue}{$\eta_t =1/\sqrt{b_t}$}		
		\STATE $g_t = \nabla f_t(x_t)$ (Get the gradient)
		\STATE $m_{t} = \beta_{1} m_{t-1} + (1 - \beta_{1})  g_{t}$
		\STATE $v_{t} = \beta_2 v_{t-1} + (1 - \beta_2) g_{t}^2$ 
		\STATE $\hat{m}_{t} =  m_{t}/ (1 - \beta_{1}^t)$ 
		\STATE $\hat{v}_{t} =v_t/(1-\beta_2^t)$ 
		\STATE (Adam) $ x_{t+1} = x_{t} -\eta \hat{m}_{t}/ \sqrt{\hat{v}_t+\epsilon}$
		\STATE \textcolor{blue}{(AdamLoss)} $ x_{t+1} = x_{t} -$ \textcolor{blue}{$\eta_t$} $\hat{m}_{t}/ \sqrt{\hat{v}_t+\epsilon} $ 
		\ENDFOR
\end{algorithmic}
\end{algorithm}
    \end{small}


The first task is two-class (Fake/True News) text classification using one-layer LSTM (see Section \ref{sec:exp-details} for details). The left plot in Figure \ref{fig:adaloss-lstm} implies that the training loss is very robust to any initialization of the AdamLoss algorithm and subsequently achieves relatively better test accuracy. The right plot in Figure \ref{fig:adaloss-lstm} captures the dynamics of $1/b_t$ for the first 200 iterations at the beginning of the training. We see that when $b_0=0.1$ (red) or $b_0=1$ (blue), the stochastic loss (bottom right) is very high such that after $25$ iterations, it reaches $1/b_t \approx 0.01$ and then stabilizes.
When $b_0=400$, the stochastic loss shows a decreasing trend at the beginning, which means it is around the critical threshold.  

The second task is to solve the classical control problem: inverted pendulum swing-up. One popular algorithm is the actor-critic algorithm \cite{konda2000actor}, where the actor algorithm is optimized by proximal policy gradient methods \cite{zoph2018learning}, and the critic algorithm is optimized by function approximation methods \cite{fujimoto2018addressing}. The actor-network and critic-network are fully connected layers with different depths.  We use Adam and AdamLoss to optimize the actor-critic algorithm independently for four times and average the rewards. The code source is provided in the supplementary material.
The left plot of Figure \ref{fig:adaloss-rl} implies that AdaLoss is very robust to different initialization, while the standard Adam is extremely sensitive to  $\eta=\frac{1}{b_0}$. Interestingly, AdamLoss does better when starting with $\eta_0=\frac{1}{200}$.  We plot the corresponding $1/b_t$ on the right-hand side in Figure \ref{fig:adaloss-rl}. We see that regardless of the initialization of $b_0$, the final value $1/b_t \approx 0.01$ reaches a value between $0.002$ and $0.001$. 

Overall, AdamLoss is shown numerically robust to any initialization $b_0$ for the two-class text classification and the inverted pendulum swing-up problems. See appendix for more experiments.




%% file: appendix-sum.tex
\newpage
\appendix
\section*{Appendix}
Appendix for the paper: "AdaLoss: A computationally-efficient and provably convergent adaptive gradient method". This appendix includes:
\begin{itemize}

    \item Appendix \ref{sec:linear-proof}: Proof for Linear Regression Theorem \ref{thm:linearRegression}
    
    \item Appendix \ref{sec: pfslrt}: Proof for Stochastic Linear Regression 
        \item Appendix \ref{sec:exp}: Experiments for $\|\vH(k)\|$
    \item Appendix \ref{sec:proof_gd}: Proof for Theorem \ref{thm:main_gd}
        \item Appendix \ref{D}: Proof for Theorem \ref{thm:main_adagradloss}
    \item Appendix \ref{sec:basic}: Technical Lemmas
    \item Appendix \ref{sec:exp-details}: Experiments

\end{itemize}

\section{Linear Regression}

\label{sec:linear-proof}
\input{linear_proof_non.tex}

\subsection{Proof for Stochastic Linear Regression 
}
\label{sec: pfslrt}
\input{linear_proof_sto.tex}

\section{Two-Layer Networks}
\subsection{Experiments}
\label{sec:exp}
\input{experiments.tex}

\subsection{Proof of Theorem \ref{thm:main_gd}}
\label{sec:proof_gd}
\input{discrete.tex}

\subsection{Proof Sketch of Theorem~\ref{thm:main_adagradloss}}
\label{sec:adaloss_proof}
\input{AproofGD_adagrad.tex}

\section{Technical Lemmas}
\label{sec:basic}
\input{basic.tex}

\section{Experiments}
\label{sec:exp-details}
\input{exp-details.tex}

%% file: linear_proof_non.tex
This whole section is devoted to Theorem \ref{thm:linearRegression} that we restate with an explicit form of $T$ as follows. Note that the minimum eigenvalue is denoted by $\bar{\lambda}_n$ instead of $\bar{\lambda}_0$ in the main text.
\begin{thm}
\label{thm:linearRegression1}\textbf{[Restatement of Theorem \ref{thm:linearRegression}, Improved AdaGrad-Norm Convergence]}
Denote $\Delta_0 = \vw_0-\vw^*$. Consider  the problem \eqref{eq:linear}  and the gradient descent method:
\begin{align}
\vw_{t+1} =  \vw_t-\left({\eta}/{b_{t+1}}\right)\vX^T\left(\vX\vw_t-y\right) 
\end{align}
with $b^2_{t+1}= b^2_{t}+\|\vX^T\left(\vX\vw_t-y\right)\|^2$. \\
 \textbf{(a)} If $b_0 \geq \eta\bar{\lambda}_1/2$,  we have $\| \vw_{T}-\vw^*\|^2\leq \epsilon$ after
\[T= \left\lceil\max \left\{ T_1, T_2, T_3  \right\}\frac{1}{\bar{\lambda}_n} \log\left(\frac{\|\Delta_0\|^2}{\epsilon}\right)\right\rceil+1 \]
iterations, where the terms $T_1, T_2$ and $T_3$ are
\begin{align}
 T_1 &= \frac{b_0^2+\|\vX^T\vX \Delta_0\|^2}{\left(2\sqrt{b_0^2+\|\vX^T\vX \Delta_0\|^2 }-\eta \bar{\lambda}_1\right)} \mathbbm{1}_{ \left\{\sqrt{b_0^2+\|\vX^T\vX \Delta_0\|^2} \in \left(\frac{ \eta \bar{\lambda}_1}{2}, \frac{ \eta (\bar{\lambda}_n+\bar{\lambda}_1)}{2}\right) \right\} },\\
 T_2 &= \frac{ \eta( \bar{\lambda}_1+\bar{\lambda}_n) }{2}, \\
 T_3 &= \bar{\lambda}_1+\frac{ \|\vX \Delta_0\|^2}{\eta^2}.
\end{align}
 \textbf{(b)}  If $b_0<\eta\bar{\lambda}_1/2$,   we have $\| \vw_{T+s}-\vw^*\|^2\leq \epsilon$ after 
\begin{align*}
T =  \left\lceil\max \left\{ 
\bar{T}_1, \bar{T}_2, \bar{T}_3 \right\}
\frac{1}{\bar{\lambda}_n}\log \left( 
  \frac{\|\bar{\Delta}_0\|^2}{\epsilon}
  \right)\right\rceil+1.
\end{align*}
iterations, where $\|\bar{\Delta}_0\|^2=\|{\Delta_0}\|^2+\eta^2 \log\left(\frac{ \eta \bar{\lambda}_1}{2b_0} \right)$ and the terms $s, \bar{T}_1, \bar{T}_2 $ and $\bar{T}_3$ are
\begin{align} 
s &\leq 2 \log  \left(1+
\frac{ \left({\eta \bar{\lambda}_1} \right)^2 - 4(b_0)^2 }{\eta^2  \bar{\lambda}_1^2 }
\right) / \log \left (1+   \frac{ 4}{ \eta^2 } \left([\hat{\vw}_{0}]_1 -[\hat{\vw}^{*}]_1 \right)^2 \right)\\
  \bar{T}_1&= \left( \eta\bar{\lambda}_n+5\eta \bar{\lambda}_1  \right) (\delta_s+1)\mathbbm{1}_{ \left\{ b_{s+1}\in\left(\frac{ \eta \bar{\lambda}_1}{2}, \frac{ \eta( \bar{\lambda}_1+\bar{\lambda}_n) }{2}\right) \right\} },
  \\
  \bar{T}_2&=  {(\bar{\lambda}_1+\bar{\lambda}_n)}/{ 2 },
\\
  \bar{T}_3 &= {\bar{\lambda}_1}  + \frac{ \| \vX \Delta_0\|^2}{\eta^2}+ \frac{\bar{\lambda}_1}{\eta}\log\left(\frac{ \eta^2 \bar{\lambda}_1}{2b_0}\right),
\end{align}

For  the expression of $\bar{T}_1$, the value $\delta_s$ is a small constant of $O(1)$ (expression in Table \ref{table} or Lemma \ref{lem:increase_sub1}).
\end{thm}

The theorem above is stated separately in Theorem \ref{lem:increase_sub1} for Case (1) and Theorem \ref{lem:conv-small} for Case (2). Theorem  \ref{lem:increase_sub1}  will use Lemma \ref{lem:b-max}; Theorem \ref{lem:conv-small}  will use Lemma \ref{lem:increase_sub} and Theorem \ref{lem:conv-small}.

We briefly discuss the expressions of $T_1, T_2$ and $T_3$ for case (1) and $\bar{T}_1, \bar{T}_2 $ and $\bar{T}_3 $ for case (2).\\
\textbf{For case (1).} The term $T_3$ is directly related to the upper bound of $b_t$ (i.e., $\sup_{t} b_t$), while $T_1$ is related to the first update of $b_t$ (i.e., $b_1$) and $T_2$ is for the transition of $b_t$  between $T_1$ and $T_2$.\\
\textbf{For case (2).} The term $s$ is the iteration for $b_t$ from a small value to reach the critical one $\eta \bar{\lambda}_1/2$. Then $\bar{T}_1$, $\bar{T}_2$ and so $\bar{T}_3$ respectively corresponds to terms $T_1, T_2$ and $T_3$.\\
Suppose $\bar{\lambda}_1\gg 1$, $\eta<1$ and the term $\frac{\|\vX{\Delta_0}\|}{\eta}\geq b_0$, then $T_3$ is the greater than $T_1$ and $T_2$, and $\bar{T}_3$ is greater than $\bar{T}_1$ and $\bar{T}_2$.  The order in the main text
\begin{small}
\begin{align}
T=  \widetilde{\mathcal{O}}\left( \left(\frac{\|\vX(\vw_0-\vw^*)\|^2}{\eta\bar{\lambda}_1}+\frac{\eta}{s_0^2}\log \left(\frac{ \left({\eta \bar{\lambda}_1} \right)^2 - 4(b_0)^2 }{\eta^2  \bar{\lambda}_1^2 }\right) \mathbbm{1}_{\{2b_0\leq \eta \bar{\lambda}_1\}} \right)\frac{\bar{\lambda}_1\log\left({1}/{\epsilon}\right)}{\eta\bar{\lambda}_0} \right)
\end{align}
\end{small}
is simplified with ${T}_3$ and $s$ given in Lemma \ref{lem:increase_sub}.

\begin{cor}
\label{cor:linearRegression2}\textbf{[Restatement of Corollary \ref{cor:linearRegression}, AdaLoss Convergence]}
 Denote $\Delta_0 = \vw_0-\vw^*$. Consider  the problem \eqref{eq:linear}  and the gradient descent  method:
\begin{align}
\vw_{t+1} =  \vw_t-\left({\eta}/{b_{t+1}}\right)\vX^T\left(\vX\vw_t-y\right) 
\end{align}
with $b^2_{t+1}= b^2_{t}+\|\vX\vw_t-y\|^2$.\\
 \textbf{(a)} If $b_0 \geq \eta\bar{\lambda}_1/2$,  We have $\| \vw_{T}-\vw^*\|^2\leq \epsilon$ after
\[T= \left\lceil\max \left\{ T_1,  \bar{\lambda}_1+\frac{ \|\Delta_0\|^2}{\eta^2}  \right\}\frac{1}{\bar{\lambda}_n} \log\left(\frac{\|\Delta_0\|^2}{\epsilon}\right)\right\rceil+1 \]
iterations, where  $T_1 := \max\left\{ \frac{b_0^2+\|\vX \Delta_0\|^2}{\left(2\sqrt{b_0^2+\|\vX \Delta_0\|^2 }-\eta \bar{\lambda}_1\right)} \mathbbm{1}_{ \left\{ \sqrt{b_0^2+\|\vX \Delta_0\|^2} \in\left(\frac{ \eta \bar{\lambda}_1}{2}, \frac{ \eta( \bar{\lambda}_1+\bar{\lambda}_n) }{2}\right) \right\} },   \frac{(\bar{\lambda}_1+\bar{\lambda}_n)}{ 2  } \right\}$.\\
 \textbf{(b)}  If $b_0<\eta\bar{\lambda}_1/2$,   we have $\| \vw_{T+\widetilde{s}}-\vw^*\|^2\leq \epsilon$ after
\begin{align*}
T = 
 \left\lceil\max \left\{ 
{T_2}, 
 {\bar{\lambda}_1}  + \frac{ \|\Delta_0\|^2}{\eta^2}+ \frac{\bar{\lambda}_1}{\eta}\log\left(\frac{ \eta^2 \bar{\lambda}_1}{2b_0}\right)
 \right\}
\frac{1}{\bar{\lambda}_n}\log \left( 
  \frac{\|\hat{\Delta_0}\|^2}{\epsilon}
  \right)\right\rceil+1 
\end{align*}
iterations, where  $ T_2 := \max \left\{\left( \eta\bar{\lambda}_n+5\eta \bar{\lambda}_1  \right) (\delta_{\widetilde{s}}+1)\mathbbm{1}_{ \left\{ b_{\widetilde{s}+1}\in\left(\frac{ \eta \bar{\lambda}_1}{2}, \frac{ \eta( \bar{\lambda}_1+\bar{\lambda}_n) }{2}\right) \right\} },  {(\bar{\lambda}_1+\bar{\lambda}_n)}/{ 2 }  \right\} $ and $\|\hat{\Delta}_0\|^2=\|{\Delta_0}\|^2+\eta^2 \log\left(\frac{ \eta \bar{\lambda}_1}{2b_0} \right)$.  
Here, $\widetilde{s}$ is the first index such that $b_{\widetilde{s}}\geq \eta\bar{\lambda}_1/2$ from $b_{\widetilde{s}-1}<\eta\bar{\lambda}_1/2$ . $ \delta_{\widetilde{s}} $
and $\widetilde{s}$ are  two small scalars of $O(1)$ (expression in Table \ref{table} or Lemma \ref{lem:increase_sub1} ).
\end{cor}

Note that Corollary \ref{cor:linearRegression2} can be easily proved by   following the same process as the proof of Theorem \ref{thm:linearRegression}.  One only need to identify the difference in between the two update by checking Lemma \ref{lem:b-max} and Lemma \ref{lem:increase_sub}. Thus, we will skip the proof for Corollary \ref{cor:linearRegression2}.

\subsection{Proof for Linear Regression Theorem \ref{thm:linearRegression}}
 We separate   Theorem \ref{thm:linearRegression} into Lemma \ref{lem:increase_sub1} and Lemma \ref{lem:conv-small} for better understanding. The proof of the theorem is thus a combination of  the proof in Lemma \ref{lem:increase_sub1} and Lemma \ref{lem:conv-small}. Before we present the lemmas, we first state some basic facts of linear regression.

We can rewrite the linear regression problem  \eqref{eq:linear} as
\begin{align}
\arg \min_{\vw \in \mathbb{R}^{d} } \frac{1}{2}\|\vU\vect{\Sigma}^{\frac{1}{2}} \vV^T(\vw-\vw^*)\|^2 = \arg \min_{\hat{\vw} } \frac{1}{2}\|\vect{\Sigma}^{\frac{1}{2}} (\hat{\vw}-\hat{\vw}^*)\|^2,  \label{eq:svd}
\end{align}
where $\hat{\vw} = \vV^\top\vw $. Let $[\hat{\vw}_{t+1}]_i$ denotes the entry of the $i$-th dimension of the vector. 
%
Note that 
\begin{align}
\label{eq: scale}
    [\hat{\vw}_{t+1}]_i =  [\hat{\vw}_{t}]_i - \frac{\eta \bar{\lambda}_i^2}{b_{t+1}}([\hat{\vw}_{t}]_i -[\hat{\vw}^{*}]_i )
\end{align}
can be equivalently written as
\begin{align}
s_{t+1}^{(i)} :=\left([\hat{\vw}_{t+1}]_i-[\hat{\vw}^{*}]_i\right)^2  =&\left( 1 - \frac{\eta \bar{\lambda}_i}{b_{t+1}} \right)^2([\hat{\vw}_{t}]_i -[\hat{\vw}^{*}]_i  )^2 \label{eq:exact-dynamic}.
\end{align}
We will use the following key  inequalities: for $\eta \leq \frac{\bar{\lambda}_1}{2}$
\begin{align}
\| \vw_{t+1}-\vw^*\|^2
&\leq \left( 1-\frac{2\eta \bar{\lambda}_n}{b_{t+1}} \left(1- \frac{\eta \bar{\lambda}_1}{2b_{t+1}}\right) \right)\| \vw_{t}-\vw^*\|^2\label{eq:ineq2}\\
\| \vw_{t+1}-\vw^*\|^2
&\leq  \left( 1- \frac{2\eta \bar{\lambda}_1 \bar{\lambda}_n}{ (\bar{\lambda}_n+\bar{\lambda}_1)b_{t+1}}\right)\|\vw_{t}-\vw^*\|^2 \nonumber\\
&\quad  +\frac{\eta}{b_{t+1}} \left(\frac{\eta}{b_{t+1}}- \frac{2}{\bar{\lambda}_n +\bar{\lambda}_1}\right)\| \vX^\top \vX\left(\vw_{t}-\vw^*\right)\|^2   \label{eq:ineq3}
\\
\| \vX\left( \vw_{t+1}-\vw^*\right)\|^2
&\leq\| \vX \left(\vw_{t}-\vw^*\right)\|^2   - \frac{2\eta}{b_{t+1}} \left( 1 - \frac{\eta\bar{\lambda}_1}{2b_{t+1}} \right)\| \vX^\top \vX \left(\vw_{t}-\vw^*\right)\|^2   \label{eq:ineq1}
\end{align}
Based on the above facts, we first present the
case  when $b_0\geq\eta\bar{\lambda}_1/2$ and then the case when $b_0\leq\eta\bar{\lambda}_1/2$.

Then we prove an important lemma for the upper bound of $b_t$ for the case when $b_0\geq \eta\bar{\lambda}_1/2$.
\begin{lem}\label{lem:b-max}
Consider the following two update sequences
\begin{align}
&\text{(AdaGrad-Norm) } b_{t+1}^2 =b_{t}^2 +{\sum_{i=1}^n\bar{\lambda}_i^2([\hat{\vw}_{t}]_i -[\hat{\vw}^{*}]_i  )^2};  
\\
&\text{(AdaLoss) } 
 \widetilde{b}_{t+1}^2 = \widetilde{b}_{t}^2 +{\sum_{i=1}^n{\bar{\lambda}}_i([\hat{\vw}_{t}]_i -[\hat{\vw}^{*}]_i  )^2}
\end{align}
Suppose the monotone increasing sequences $\{b_t\}_{t=0}^{\infty}$  and $\{ \widetilde{b}_t\}_{t=0}^{\infty}$ satisfy  $b_{0}\geq \eta \bar{\lambda}_1$ and $ \widetilde{b}_{0}\geq \eta \bar{\lambda}_1$, then the upper bounds are
\begin{align*}
\lim_{t \rightarrow \infty}b_{t}\leq  \eta\bar{\lambda}_1 +  \frac{\| \vX\left(\vw_{0}-\vw^*\right)\|^2}{\eta}\quad\text{and} \quad  \lim_{t \rightarrow \infty}\widetilde{b}_{t}\leq  \eta\bar{\lambda}_1  +  \frac{\| \vw_{0}-\vw^*\|^2}{\eta}. 
\end{align*}
\end{lem}
\begin{proof}
Our strategy to obtain the exponential decay in $\| \vw_{t+1}-\vw^*\|^2$ for $t>M_2$ is to ensure that the limit of  $b_{t} $ is a finite number  as in \cite{ward2018adagrad} and \cite{wu2018wngrad}. However, the difference is that we can  bound it concretely for each dimension. We first look at $b_t$.
For $t>M_2$, following the update of $b_t$ in \eqref{eq:update_u},
\begin{align} 
b_{t+1}
 &= b_{t} +\frac{1}{b_{t}+b_{t+1}} {\sum_{i=1}^n\bar{\lambda}_i^2([\hat{\vw}_{t} -\hat{\vw}^{*}]_i  )^2} \nonumber\\
&\leq  \eta \bar{\lambda}_1 + \sum_{i=1}^n\bar{\lambda}_i^2\sum_{\ell=0}^t\frac{1}{b_{\ell+1}}([\hat{\vw}_{\ell} -\hat{\vw}^{*}]_i  )^2
\label{eq:bt}
\end{align}
Now we can bound the second term of the above inequality by using \eqref{eq:exact-dynamic}. For $i\in[n]$, we have
\begin{align}
\left([\hat{\vw}_{t+1}-\hat{\vw}^{*}]_i\right)^2  =&([\hat{\vw}_{t}-\hat{\vw}^{*}]_i  )^2 -\frac{\eta \bar{\lambda}_i}{b_{t+1}}\left( 2 - \frac{\eta \bar{\lambda}_i}{b_{t+1}} \right)([\hat{\vw}_{t}-\hat{\vw}^{*}]_i  )^2 \nonumber\\
\leq & ([\hat{\vw}_{t}-\hat{\vw}^{*}]_i  )^2 -\frac{\eta \bar{\lambda}_i}{b_{t+1}}([\hat{\vw}_{t}-\hat{\vw}^{*}]_i  )^2\nonumber\\
\leq & ([\hat{\vw}_{0}-\hat{\vw}^{*}]_i  )^2 - \sum_{\ell=0}^t\frac{\eta \bar{\lambda}_i}{b_{\ell+1}}([\hat{\vw}_{t}-\hat{\vw}^{*}]_i  )^2\nonumber.
\end{align}
Let $t \rightarrow \infty$ and arrange the inequality, 
\begin{align}
 \sum_{\ell=0}^\infty \frac{1}{b_{\ell+1}}([\hat{\vw}_{t}-\hat{\vw}^{*}]_i  )^2 \leq & \frac{1}{\eta  \bar{\lambda}_i} ([\hat{\vw}_{0}-\hat{\vw}^{*}]_i  )^2 \label{eq:w-i}
\end{align}
Returning to \eqref{eq:bt}, we can obtain the finite upper bound of $b_{t}$ by
\begin{align} 
b_{t+1}
 &= b_{t} +\frac{1}{b_{t}+b_{t+1}} {\sum_{i=1}^n\bar{\lambda}_i^2([\hat{\vw}_{t} -\hat{\vw}^{*}]_i  )^2} \nonumber\\
&\leq  \eta \bar{\lambda}_1 + \frac{1}{\eta} \sum_{i=1}^n\bar{\lambda}_i ([\hat{\vw}_{0}-\hat{\vw}^{*}]_i  )^2
\end{align}

Similarly, we have the upper bound for $\widetilde{b}_t$ by only identifying 
\begin{align} 
\widetilde{b}_{t+1}
&\leq  \eta \bar{\lambda}_1 + \sum_{i=1}^n{\bar{\lambda}}_i\sum_{\ell=0}^t\frac{1}{b_{\ell+1}}([\hat{\vw}_{\ell} -\hat{\vw}^{*}]_i  )^2 
\end{align}
which implies
\begin{align} 
\lim_{t\to \infty} \widetilde{b}_{t+1}
&\leq  \eta \bar{\lambda}_1 + \sum_{i=1}^n  \frac{1}{\eta} ([\hat{\vw}_{0}-\hat{\vw}^{*}]_i  )^2 
\end{align}

\end{proof}

\begin{thm}
\textbf{(AdaGrad-Norm Convergence Rate  if $b_0 \geq \frac{\eta\bar{\lambda}_1}{ 2n}$)}
Consider  the problem \eqref{eq:linear}  and the adaptive gradient descent  method \eqref{eq:adalinear}.  Suppose there exists the unique $\vw^*$ such that $\vy=\vX\vw^*$ and suppose $\vX^T\vX$ is a positive definite matrix  with singular value decomposition $\vX^T\vX = \vV\vect{\Sigma} \vV^\top$  where  the diagonal values of $\vect{\Sigma}$ have the arrangement $  \bar{\lambda}_1\geq \bar{\lambda}_2\geq\ldots\geq  \bar{\lambda}_n>0$. Assume $\| \vw_{0}-\vw^*\|^2>\epsilon$. If $b_0 \geq \eta\bar{\lambda}_1/2$,   we have $\| \vw_{T}-\vw^*\|^2\leq \epsilon$ for
\begin{align}
T= \max \left\{ \widetilde{T},  \frac{\bar{\lambda}_1}{\bar{\lambda}_n}\left( 1+  \frac{\| \vX \left(\vw_{0}-\vw^*\right)\|^2}{\bar{\lambda}_1 \eta^2} \right)\right\}\log\left( \frac{\| \vw_{0}-\vw^*\|^2}{\epsilon}\right)
\end{align}
\begin{align*}
\text{where } \quad \widetilde{T} &:= \max \left\{ \frac{ b_1^2}{\bar{\lambda}_n\left(2b_1-\eta \bar{\lambda}_1\right)} \mathbbm{1}_{ \left\{ b_1\in\left(\frac{ \eta \bar{\lambda}_1}{2}, \frac{ \eta( \bar{\lambda}_1+\bar{\lambda}_n) }{2}\right) \right\} },   \frac{\bar{\lambda}_1+\bar{\lambda}_n}{ 2 \bar{\lambda}_n } \right\}\\
 \text{with }  &\quad b_1 = \sqrt{b_0^2 +\|\vX^T\vX(\vw_0-\vw^*)\|^2} > b_0\geq \frac{\eta \bar{\lambda}_1}{2} 
\end{align*}
%

\end{thm}
\begin{proof}
We have that $b_t$ eventually stabilizes  in following 3 regions: 
 (a) $ \frac{\eta\bar{\lambda}_1}{2}\leq b_t< \frac{ \eta \left(\bar{\lambda}_1+\bar{\lambda}_n\right)}{2},$
(b) $  \frac{ \eta \left(\bar{\lambda}_1+\bar{\lambda}_n\right)}{2}\leq b_t \leq  \bar{\lambda}_1,$
and   (c) $ b_t\geq \eta \bar{\lambda}_1.$


For any $t$ such that $\frac{\eta\bar{\lambda}_1}{2}\leq b_t\leq  \frac{ \eta \left(\bar{\lambda}_1+\bar{\lambda}_n\right)}{2}$, inequality \eqref{eq:ineq2} implies
\begin{align}
\| \vw_{t+1}-\vw^*\|^2
&\leq \left( 1-\frac{2\eta \bar{\lambda}_n}{b_{t+1}} \left(1- \frac{\eta \bar{\lambda}_1}{2b_{t+1}}\right) \right)\| \vw_{t}-\vw^*\|^2 \nonumber\\
&\leq  \left( 1-\frac{2 \eta \bar{\lambda}_n}{b_1} \left(1- \frac{\eta \bar{\lambda}_1}{2b_{1}}\right) \right)\| \vw_{t}-\vw^*\|^2 \nonumber\\
&\leq \exp\left(-\frac{2\eta \bar{\lambda}_n}{b_1}  \left(1-\frac{\eta \bar{\lambda}_1}{2b_{1}}\right)\right)\| \vw_{t}-\vw^*\|^2\label{eq:first-bound}.
\end{align}
Thus, after iterations
\begin{align*}
T_0 =  \frac{ b_0^2 +\|\vX^T\vX(\vw_0-\vw^*)\|^2}{ \eta\bar{\lambda}_n\left(2 \sqrt{ b_0^2+\|\vX^T\vX(\vw_0-\vw^*)\|^2}-\eta \bar{\lambda}_1\right)}\log\left( \frac{\| \vw_{0}-\vw^*\|^2}{\epsilon}\right),\quad \text{if}\quad b_{T_0}\leq \frac{\eta (\bar{\lambda}_1+\bar{\lambda}_n)}{2},
\end{align*}
then $\| \vw_{T_0}-\vw^*\|^2\leq \epsilon$. However, after $T_0$ iterations, it is possible that $b_{T_0}> \frac{\eta (\bar{\lambda}_1+\bar{\lambda}_n)}{2}$.  We then continue to analyze case (b) and then (c).\\
Suppose there exists an index $M_1$ such that  $ \frac{ \eta \left(\bar{\lambda}_1+\bar{\lambda}_n\right)}{2}< b_{t+1} \leq \eta\bar{\lambda}_1$  for all $t\geq M_1$. The inequality  \eqref{eq:ineq3} reduces to
\begin{align*}
\| \vw_{t+1}-\vw^*\|^2
&\leq  \left( 1- \frac{2\eta \bar{\lambda}_1 \bar{\lambda}_n}{ b_{t+1}(\bar{\lambda}_n+\bar{\lambda}_1)}\right)\|\vw_{t}-\vw^*\|^2 \nonumber\\
&\leq  \left( 1-\frac{2\bar{\lambda}_n}{ (\bar{\lambda}_n+\bar{\lambda}_1)}\right) ^{t-M_1}\|\vw_{M_1}-\vw^*\|^2 \nonumber\\
&\leq  \left( 1-  \frac{2 \bar{\lambda}_n}{ (\bar{\lambda}_n+\bar{\lambda}_1)}\right) ^{t-M_1}\left(1-\frac{2\bar{\lambda}_n}{b_1}  \left( 1-\frac{\eta \bar{\lambda}_1}{2b_{1}}\right)\right)^{M_1}\|\vw_{0}-\vw^*\|^2 \nonumber\\
 &\leq \left( 1 - \frac{1}{D_1}\right)^{t} \| \vw_{0}-\vw^*\|^2 
\end{align*}
where at the last step we  define the $D_1 $ as follows
$$ D_1:=  \max \left\{ \frac{ b_1^2}{\bar{\lambda}_n\left(2b_1-\eta \bar{\lambda}_1\right)} \mathbbm{1}_{ \left\{ b_1\in\left(\frac{ \eta \bar{\lambda}_1}{2}, \frac{ \eta( \bar{\lambda}_1+\bar{\lambda}_n) }{n}\right) \right\} },   \frac{\bar{\lambda}_1+\bar{\lambda}_n}{ 2 \bar{\lambda}_n } \right\}$$
It  implies  that after 
\begin{align*}
T_1 = D_1\log\left( \frac{\| \vw_{0}-\vw^*\|^2}{\epsilon}\right),\quad \text{if }b_{T_1} \in \left(\frac{ \eta \left(\bar{\lambda}_1+\bar{\lambda}_n\right)}{2},\eta\bar{\lambda}_1\right),
\end{align*}
then $\| \vw_{T_1+1}-\vw^*\|^2\leq \epsilon$. However, after $T_1$, it is possible that $b_{T_1}> \eta \bar{\lambda}_1$.  We then continue to analyze (c). By Lemma \ref{lem:b-max}, we have
\begin{align*}
\lim_{t \rightarrow \infty}b_{t}\leq  \eta\bar{\lambda}_1 \left(1 +  \frac{1}{\eta^2}\frac{\| \vw_{0}-\vw^*\|^2}{\bar{\lambda}_1}\right) \triangleq b_{\max} 
\end{align*}
Now, recall  the contraction formula \eqref{eq:ineq2} , we have $t>M_2$
\begin{align*}
\| \vw_{t+1}-\vw^*\|^2
&\leq\left( 1-\frac{2\eta \bar{\lambda}_n}{b_{t+1}} \left(1- \frac{\eta \bar{\lambda}_1}{2b_{t+1}}\right) \right)\| \vw_{t}-\vw^*\|^2 \nonumber\\ 
&\leq \left( 1 - \frac{\eta \bar{\lambda}_n}{b_{\max} }\right)\| \vw_{t}-\vw^*\|^2\quad \text{ since } b_{\max} \geq b_{t+1}\geq \eta \bar{\lambda}_1\nonumber\\ 
 &\leq \left( 1 - \frac{\eta \bar{\lambda}_n}{b_{\max} }\right)^{t-M_2}\| \vw_{M_2}-\vw^*\|^2 \\
  &\leq \left( 1 - \frac{\eta \bar{\lambda}_n}{b_{\max} }\right)^{t-M_2}\left(1- \frac{1}{D_1}\right)^{M_2}\| \vw_{0}-\vw^*\|^2 \\
  &\leq \left( 1 - \frac{1}{D_2}\right)^{t} \| \vw_{0}-\vw^*\|^2 
\end{align*}
where at the last step we  define the $D_2 $ as follows
$$ D_2 : =\max \left\{D_1,  \frac{\bar{\lambda}_1}{\bar{\lambda}_n}\left( 1+  \frac{\| \vw_{0}-\vw^*\|^2}{\bar{\lambda}_1 \eta^2} \right)\right\}$$
Therefore, we have that $\| \vw_{T_2+1}-\vw^*\|^2\leq \epsilon$ after
\begin{align*}
T_2 = \frac{D_2}{2} \log\left( \frac{\| \vw_{0}-\vw^*\|^2}{\epsilon}\right)
\end{align*}
Note that the above bounds satisfy $T_2 \geq T_1 \geq T_0$, thus we have the statement.
\end{proof}

In the remaining section, we discuss the convergence of $b_0\leq \eta \bar{\lambda}_1/2$. We first prove that it only takes $\mathcal{O}\left( 1\right)$ steps for $b_t$ to reach the critical value $\frac{\eta \bar{\lambda}_1}{2}$ if $b_0\leq \bar{\lambda}_1/2$. Then we explicitly characterize the increase of $\|\vw_t -\vw^*\|$ for $t\in \{\ell, b_\ell\leq \bar{\lambda}_1/2\}$ and conclude the linear convergence.
\begin{lem} \textbf{ (Restatement of Lemma \ref{lem:increase_sub}, Exponential Increase for  $b_0 < \frac{\eta\bar{\lambda}_1}{ 2n}$)}
Consider  the problem \eqref{eq:linear}  and the adaptive gradient descent  methods:
\begin{align}
&\text{(AdaGrad-Norm) } b_{t+1}^2 =b_{t}^2 +{\sum_{i=1}^n\bar{\lambda}_i^2([\hat{\vw}_{t}]_i -[\hat{\vw}^{*}]_i  )^2};  
\\
&\text{(AdaLoss) } 
 \widetilde{b}_{t+1}^2 = \widetilde{b}_{t}^2 +{\sum_{i=1}^n{\bar{\lambda}}_i([\hat{\vw}_{t}]_i -[\hat{\vw}^{*}]_i  )^2}
 \label{eq:update_u}
\end{align}
Suppose there exists the unique $\vw^*$ such that $\vy=\vX\vw^*$ and suppose $\vX^T\vX$ is a positive definite matrix  with singular value decomposition $\vX^T\vX = \vV\vect{\Sigma} \vV$  where  the diagonal values of $\vect{\Sigma}$ have the arrangement $  \bar{\lambda}_1\geq \bar{\lambda}_2\geq\ldots\geq  \bar{\lambda}_n>0$.   \begin{itemize}
 \item \textbf{AdaGrad-Norm.} Suppose that we start with small initialization: $0<{b}_0< \bar{\lambda}_1/2$, then there exists the first index ${{N}}$ such that ${b}_{{{N}}+1}\geq \bar{\lambda}_1/2$ and ${b}_{{N}}< \bar{\lambda}_1/2$, and  $N$  satisfies
\begin{align}
 N \leq \log  \left(  1+
\frac{ \left({\eta \bar{\lambda}_1} \right)^2 - 4(b_0)^2 }{\eta^2 \bar{\lambda}_1^{2} }
\right) / \log \left (1+   \frac{ 4}{ \eta^2 } \left([\hat{\vw}_{0}]_1 -[\hat{\vw}^{*}]_1 \right)^2 \right)  +1\label{eq:TheN1}
\end{align}
    \item \textbf{AdaLoss.} Suppose that we start with small initialization: $0<\widetilde{b}_0< \bar{\lambda}_1/2$, then there exists the first index ${\widetilde{N}}$ such that $\widetilde{b}_{{\widetilde{N}}+1}\geq \bar{\lambda}_1/2$ and $\widetilde{b}_{\widetilde{N}}< \bar{\lambda}_1/2$, and  $N$  satisfies
\begin{align}
\widetilde{N}\leq \bigg \lceil{ \log  \left(  1+
\frac{ \left({\eta \bar{\lambda}_1} \right)^2 - 4b_0^2 }{\eta^2 \bar{\lambda}_1 }
\right) / \log \left (1+   \frac{ 4}{ \eta^2 \bar{\lambda}_1 } \left([\vV^T\vw_{0}]_1 -[\vV^T\vw^{*}]_1 \right)^2 \right)\bigg \rceil} +1 \label{eq:TheN2}
\end{align}
\end{itemize}


\end{lem}
\begin{proof}
%
Starting with the fact in \eqref{eq:exact-dynamic}
\begin{align}
s_{t+1}^{(i)} :=\left([\hat{\vw}_{t+1}]_i-[\hat{\vw}^{*}]_i\right)^2  =&\left( 1 - \frac{\eta \bar{\lambda}_i}{b_{t+1}} \right)^2([\hat{\vw}_{t}]_i -[\hat{\vw}^{*}]_i  )^2. \label{eq:begin}
\end{align}
Observe that each $i$th sequence $\{s_{t}^{(i)}\}_{t=0}^k$ is monotonically increasing up to  $b_{k} \leq \frac{\eta \bar{\lambda}_i}{2}$.
In particular, if  $b_{0} \leq \frac{\eta \bar{\lambda}_1}{2}$ and after some $N$ iterations such that $b_{N+1} \geq \frac{\eta \bar{\lambda}_1}{2}$ and $b_{N} \leq \frac{\eta \bar{\lambda}_1}{2}$, then we have the increasing sequence  of $\left([\hat{\vw}_{t}]_1-[\hat{\vw}^{*}]_1\right)^2, t =0, 1,\ldots,N$ not including  $\left([\hat{\vw}_{N+1}]_1-[\hat{\vw}^{*}]_1\right)^2$.  
We can estimate the number of iterations $N$ for $b_t$ to grow to $b_{N+1}\geq\frac{\eta \bar{\lambda}_1}{2}$ from $b_0<\frac{\eta \bar{\lambda}_1}{2}$ and  $b_N< \frac{\eta \bar{\lambda}_1}{2}$   as follows
\begin{align}
\left( \frac{\eta \bar{\lambda}_1}{2} \right)^2 > b_{N}^2  & = b_{0}^2 + \sum_{t=0}^{N}  {\sum_{i=1}^n \bar{\lambda}_i^2([\hat{\vw}_{t}]_i -[\hat{\vw}^{*}]_i  )^2}\notag\\
& \geq b_{0}^2 + \sum_{t=0}^{N} \bar{\lambda}_1^2 \left([\hat{\vw}_{t}]_1 -[\hat{\vw}^{*}]_1 \right)^2\notag\\  
&= b_{0}^2 + \bar{\lambda}_1^2\left([\hat{\vw}_{0}]_1 -[\hat{\vw}^{*}]_1 \right)^2 \sum_{t=0}^{N}  \prod_{\ell=0}^t  \left( 1 - \frac{\eta \bar{\lambda}_1}{ b_{\ell}} \right)^{2}  \nonumber\\
&\geq b_{0}^2 + \bar{\lambda}_1^2 \left([\hat{\vw}_{0}]_1 -[\hat{\vw}^{*}]_1 \right)^2 \sum_{t=0}^{N-1}  \left( 1 - \frac{\eta \bar{\lambda}_1}{ b_{N-1}} \right)^{2t} 
 \label{eq:sum_term}
\end{align}
Since we have that
\begin{align}
&\left(\eta \bar{\lambda}_1\right)^2  >  4b_N^2  \geq  4b_{N-1}^2+  4\bar{\lambda}_1^2 \left( [\hat{\vw}_{N-1}]_1 -[\hat{\vw}^{*}]_1 \right)^2 \nonumber\\
\Rightarrow \left( \frac{\eta \bar{\lambda}_1}{ b_{N-1}} \right)^2  
&  \geq 4+\frac{4}{b_{N-1}^2} \bar{\lambda}_1^2 \left([\hat{\vw}_{N-1}]_1 -[\hat{\vw}^{*}]_1 \right)^2
 \geq 4+   \frac{ 16}{ \eta^2 } \left([\hat{\vw}_{0}]_1 -[\hat{\vw}^{*}]_1 \right)^2  \label{eq:b_n-1}
\end{align}

Thus, we can further lower bound the term of the  summation in inequality  \eqref{eq:sum_term}  by first observing that
\begin{align}
\left( 1 - \frac{\eta \bar{\lambda}_1}{ b_{N-1}} \right)^{2} 
&= 1 + \frac{\eta \bar{\lambda}_1}{ b_{N-1}} \left(  \frac{\eta \bar{\lambda}_1}{ b_{N-1}} -2\right)   \nonumber\\
 & \geq  1 + \frac{\eta \bar{\lambda}_1}{ b_{N-1}} \left( \sqrt{4+   \frac{ 16}{ \eta^2  } \left([\hat{\vw}_{0}]_1 -[\hat{\vw}^{*}]_1 \right)^2} -2 \right)    \text{ by inequality }\eqref{eq:b_n-1}    \nonumber\\
  & \geq  \sqrt{1+   \frac{ 4}{ \eta^2} \left([\hat{\vw}_{0}]_1 -[\hat{\vw}^{*}]_1 \right)^2}    \quad \text{ by inequality } b_{N-1}\leq \frac{\eta \bar{\lambda}_1}{2}    \nonumber
\end{align}
 Applying above inequality, we have the summation in  \eqref{eq:sum_term}  explicitly  lower bounded by
\begin{align}
\sum_{t=0}^{N-1}\left( 1 - \frac{\eta  \bar{\lambda}_1}{ b_{N-1}} \right)^{2t} 
& \geq 
\sum_{t=0}^{N-1}\left(1+   \frac{ 4}{ \eta^2  } \left([\hat{\vw}_{0}]_1 -[\hat{\vw}^{*}]_1 \right)^2\right)^{t}   \nonumber\\
& \geq  \frac{  \eta^2}{4\left([\hat{\vw}_{0}]_1 -[\hat{\vw}^{*}]_1 \right)^2} \left(      \left(1+   \frac{ 4}{ \eta^2  } \left([\hat{\vw}_{0}]_1 -[\hat{\vw}^{*}]_1 \right)^2\right)^{N}       -1\right) 
\label{eq:summation}
\end{align}
Plugging above bound of summation \eqref{eq:summation} back to the \eqref{eq:sum_term}, we can estimate $N$:
\begin{align*}
N \leq \log  \left(  1+
\frac{ \left({\eta \bar{\lambda}_1} \right)^2 - 4(b_0)^2 }{\eta^2 \bar{\lambda}_1^2 }
\right) / \log \left (1+   \frac{ 4}{ \eta^2 } \left([\hat{\vw}_{0}]_1 -[\hat{\vw}^{*}]_1 \right)^2 \right) 
\end{align*}
Replacing $\hat{\vw}_{0}=\vV^T\vw_0$ gives the bound  \eqref{eq:TheN1}.
Following the same process for AdaLoss, we could get the bound quickly.

\end{proof}
Now we will prove the convergence for $b_t$ start from a very small initialization  $b_0\leq \eta \bar{\lambda}_1/2$
\begin{thm}
\label{lem:conv-small}\textbf{ (AdaGrad-Norm Convergence for the case $b_0 < \frac{\eta\bar{\lambda}_1}{ 2n}$)}
Under the same setup as Lemma \ref{lem:increase_sub}.  If $b_0<\eta\bar{\lambda}_1/2$, then there exists an index $s\leq \infty$ such that $b_s\geq \eta\bar{\lambda}_1/2$ and $b_{s-1}<\eta\bar{\lambda}_1/2$ and $s$ satisfies 
$$s \leq 2 \log  \left(  1+
\frac{ \left({\eta \bar{\lambda}_1} \right)^2 - 4(b_0)^2 }{\eta^2  \bar{\lambda}_1^2 }
\right) / \log \left (1+   \frac{ 4}{ \eta^2 } \left([\hat{\vw}_{0}]_1 -[\hat{\vw}^{*}]_1 \right)^2 \right).$$ We have $\| \vw_{T+s}-\vw^*\|^2\leq \epsilon$ for  $T$ can be explicitly expressed as 
\begin{align*}
T = &\max \left\{ \frac{\widetilde{T} }{\bar{\lambda}_n}, \frac{\bar{\lambda}_1}{\bar{\lambda}_n}\left( 1+  \frac{ \| \vX \left( \vw_{0}-\vw^*\right)\|^2}{\eta^2 \bar{\lambda}_1}+\log\left(\frac{ \eta \bar{\lambda}_1}{2b_0}\right) \right)\right\}\log\left(  \frac{\|\vw_{0}-\vw^*\|^2}{\bar{\lambda}_n}+\eta^2 \log\left(\frac{ \eta \bar{\lambda}_1}{2b_0}\right)\right) 
\end{align*}
where  $ \widetilde{T} := \max \left\{ \left( \eta\bar{\lambda}_n+5\eta \bar{\lambda}_1  \right) (\delta_s+1) \mathbbm{1}_{\left\{ b_s\in\left(\frac{ \eta \bar{\lambda}_1}{2}, \frac{ \eta( \bar{\lambda}_1+\bar{\lambda}_n) }{2}\right) \right\}}  ,     \frac{\bar{\lambda}_1+\bar{\lambda}_n}{ 2}  \right\} $ 
and
$$
 \delta_s 
 =
  \frac{\eta^2 \left( \bar{\lambda}_1+\bar{\lambda}_n\right)^4\left( 1+\frac{b_0^2}{ ([\vV^\top {\vw}_{0}]_1 -[\vV^\top {\vw}^{*}]_1  )^2  } \right)} 
  {4\bar{\lambda}_1 \left( \bar{\lambda}_1-\bar{\lambda}_n\right)^2 \left( \eta\bar{\lambda}_1-\sqrt{b_0^2+([\vV^\top {\vw}_{0}]_1 -[\vV^\top {\vw}^{*}]_1  )^2 } \right)^2 }.
  $$

\end{thm}
\begin{proof}
Continuing with Lemma \ref{lem:increase_sub}, we have after $s=N$ steps such that
$b_{N+1}\geq \frac{\eta \bar{\lambda}_1}{2}$ and $b_{N}< \frac{\eta \bar{\lambda}_1}{2}$ from  $b_0 < \frac{\eta\bar{\lambda}_1}{ 2}$. Then we have contraction as in Lemma \ref{lem:conv-small}, for $t\geq N+1$
$$\| \vw_{t+1}-\vw^*\|^2
\leq\left( 1-\frac{2\eta \bar{\lambda}_n}{b_{t+1}} \left(1- \frac{\eta \bar{\lambda}_1}{2b_{t+1}}\right) \right)\| \vw_{t}-\vw^*\|^2.$$
Starting at $N+1$, we have $\|\vw_{N+1+T}-\vw^*\|\leq \epsilon$ by Lemma \ref{lem:conv-small}  where $T$ satisfies
\begin{align}
T = &\max \left\{\widetilde{T}, \frac{\bar{\lambda}_1}{\bar{\lambda}_n}\left( 1+  \frac{\| \vX \left(\vw_{N+1}-\vw^*\right)\|^2}{\bar{\lambda}_1 \eta^2} \right)\right\}\log\left( \frac{\| \vw_{N+1}-\vw^*\|^2}{\epsilon}\right) \label{eq:fill-need}\\
\quad \text{where } \quad \widetilde{T} &:= \max \left\{ \frac{ (b_{N+2})^2}{\bar{\lambda}_n\left(2b_{N+2}-\eta \bar{\lambda}_1\right)} \mathbbm{1}_{ \left\{ b_{N+2}\in\left(\frac{ \eta \bar{\lambda}_1}{2}, \frac{ \eta( \bar{\lambda}_1+\bar{\lambda}_n) }{2}\right) \right\} },    \frac{\bar{\lambda}_1+\bar{\lambda}_n}{ 2 \bar{\lambda}_n }  \right\}  \nonumber
\end{align}

where $b_{N+2} = \sqrt{b_{N+1}^2 +\|\vX^T\vX(\vw_{N+1}-\vw^*)\|^2} >\frac{\eta \bar{\lambda}_1}{2}$ assuming $\|\vX^T\vX(\vw_{N+1}-\vw^*)\|^2\neq0$.
However, we need to quantify  the convergence  rate in terms of the small initialization $b_0$ and $\| \vw_{0}-\vw^*\|^2$ instead of  $N$. That is, we need the following bounds
\begin{itemize}
\item[(a)] the upper bound for $ \|\vX(\vw_{N+1}-\vw^*)\|^2$  and $ \|\vw_{N+1}-\vw^*\|^2$;
\item [(b)] the upper bound for $\frac{ (b_{N+2})^2}{\left(2b_{N+2}-\eta \bar{\lambda}_1\right)}\mathbbm{1}{ \left\{ b_{N+2}\in\left( \frac{\eta \bar{\lambda}_1}{2} ,\frac{ \eta( \bar{\lambda}_1+\bar{\lambda}_n) }{2}\right) \right\}   }$.
\end{itemize}

\paragraph{ (a) Upper bound for $ \|\vX(\vw_{N+1}-\vw^*)\|^2$  and $ \|\vw_{N+1}-\vw^*\|^2$  }   As in \cite{ward2018adagrad} and \cite{wu2018wngrad}, we first estimate the upper bounds using Lemma 3.2 in  \cite{ward2018adagrad}. From \eqref{eq:ineq3}
\begin{align*}
\| \vX\left( \vw_{t+1}-\vw^*\right)\|^2
&\leq \| \vX \left( \vw_{t}-\vw^*\right)\|^2+ \frac{\eta^2\bar{\lambda}_1}{b_{t+1}^2} \| \vX^\top \vX \left(\vw_{t}-\vw^*\right)\|^2   \\
&\leq \| \vX \left( \vw_{0}-\vw^*\right)\|^2+ \sum_{\ell=0}^{t} \frac{\eta^2\bar{\lambda}_1}{b_{\ell+1}^2} \| \vX^\top \vX \left(\vw_{\ell}-\vw^*\right)\|^2   \\
&\leq \| \vX \left( \vw_{0}-\vw^*\right)\|^2+ \eta^2\bar{\lambda}_1\log\left( b_{t+1}/b_0\right)  
\end{align*}
where in the last inequality we use   Lemma 6 in \cite{ward2018adagrad}.
Thus, we have  for $b_N<\frac{ \eta \bar{\lambda}_1}{2}$ and $b_{N+1}\geq\frac{ \eta \bar{\lambda}_1}{2}$
\begin{align}
\| \vX\left( \vw_{N+1}-\vw^*\right)\|^2\leq \| \vX\left( \vw_{N}-\vw^*\right)\|^2\leq  \| \vX \left( \vw_{0}-\vw^*\right)\|^2+ \eta^2\bar{\lambda}_1\log\left(\frac{ \eta \bar{\lambda}_1}{2b_0}\right) \label{eq:lower-bound1}.
\end{align} 
We have from \eqref{eq:ineq2}
\begin{align}
\|  \vw_{t+1}-\vw^*\|^2
&\leq \|\vw_{0}-\vw^*\|^2+ \sum_{\ell=0}^{t} \frac{\eta^2}{b_{\ell+1}^2} \| \vX^T\vX\left( \vw_{\ell}-\vw^*\right)\|^2  \nonumber  \\
&\leq \| \vw_{0}-\vw^*\|^2+ \sum_{\ell=0}^{t} \frac{\eta^2}{b_{\ell+1}^2} \| \vX^T\vX\left( \vw_{\ell}-\vw^*\right)\|^2  \nonumber  \\
&\leq \|  \vw_{0}-\vw^*\|^2+ \eta^2\log\left( b_{t+1}/b_0\right)  \nonumber \\
\Rightarrow \|\vw_{N+1}-\vw^*\|^2&\leq \|\vw_{N}-\vw^*\|^2 \leq \| \vw_{0}-\vw^*\|^2+\eta^2 \log\left(\frac{ \eta \bar{\lambda}_1}{2b_0}\right). \label{eq:lower-bound2}
\end{align}

\paragraph{(b) Upper bound for $\frac{ (b_{N+2})^2}{\left(2b_{N+2}-\eta \bar{\lambda}_1\right)}\mathbbm{1}{ \left\{ b_{N+2}\in\left( \frac{\eta \bar{\lambda}_1}{2} ,\frac{ \eta( \bar{\lambda}_1+\bar{\lambda}_n) }{2}\right) \right\}   }$} We first estimate the upper bound  based on 
\begin{align}
\left(\frac{\eta \bar{\lambda}_1}{2} \right)^2\leq b_{N+1}^2
= &b_{N+2}^2- \| \vX^\top \vX\left(\vw_{N+1}-\vw^* \right) \|^2  \label{eq:lower1} \\
\left([\hat{\vw}_{N+1}]_1-[\hat{\vw}^{*}]_1\right)^2 & = \left(1-\frac{\eta \bar{\lambda}_1}{b_{N+1} } \right)^2 ([\hat{\vw}_{N}]_1 -[\hat{\vw}^{*}]_1  )^2   \label{eq:lower2}
 \end{align}
Observe the lower bound for 
\begin{align}
\frac{ \| \vX^\top \vX\left(\vw_{N+1}-\vw^* \right) \|^2 }{b^2_{N+2}}
&= \frac{  \| \vX^\top \vX\left(\vw_{N+1}-\vw^* \right) \|^2 }{b_{N+1}^2 + \| \vX^\top \vX\left(\vw_{N+1}-\vw^* \right) \|^2 } \nonumber\\
&\geq   \frac{  \bar{\lambda}_1\left([\hat{\vw}_{N+1}]_1 -[\hat{\vw}^{*}]_1 \right)^2}{b_{N+1}^2 +  \bar{\lambda}_1 \left([\hat{\vw}_{N+1}]_1 -[\hat{\vw}^{*}]_1 \right)^2 } \nonumber\\
&= \frac{  \bar{\lambda}_1\left( 1 - \frac{\eta \bar{\lambda}_1}{nb_{N+1}} \right)^2([\hat{\vw}_{N}]_1 -[\hat{\vw}^{*}]_1 )^2}{b_{N+1}^2 +  \bar{\lambda}_1 \left( 1 - \frac{\eta \bar{\lambda}_1}{nb_{N+1}} \right)^2([\hat{\vw}_{N}]_1 -[\hat{\vw}^{*}]_1 )^2 }  \nonumber\\
&\geq \frac{  \bar{\lambda}_1([\hat{\vw}_{N}]_1 -[\hat{\vw}^{*}]_1 )^2}{ \frac{\eta^2 \left( \bar{\lambda}_1+\bar{\lambda}_n\right)^4}{4\left( \bar{\lambda}_1-\bar{\lambda}_n\right)^2}+  \bar{\lambda}_1([\hat{\vw}_{N}]_1 -[\hat{\vw}^{*}]_1 )^2 } 
\label{eq:lowerbound}
\end{align}
where the last equality is due to 
\begin{align*}
\frac{b_{N+1}^2}{ \left( \frac{\eta \bar{\lambda}_1}{b_{N+1}}-1 \right)^2} 
& =  \frac{b_{N+1}^4}{ \left(\eta \bar{\lambda}_1-b_{N+1} \right)^2} 
\leq \frac{\frac{\eta^4 (\bar{\lambda}_1+\bar{\lambda}_n)^4}{16} }{ \left( \eta \bar{\lambda}_1 -\frac{\eta (\bar{\lambda}_1+\bar{\lambda}_n)}{2}\right)^2} = \frac{\eta^2 \left( \bar{\lambda}_1+\bar{\lambda}_n\right)^4}{4\left( \bar{\lambda}_1-\bar{\lambda}_n\right)^2}
\end{align*} 
Putting the lower bound of \eqref{eq:lowerbound}  back to inequality \eqref{eq:lower1} and arranging the  inequality, we get
\begin{align}
\frac{1}{2b_{N+2}}\left(1-\frac{\eta \bar{\lambda}_1}{2b_{N+2}} \right)
\geq &\frac{\| \vX^\top \vX\left(\vw_{N+1}-\vw^* \right) \|^2}{2b_{N+2}^2 \left(b_{N+2}+2\eta \bar{\lambda}_1 \right)} \nonumber\\
\geq & \frac{  \bar{\lambda}_1([\hat{\vw}_{N}]_1 -[\hat{\vw}^{*}]_1 )^2}{ 2\left(\frac{ \eta \left(\bar{\lambda}_n+\bar{\lambda}_1\right)}{2}+2\eta \bar{\lambda}_1  \right) \left( \frac{\eta^2 \left( \bar{\lambda}_1+\bar{\lambda}_n\right)^4}{4\left( \bar{\lambda}_1-\bar{\lambda}_n\right)^2} +  \bar{\lambda}_1([\hat{\vw}_{N}]_1 -[\hat{\vw}^{*}]_1 )^2 \right) } \nonumber\\
\Rightarrow  \frac{ (b_{N+2})^2}{\left(2b_{N+2}-\eta \bar{\lambda}_1\right)}
\leq & \left( \eta\bar{\lambda}_n+5\eta \bar{\lambda}_1  \right)   \left( \frac{\eta^2 \left( \bar{\lambda}_1+\bar{\lambda}_n\right)^4}{4\left( \bar{\lambda}_1-\bar{\lambda}_n\right)^2\bar{\lambda}_1([\hat{\vw}_{N}]_1 -[\hat{\vw}^{*}]_1 )^2} +  1 \right)  \nonumber \\
\leq &\left( \eta\bar{\lambda}_n+5\eta \bar{\lambda}_1  \right) \left( \delta_s+1\right)
  \label{eq:last}
 \end{align}
 where we let $\delta_s =  \frac{\eta^2 \left( \bar{\lambda}_1+\bar{\lambda}_n\right)^4  \left( 1+\frac{b_0^2}{ ([\hat{\vw}_{0}]_1 -[\hat{\vw}^{*}]_1 )^2 } \right)} 
  {4\bar{\lambda}_1 \left( \bar{\lambda}_1-\bar{\lambda}_n\right)^2 \left( \eta\bar{\lambda}_1-\sqrt{b_0^2+ ([\hat{\vw}_{0}]_1 -[\hat{\vw}^{*}]_1 )^2} \right)^2 }$
as in the last inequality we note that  for $b_t\leq \frac{\eta \bar{\lambda}_1}{2} , t \in [N]$ 
 \begin{align}
 ([\hat{\vw}_{N}]_1 -[\hat{\vw}^{*}]_1 )^2 
 &= \prod_{t=1}^N \left(\frac{\eta\bar{\lambda}_1}{b_t}-1\right)^2 ([\hat{\vw}_{0}]_1 -[\hat{\vw}^{*}]_1 )^2 \nonumber \\
  &\geq  \left( \frac{\eta\bar{\lambda}_1}{\sqrt{b_0^2+ ([\hat{\vw}_{0}]_1 -[\hat{\vw}^{*}]_1 )^2} }-1 \right)^2 ([\hat{\vw}_{0}]_1 -[\hat{\vw}^{*}]_1 )^2 \label{eq:lastss}
\end{align}
Putting the bounds \eqref{eq:lower-bound1}, \eqref{eq:lower-bound2} and \eqref{eq:last} back to \eqref{eq:fill-need}, and adding the sharp estimate $N$ from Lemma \ref{lem:increase_sub}, we get the statement. Note that we use $s$ in this Lemma is $N$. We use $s$ as to point out it is a small constant.
\end{proof}

%% file: linear_proof_sto.tex
According to \cite{xie2019linear}, we make the following two assumptions for AdaLoss and AdaGrad-Norm, respectively. Note that the parameters $(\mu, \gamma)$ are available in linear regression due to  \cite{xie2019linear}.

\def \mug {\mu_{\scaleto{LR}{4pt}} }
\def \mulr {\mu_{\scaleto{G}{4pt}} }

\begin{asmp} \label{ruil}
\textbf{$(\varepsilon, \mu_{\scaleto{LR}{4pt}}, \gamma)-$Restricted Uniform Inequality of Loss for Linear Regression (RUIL-LR):} $ \forall \varepsilon >0$, $\exists~ (\mug, \gamma)$ s.t. for any fixed $ \vw \in \mathcal{D}_{\varepsilon} \triangleq \{ \vw \in \mathbb{R}^d: \| \vw-\vw^*\|^2 > \varepsilon\}$, $\mathbb{P}_{\xi_k} (\langle \vect{x}_{\xi_k}, \vw - \vw^* \rangle ^2 \geq \mu_{\scaleto{LR}{4pt}} \|\vw-\vw^*\|^2) \geq \gamma$, $\forall \xi_k = 1,2,\dots$.
\end{asmp}

\begin{asmp}
\label{ruig}
\textbf{$(\varepsilon, \mu_{\scaleto{G}{4pt}}, \gamma)-$Restricted Uniform Inequality of Gradients (RUIG):} $\forall \varepsilon>0$, for any fixed $\vw \in \mathcal{D}_{\varepsilon} \triangleq \{ \vw \in \mathbb{R}^d: \|\vw-\vw^*\|^2 > \varepsilon\}$, $\exists (\mu_{\scaleto{G}{4pt}}, \gamma)$ s.t. $\mu_{\scaleto{G}{4pt}} >0$, $\gamma>0$, and $\mathbb{P}_i (\|\nabla f_i(\vw)\|^2 \geq \mu_{\scaleto{G}{4pt}} \|\vw-\vw^*\|^2) \geq \gamma$, $\forall i = 1,2,\dots$.
\end{asmp}

\begin{thm}[\textbf{Stochastic AdaGrad-Norm: Restatement of Theorem 1 in \cite{xie2019linear} for Linear Regression Case}]\label{thm:sadagrad}
Consider the linear regression problem using AdaGrad-Norm Algorithm in the stochastic setting with Assumption \ref{ruig}. Suppose that the loss function $f(\vw) = \frac{1}{2n} \|\vX \vw - y\|^2 $ is $\bar{\lambda_n}-$strongly convex, $f_i = \frac{1}{2}(\vect{x}_i^T\vw - y_i)^2$ is $L_i$-smooth, $\bar{\lambda_1} \triangleq \sup_i L_i$, and $\vX \vw^* = \vy$, then 
\begin{itemize}
\item If $b_0 > \eta \bar{\lambda}_1$, we have $\|\vw_T - \vw^*\|^2 \leq \varepsilon$ with high probability $1-\delta_h$, after 
$$ T =\lceil \frac{ b_0 + (\bar{\lambda_1}/\eta)\|\vw_0-\vw^*\|^2}{\bar{\lambda_n}}\log \frac{\|\vw_0-\vw^*\|^2}{\varepsilon\delta_h} \rceil + 1$$

\item If $b_0 \leq \eta \bar{\lambda}_1$, we have $\min_{i}\|\vw_i-\vw^*\|^2 \leq \varepsilon$ with high probability $ 1-\delta_h - \delta_1$, after 
$$T =\lceil \frac{\eta^2\bar{\lambda_1}^2-b_0^2}{\mu_{\scaleto{G}{4pt}} \gamma \varepsilon}+\frac{\delta}{\gamma} + \frac{ \bar{\lambda}_1(\eta + \|\Delta\|^2/\eta) }{\bar{\lambda_n}} \log \frac{\|\Delta\|^2}{\varepsilon\delta_h}\rceil + 1$$ where $\|\Delta\|^2 = \|\Delta_0\|^2 + \eta^2 (2\log(\frac{\eta \bar{\lambda}_1 }{b_0}) + 1)$ and $\delta_1 = \exp(-\frac{\delta^2}{2(N\gamma(1-\gamma)+\delta)})$
\end{itemize}
\end{thm}

\begin{thm}\label{thm:ada-loss}(\textbf{Stochastic AdaLoss: Convergence for Linear Regression)} Consider the same setting as Theorem \ref{thm:linearRegression} with Adaloss algorithm \eqref{alg:linear_adas}. Let $\bar{\lambda_1} \triangleq \sup_i \|\vx_i\|^2$, then under Assumption \ref{ruil}, we have the following convergence rates\\
\textbf{(a)} If $b_0 > \eta \bar{\lambda}_1$, we have $\|\vw_T - \vw^*\|^2 \leq \varepsilon$ with high probability $1-\delta_h$, after 
    $$ T =\lceil \frac{ b_0 + (1/\eta)\|\vw_0-\vw^*\|^2}{\bar{\lambda_n}}\log \frac{\|\vw_0-\vw^*\|^2}{\varepsilon\delta_h} \rceil + 1$$
\textbf{(b)}  If $b_0 \leq \eta \bar{\lambda}_1$, we have $\min_{i}\|\vw_i-\vw^*\|^2 \leq \varepsilon$ with high probability $ 1-\delta_h - \delta_1$, after 
    $$T =\lceil \frac{\eta^2\bar{\lambda_1}^2-b_0^2}{\mu_{\scaleto{LR}{4pt}} \gamma\varepsilon}+\frac{\delta}{\gamma} + \frac{ (\eta \bar{\lambda}_1 + \|\Delta\|^2/\eta) }{\bar{\lambda_n}} \log \frac{\|\Delta\|^2}{\varepsilon\delta_h}\rceil + 1$$
    where $\|\Delta\|^2 = \|\Delta_0\|^2 + \eta^2 \bar{\lambda}_1 (2\log(\frac{\eta \bar{\lambda}_1 }{b_0}) + 1)$ and $\delta_1 = \exp(-\frac{\delta^2}{2(N\gamma(1-\gamma)+\delta)})$.
\end{thm}

\paragraph{Proof of Theorem \ref{thm:ada-loss} and \ref{thm:sadagrad}}

We follow the proof in \cite{xie2019linear} to get the convergence rates of AdaGrad-Norm in linear regression and adapt it to AdaLoss with Assumption \ref{ruil}. The proof of the lemmas are in Section \ref{pf:lemmas-linear}.

First, we derive the number of iterations that is sufficient to assure AdaLoss (or AdaGrad-Norm) to achieve either $b_N > \eta \bar{\lambda}_1$ or $\min_j \|\vw_j - \vw^*\|^2 \leq \varepsilon$ in Lemma \ref{highp}.

\begin{lemma}\label{highp}
$\forall \varepsilon$, with Assumption either $(\varepsilon, \mu, \gamma)-$RUIG or $(\varepsilon, \mu, \gamma)-$RUIL, after $N = \lceil\frac{\eta \bar{\lambda}_1^2-b_0^2}{\mu\gamma\varepsilon}+\frac{\delta}{\gamma}\rceil +1$ steps, with probability $1- \exp(-\frac{\|\Delta\|^2}{2(N\gamma(1-\gamma)+\delta)})$, either $b_N > \eta \bar{\lambda}_1$ or $\min_j \|\vw_j - \vw^*\|^2 \leq \varepsilon$.
\end{lemma}

From Lemma \ref{highp}, after $N \geq \frac{\eta^2\eta \bar{\lambda}_1^2-b_0^2}{\mu\gamma\varepsilon}+\frac{\delta}{\gamma}$ steps, 
if $\min_{0\leq i\leq N-1}\|\vw_i-\vw^*\|^2 > \varepsilon$, then with high probability $ 1- \exp(-\frac{\|\Delta\|^2}{2(N\gamma(1-\gamma)+\delta)})$, $b_N > \eta \bar{\lambda}_1 $. Then, there exists a first index $k_0 < N$, s.t. $b_{k_0} >\eta \bar{\lambda}_1$ but $b_{k_0-1} < \eta \bar{\lambda}_1$.

If $k_0 \geq 1$, then
\begin{equation}
\begin{split}
     \|\vw_{k_0+l}-\vw^*\|^2  & = \|\vw_{k_0-1+l}-\vw^*\|^2 + \frac{\eta^2}{b_{k_0+l}^2}\|G_{k_0-1+l}\|^2 - \frac{2\eta}{b_{k_0+l}}\langle \vw_{k_0-1+l}-\vw^*, G_{k_0-1+l}\rangle \\
     & \leq \|\vw_{k_0-1+l}-\vw^*\|^2 + (\frac{\eta^2 \bar{\lambda}_1}{b_{k_0+l}^2}- \frac{2\eta}{b_{k_0+l}}) \langle \vw_{k_0-1+l}-\vw^*, G_{k_0-1+l}\rangle\\
     & \leq \|\vw_{k_0-1+l}-\vw^*\|^2 - \frac{\eta}{b_{k_0+l}}\langle \vw_{k_0-1+l}-\vw^*, G_{k_0-1+l}\rangle\\
     & \leq \|\vw_{k_0-1+l}-\vw^*\|^2 - \frac{\eta}{b_{\max}}\langle \vw_{k_0-1+l}-\vw^*, G_{k_0-1+l}\rangle
\end{split}
\end{equation}
where $G_j:= \nabla f_j(\vw_j), \forall j$, and  the last second inequality is from the condition $b_{k_0} > \eta \bar{\lambda}_1$.

Take expectation regarding $\xi_{k_0-1+l}$, and use the fact that when $j > k_0$, $b_j > \eta \bar{\lambda}_1$, when $l \geq 1$ and $0< \frac{\eta \bar{\lambda_n}}{b_{k_0-1+l}} <  \frac{ \bar{\lambda_n}}{\bar{\lambda_1}} < 1$, then we can get
\begin{equation}
\begin{split}
     \mathbb{E}_{\xi_{k_0-1+l}}\|\vw_{k_0+l}-\vw^*\|^2 
     & \leq  \|\vw_{k_0-1+l}-\vw^*\|^2 - \frac{\eta}{b_{\max}} \langle \vw_{k_0-1+l}-\vw^*, \nabla F_{k_0-1+l}\rangle \\
     & \leq (1-\frac{\eta \bar{\lambda_n}}{ b_{\max} }) \|\vw_{k_0-1+l}-\vw^*\|^2 \\
     & \leq \prod_{j=0}^{l} (1-\frac{\eta \bar{\lambda_n}}{ b_{\max} }) \|\vw_{k_0-1}-\vw^*\|^2\\
     & \leq \exp( -\frac{\eta \bar{\lambda_n}l}{b_{\max}})\|\Delta\|^2
\end{split}
\end{equation}
where the second inequality is from the strong convexity of $F(\vw)$.

From the following lemmas, we can give an upper bound for $b_{\max}$ and $\|\Delta\|^2 = \|\vw_{k_0-1}-\vw^*\|^2$.

\begin{lemma}\label{boundset}
Suppose that $J$ is the first index s.t. $ b_J > \eta \bar{\lambda}_1$, then $\|\vw_{J-1} - \vw^*\|^2 \leq \|\vw_0-\vw^*\|^2 + \eta^2 (\log(\frac{\eta \bar{\lambda}_1^2}{b_0^2}) + 1)$ with AdaGrad-Norm;  $\|\vw_{J-1} - \vw^*\|^2 \leq \|\vw_0-\vw^*\|^2 + \eta^2 \bar{\lambda}_1 (\log(\frac{\eta \bar{\lambda}_1^2}{b_0^2}) + 1)$ with AdaLoss.
\end{lemma}

\begin{lemma}\label{lem:bmax}
 If $J$ is the first index s.t. $ b_J > \eta \bar{\lambda}_1$, then $b_{\max} \triangleq \max_{l\geq 0}b_{J+l} \leq  \eta \bar{\lambda}_1 + \frac{ \bar{\lambda}_1}{\eta} \|\vw_{J-1} -\vw^*\|^2$, if using AdaGrad-Norm; $b_{\max} \leq  \eta \bar{\lambda}_1 + \frac{1}{\eta} \|\vw_{J-1} -\vw^*\|^2$ if using AdaLoss, where $\|\vw_{J-1} -\vw^*\|^2$ is bounded by Lemma \ref{boundset}. 
\end{lemma}

Then take the iterated expectation, and use Markov in inequality, with high probability $1-\delta_h$, 
\begin{equation}
     \|\vw_{k_0+l}-\vw^*\|^2 \leq \frac{1}{\delta_h} \exp( -\frac{\eta \bar{\lambda_n}l}{b_{\max}})\|\Delta\|^2
\end{equation}
Hence, after $M \geq \frac{b_{\max}}{\eta \bar{\lambda_n} }\log \frac{\|\Delta\|^2}{\varepsilon \delta_h}$, with high probability more than $1-\delta_h - \exp(-\frac{\|\Delta\|^2}{2(N\gamma(1-\gamma)+\delta)}) $
\begin{equation}
  \|\vw_{k_0+M} - \vw^*\|^2 \leq \varepsilon
\end{equation}
Otherwise, if $k_0 = 0$, i.e. $b_0 > \eta \bar{\lambda}_1$, then use the same inequality as above,
\begin{equation}
\begin{split}
    \mathbb{E}_{\xi_{M-1}}\|\vw_M - \vw^* \|^2 & \leq (1-\frac{\eta \bar{\lambda_n}}{b'_{\max}})\|\vw_{M-1} - \vw^*\|^2 \leq \|\vw_0 - \vw^*\|^2 \exp( -\frac{\eta \bar{\lambda_n} M }{b'_{\max}})
\end{split}
\end{equation}
Then after $M \geq \frac{b'_{\max}}{\eta \bar{\lambda_n}}\log \frac{\|\vw_0-\vw^*\|^2}{\varepsilon\delta_h}$, by Markov's inequality, 
$$\mathbb{P} (\|\vw_M -\vw^*\|^2 \geq \varepsilon) \leq \frac{\mathbb{E}\|\vw_M-\vw^*\|^2}{\varepsilon}\leq \delta_h$$
where we estimate $b'_{\max}$ by plugging in $\|\Delta\|^2 = \|\vw_0 -\vw^*\|^2$.

\subsection{Proof of the Lemmas} \label{pf:lemmas-linear}



\begin{proof}[Proof of Lemma \ref{highp}]
If $\min_j \|\vw_j - \vw^*\|^2 \leq \varepsilon$, then we are done.\\
Otherwise, we have $\|\vw_j-\vw^*\|^2 > \varepsilon, \forall j = 0,1,2,...,N$. Assume that $F(\vw)$ satisfies $(\varepsilon, \mu, \gamma)-$ RUIG or $(\varepsilon, \mu, \gamma)-$ RUIL, use independent Bernoulli random variables $\{Z_j\}$:
\begin{equation*}
    Z_j = \left\{
    \begin{array}{cc}
    1 & if ~ \|\nabla f_{\xi_j}(\vw_j)\|^2 \geq \mu \|\vw_j - \vw^*\|^2 \quad \text{(RUIL:}\langle \vect{x}_i, \vw - \vw^* \rangle ^2 \geq \mu \|\vw_j - \vw^*\|^2)  \\
    0 & \text{else} 
    \end{array}
    \right.
\end{equation*}
where $\mathbb{P}(Z_j = 1) \geq \gamma, \forall j$. Then let $Z=\sum_j Z_j$, from Bernstein's inequality, with high probability bigger than $1- \exp(-\frac{\delta^2}{2(N\gamma(1-\gamma)+\delta)})$, $Z \geq \gamma N -\delta, \forall N$. Thus, after $N \geq \frac{C^2-b_0^2}{\mu\gamma\varepsilon}+\frac{\delta}{\gamma}$ steps, with $1- \exp(-\frac{\delta^2}{2(N\gamma(1-\gamma)+\delta)})$, we have
$$b_N^2 = b_0^2 + \sum_{i=0}^{N-1} \| \nabla f_{\xi_i}(x_i)\|^2 > b_0^2 + (\gamma N -\delta)\mu \varepsilon \geq \eta^2 \bar{\lambda}_1^2 $$
\end{proof}

\begin{proof}[Proof of Lemma \ref{boundset}]
If we implement AdaGrad-Norm, then
\begin{equation}
    \begin{split}
        \|\vw_{J-1} - \vw^*\|^2 & =  \|\vw_{J-2}-\vw^*\|^2 + \|\frac{\eta G_{J-2}}{b_{J-1}} \|^2 - 2 \langle \frac{\eta G_{J-2}}{b_{J-1}} ,\vw_{J-2}-\vw^*\rangle  \\
        & \leq  \|\vw_{J-2}-\vw^*\|^2 + \|\frac{\eta G_{J-2}}{b_{J-1}} \|^2 - \frac{2\eta}{b_{J-1}\eta \bar{\lambda}_1}(f_{J-2}(\vw_{J-2}) - f_{J-2}(\vw^*)) \\
        & \leq  \|\vw_{J-2}-\vw^*\|^2 + \frac{\eta^2\|G_{J-2}\|^2}{b_{J-1}^2}  \leq  \|\vw_0-\vw^*\|^2 + \eta^2 \sum_{j=0}^{J-2} \frac{\|G_j\|^2}{b_{j+1}^2} \\
        & \leq  \|\vw_0-\vw^*\|^2 + \eta^2 \sum_{j=0}^{J-2} \frac{\|G_j\|^2 / b_0^2}{1+\sum_{l=0}^{j}\|G_l\|^2 / b_0^2}\\
        & \leq  \|\vw_0-\vw^*\|^2 + \eta^2 (\log(\sum_{j=0}^{J-2}  \|G_j\|^2/b_0^2) + 1)\\
        & \leq  \|\vw_0-\vw^*\|^2 + \eta^2 (\log(\frac{\eta \bar{\lambda}_1^2}{b_0^2}) + 1) \\
    \end{split}
\end{equation}
where the first inequality is from the convex assumption, last second inequality is from integral lemma and the last one is from the assumption that $J$ is the first index s.t. $ b_J^2 >  \eta \bar{\lambda}_1^2$.

Moreover, if we use AdaLoss instead, $\|G_j\|^2 = \|\vect{x}_{\xi_j}\|^2 \langle \vect{x}_{\xi_j}, \vw_j - \vw^*\rangle^2 \leq L \langle \vect{x}_{\xi_j}, \vw_j - \vw^*\rangle^2 $. $b_j^2 = b_0^2 +\langle \vect{x}_{\xi_j}, \vw_j - \vw^*\rangle^2$, then 
\begin{equation}
    \begin{split}
    \|\vw_{J-1} - \vw^*\|^2 & \leq \|\vw_0-\vw^*\|^2 + \eta^2 \bar{\lambda}_1 \sum_{j=0}^{J-2} \frac{\langle \vect{x}_{\xi_j}, \vw_j - \vw^*\rangle^2 / b_0^2}{1+\sum_{l=0}^{j}\langle \vect{x}_{\xi_j}, \vw_j - \vw^*\rangle^2/ b_0^2} \\
    &\leq  \|\vw_0-\vw^*\|^2 + \eta^2 \bar{\lambda}_1 (\log(\frac{\eta \bar{\lambda}_1^2}{b_0^2}) + 1) 
\end{split}
\end{equation}
\end{proof}

\begin{proof} [Proof of Lemma \ref{lem:bmax}]
Since $b_J>\eta \bar{\lambda}_1$, if we implement AdaGrad-Norm, we have the following bound for $\|\vw_{J+l} - \vw^*\|^2$:
\begin{equation}
\label{eq:adagrad-norm-ineq}
\begin{split}
    \|\vw_{J+l} -\vw^*\|^2 &  = \|\vw_{J+l-1}-\vw^*\|^2 + \frac{\eta^2 \|G_{J+l-1}\|^2}{b_{J+l}^2} \\ 
    & - \frac{2\eta}{b_{J+l}}\langle G_{J+l-1}-\nabla f_{J+l-1}(\vw^*), \vw_{J+l-1} - \vw^*\rangle   \\
    & \leq  \|\vw_{J+l-1}-\vw^*\|^2 + \|\frac{\eta G_{J+l-1}}{b_{J+l}} \|^2 - \frac{2\eta}{b_{J+l} \bar{\lambda}_1}\|G_{J+l-1}-\nabla f_{J+l-1}(\vw^*)\|^2\\
    & \leq \|\vw_{J+l-1}-\vw^*\|^2 + \frac{\eta \|G_{J+l-1}\|^2}{b_{J+l}}(\frac{\eta}{b_{J+l}} -\frac{2}{ \bar{\lambda}_1}) \\
    & \leq \|\vw_{J+l-1}-\vw^*\|^2 - \frac{\eta}{ \bar{\lambda}_1}\frac{\|G_{J+l-1}\|^2}{b_{J+l}}\\
    & \leq \|\vw_{J-1}-\vw^*\|^2 - \sum_{j=0}^{l}\frac{\eta}{ \bar{\lambda}_1}\frac{\|G_{J+j-1}\|^2}{b_{J+j}}
\end{split}
\end{equation}
where the first two lines are from $f_i(x)$ is $\eta \bar{\lambda}_1$-smooth. Then, we have the bound of the sum:
\begin{equation}
    \sum_{j=0}^{l}\frac{\|G_{J+j-1}\|^2}{b_{J+j}} \leq  \frac{ \bar{\lambda}_1}{\eta}(\|\vw_{J-1} -\vw^*\|^2 -  \|\vw_{J+l}-\vw^*\|^2)
\end{equation}
Therefore, $b_{\max}$ is bounded as follows:
\begin{equation}
    \begin{split}
        b_{J+l} & = b_{J+l-1} + \frac{\|G_{J+l-1}\|^2}{b_{J+l}+b_{J+l-1}}\\
        & \leq b_{J-1} +\sum_{j=1}^{l} \frac{\|G_{J+j-1}\|^2}{b_{J+j}} \\
        & \leq \eta \bar{\lambda}_1 + \frac{\bar{\lambda_1}}{\eta}\|\vw_{J-1} - \vw^*\|^2 
    \end{split}
\end{equation}
Moreover, if we use AdaLoss instead, 
\begin{equation}
\label{eq: adaloss}
\begin{split}
    \|\vw_{J+l} -\vw^*\|^2 &  = \|\vw_{J+l-1}-\vw^*\|^2 + \frac{\eta^2 \|\vect{x}_{\xi_j}\|^2 \langle \vect{x}_{\xi_j}, \vw_{J+l-1} - \vw^*\rangle^2 }{b_{J+l}^2} \\ 
    & - \frac{2\eta}{b_{J+l}}\langle \langle \vect{x}_{\xi_j}, \vw_{J+l-1} - \vw^*\rangle \vect{x}_{\xi_j} , \vw_{J+l-1} - \vw^*\rangle \\
    & \leq \|\vw_{J+l-1}-\vw^*\|^2 + \frac{\eta}{b_{J+1}}(\frac{\eta \bar{\lambda}_1}{b_{J+1}} - 2 ) \langle \vect{x}_{\xi_j}, \vw_{J+l-1} - \vw^*\rangle^2\\
    & \leq  \|\vw_{J-1}-\vw^*\|^2 - \sum_{j=0}^{l}\eta \frac{\langle \vect{x}_{\xi_j}, \vw_j - \vw^*\rangle^2}{b_{J+j}}
\end{split}
\end{equation}
Then, use the same technique above,
$$b_{J+l} \leq \eta \bar{\lambda}_1 + \frac{1}{\eta} \|\vw_{J-1} - \vw^*\|^2  $$
\end{proof}

%% file: experiments.tex
We first plot the eigenvalues of the matrices $\{\mat{H} (k)\}_{k'=0}^{k}$  and then provide the details. 

We use two simulated Gaussian data sets:  i.i.d. Gaussian (the red curves) and multivariate Gaussian (the blue curves). Observe the red curves in Figure \ref{fig:data} that $\|\vH(t)\|, t = 0,1,\ldots$ are around $2.8$ and minimum eigenvalues are around $0.19$ within 100 iterations, while the  maximum  and minimum eigenvalues for the blue curves  are  around $291$ and $0.033$ respectively.  To some extent,  i.i.d. Gaussian data  illustrates the case where the data points are  pairwise  uncorrelated  such that $ \norm{\mat{H}^{\infty}} =O\left(1\right)$, while correlated Gaussian data set implies the situation when the samples are highly correlated with each other $ \norm{\mat{H}^{\infty}} =O\left(n\right)$.
\begin{figure}[ht]
\centering
 \includegraphics[width=\linewidth]{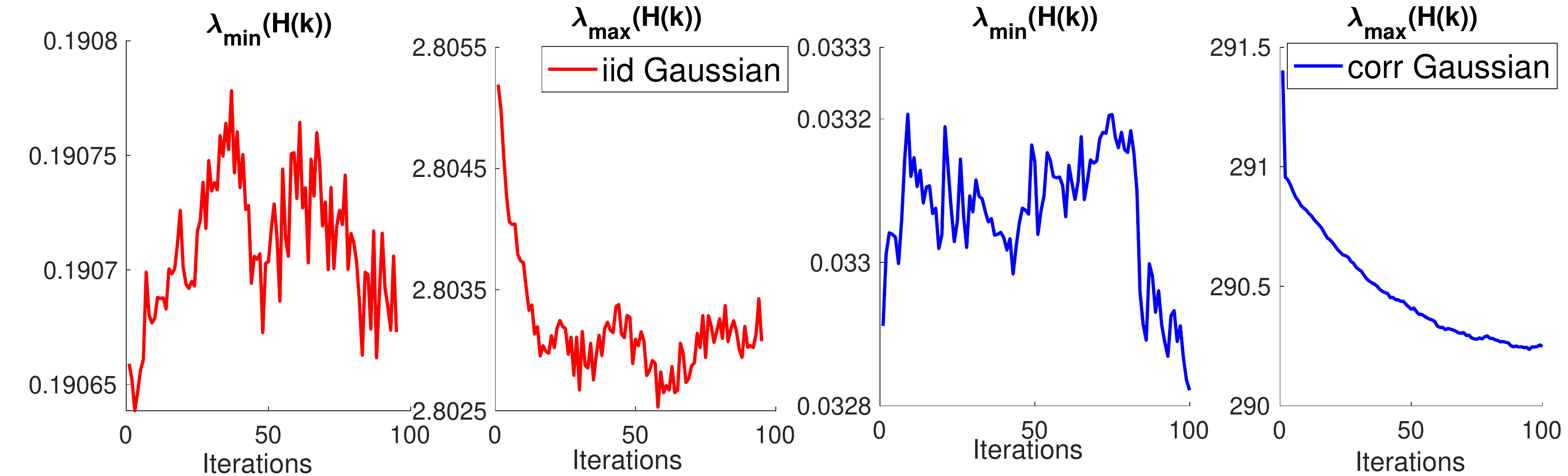}
 \includegraphics[width=\linewidth]{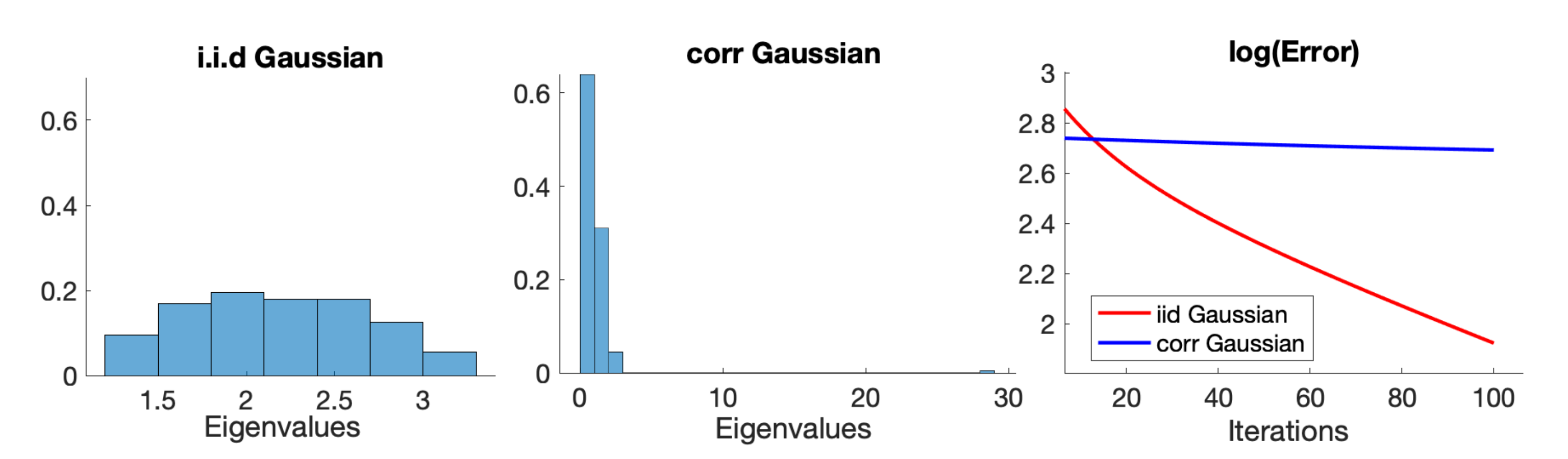}
\caption{ Top plots: y-axis is maximum or minimum eigenvalue of the matrix $\vH(k)$, the x-axis is the  iteration. 
Bottom plots (left and middle): y-axis is the probability, x-axis is the  eigenvalue of co-variance matrix induced by Gaussian data.   Bottom plots (right): y-axis is the training error in logarithm scale, x-axis is the iteration. The distributions of eigenvalues for the co-variances matrix ($d \times d$ dimension) of the data  are  plotted on the left  for i.i.d. Gaussian and  in the middle  for  correlated Gaussian. The  bottom right plot  is the training error for the two-layer neural network $m=5000$ using the two Gaussian datasets.}\label{fig:data}
\end{figure}
In the experiments, we simulate Gaussian data with training sample  $n=1000$  and  the dimension $d=200$. Figure \ref{fig:data} plots the histogram of  the eigenvalues of the co-variances for each dataset. Note that the eigenvalues are different from the eigenvalues in the top plots. We use  the two-layer neural networks $m=5000$. Although  $m$ here is far smaller than what Theorem \ref{thm:main_gd} requires, we found it sufficient for our purpose to just illustrate the maximum and minimum eigenvalues of  $\vH(k)$ for iteration $k = 0,1, \ldots, 100$. Set the learning rate  $\eta = 5\times 10^{-4}$ for i.i.d. Gaussian and  $\eta = 5\times 10^{-5}$ for correlated Gaussian. The training error is also  given in Figure \ref{fig:data}. 


%% file: discrete.tex
We prove Theorem~\ref{thm:main_gd} by induction\footnote{Note that we use the same structure as in \cite{du2018gradient}. For the sake of completeness in the proof, we will use most of their lemmas, of which the proofs can be found  in technical section or otherwise in their paper .}.
Our induction hypothesis is the following convergence rate of empirical loss.
\begin{condition}\label{cond:linear_converge}
 At the $k$-th iteration,   we have  for $m =
\Omega\left(\frac{n^6}{\lambda_0^4\delta^3}\right)
$  such that with probability   $1-\delta$, $$\norm{\vect{y}-\vect{u}(k)}^2  \le (1-\frac{\eta \lambda_0}{2})^{k} \norm{\vect{y}-\vect{u}(0)}^2 .$$   
\end{condition}

\emph{Now we show Condition~\ref{cond:linear_converge} for every $k = 0,1,\ldots$.
For the base case $k=0$, by definition Condition~\ref{cond:linear_converge} holds.
Suppose for $k'=0,\ldots,k$, Condition~\ref{cond:linear_converge} holds and we want to show Condition~\ref{cond:linear_converge} holds for $k'=k+1$. We first prove the  order of $m$ and then the contraction of $\norm{\vect{u}(k+1)-\vect{y}}$.}
\subsection{ The order of $m$ at iteration $k+1$} \label{sec:order_m_gd}
Note that the contraction for $\norm{\vect{u}(k+1)-\vect{y}}$ is mainly controlled by the smallest eigenvalue of  the sequence of matrices $ \{\mat{H} (k')\}_{k'=0}^k$ . It requires that the minimum eigenvalues of matrix 
$\vH(k’), k'=0,1,\cdots, k$ are strictly positive, which is equivalent to ask that the update of $\vw_r(k')$ is not far away from initialization  $\vw_r(0)$ for $r\in[m]$. This requirement can be fulfilled by the large hidden nodes $m$. 

The first lemma (Lemma \ref{lem:H0_close_Hinft}) gives smallest $m$  in order to have $\lambda_{\min}(\mat{H}(0))>0$.
The next two lemmas concludes the order of $m$ so that  $\lambda_{\min}(\mat{H}(k'))>0$ for $k'=0,1,\cdots,k$. Specifically,  if $R' < R$, then the conditions in Lemma~\ref{lem:close_to_init_small_perturbation} hold for all $0\le k'\le k$.
We refer the proofs  of these lemmas to \cite{du2018gradient} 
\begin{lem}\label{lem:H0_close_Hinft}
If $m = \Omega\left(\frac{n^2}{\lambda_0^2} \log^2\left(\frac{n}{\delta}\right)\right)$, we have with probability at least $1-\delta$ that $\norm{\mat{H}(0)-\mat{H}^\infty} \le \frac{\lambda_0}{4}$.
\end{lem}

\begin{lem}\label{cor:dist_from_init}
If Condition~\ref{cond:linear_converge} holds for $k'=0,\ldots,k$, then we have for every $r \in [m]$
\begin{align}
\norm{\vect{w}_r(k'+1)-\vect{w}_r(0)} \le \frac{4\sqrt{n} \norm{\vect{y}-\vect{u}(0)}}{ \sqrt{m} \lambda_0} \triangleq R'.
\end{align}
\end{lem}
\begin{lem}\label{lem:close_to_init_small_perturbation}
Suppose for $r\in [m]$, $\norm{\vw_r - \vw_r(0)} \le \frac{c\lambda_0 \delta}{n^3 } \triangleq R$  for some small positive constant $c$.
Then we have with probability $1-\delta$ over initialization, $\norm{\mat{H} - \mat{H}(0)} \le \frac{\lambda_0}{4}$ where $\vH$ is defined in Definition \ref{def:H}.
\end{lem}
Thus it is sufficient to show $R' < R$. Since $\norm{\vect{y}-\vect{u}(0)}^2=O\left(\frac{n}{\delta} \right)$ derived from Proposition \ref{prop:a2},   $R' < R$ implies that \begin{align*}
m = & \Omega\left(\frac{n^7\norm{\vect{y}-\vect{u}(0)}^2 }{\lambda_0^4\delta^2}\right)= \Omega\left(\frac{n^8 }{\lambda_0^4 \delta^3}\right) 
\end{align*}

\subsection{The contraction  of  $\norm{\vect{u}(k+1)-\vect{y}}$}\label{sec:contraction_gd}
Define the event \[A_{ir} = \left\{\exists \vect{w}: \norm{\vect{w}-\vect{w}_r(0)} \le R, \indict\left\{\vect{x}_i^\top \vect{w}_r(0) \ge 0\right\} \neq \indict\left\{\vect{x}_i^\top \vect{w} \ge 0\right\} \right\} \quad \text{ with } 
R= \frac{c \lambda_0 \delta}{n^3}\]
for some small positive constant $c$.
We let
$S_i = \left\{r \in [m]: \indict\{A_{ir}\}= 0\right\}$ and $
S_i^\perp = [m] \setminus S_i$.
 The following lemma bounds the sum of sizes of $S_i^\perp$.
\begin{lem}\label{lem:bounds_Si_perp}
With probability at least $1-\delta$ over initialization, we have $\sum_{i=1}^{n}\abs{S_i^\perp} \le  \frac{C_2mnR}{\delta}$  for some positive constant $C_2$.
\end{lem}
Next, we calculate the difference of predictions between two consecutive iterations,
\begin{align*}
[\vu(k+1)]_i - [\vu(k)]_i
= & \frac{1}{\sqrt{m}}\sum_{r=1}^{m} a_r \left(\relu{\left(
	\vect{w}_r(k) - \eta\frac{\partial L(\mat{W}(k))}{\partial \vect{w}_r(k)}
	\right)^\top \vect{x}_i} - \relu{\vect{w}_r(k)^\top \vect{x}_i}\right).
\end{align*}
Here we divide the right hand side into two parts.
$\vE_1^i$ accounts for terms that the pattern does not change and $\vE_2^i$ accounts for terms that pattern may change.

 Because $R' < R$, we know $\indict\left\{\vect{w}_r(k+1)^\top \vect{x}_i\ge 0\right\} \cap S_{i}= \indict\left\{\vect{w}_r(k)^\top \vect{x}_i \ge 0\right\}\cap S_{i}$.
\begin{align*}
[\vE_1]_i &\triangleq  \frac{1}{\sqrt{m}}\sum_{r \in S_i} a_r \left(\relu{\left(
	\vect{w}_r(k) - \eta\frac{\partial L(\mat{W}(k))}{\partial \vect{w}_r(k)}
	\right)^\top \vect{x}_i} - \relu{\vect{w}_r(k)^\top \vect{x}_i}\right) \\
	&= -\frac{1}{\sqrt{m}} \sum_{r \in S_i} a_r  \eta\langle{ \frac{\partial L(\mat{W}(k))}{\partial \vect{w}_r(k)}, \vx_i \rangle}\\
&=-\frac{\eta}{m}\sum_{j=1}^{n}\vect{x}_i^\top \vect{x}_j \left([\vu(k)]_j-\vy_j\right)\sum_{r \in S_i} \indict\left\{\vect{w}_r(k)^\top \vect{x}_i \ge 0, \vect{w}_r(k)^\top \vect{x}_j\ge 0 \right\}\\
&=  -\eta\sum_{j=1}^{n}([\vu(k)]_j-\vy_j)(\mat{H}_{ij}(k) -\mat{H}_{ij}^\perp(k))
\end{align*}
where $
\mat{H}_{ij}^\perp(k) = \frac{1}{m}\sum_{r \in S_i^\perp}\vect{x}_i^\top \vect{x}_j \indict\left\{\vect{w}_r(k)^\top \vect{x}_i \ge 0, \vect{w}_r(k)^\top \vect{x}_j \ge 0\right\}
$ is a perturbation matrix.
Let $\mat{H}^\perp(k)$ be the $n \times n$ matrix with $(i,j)$-th entry being $\vH_{ij}^\perp(k)$.
Using Lemma~\ref{lem:bounds_Si_perp}, we obtain  with  probability at least $1-\delta$, 
\begin{align} 
\norm{\mat{H}^\perp(k)} \le \sum_{(i,j)=(1,1)}^{(n,n)} \abs{\mat{H}_{ij}^\perp (k)}
\le \frac{n \sum_{i=1}^{n}\abs{S_i^\perp}}{m}
\le   \frac{n^2 mR}{\delta m} 
\le  \frac{n^2R}{\delta}.
\label{eq:norm_H_perp}
\end{align}
We view $[\vE_2]_i$ as a perturbation and bound its magnitude.
Because ReLU is a $1$-Lipschitz function and $\abs{a_r}  =1$, we have 
\begin{align}
[\vE_2]_i& \triangleq \frac{1}{\sqrt{m}}\sum_{r \in S_i^\perp} a_r \left(\relu{\left(
	\vect{w}_r(k) - \eta\frac{\partial L(\mat{W} (k))}{\partial \vect{w}_r(k)}
	\right)^\top \vect{x}_i} - \relu{\vect{w}_r(k)^\top \vect{x}_i}\right)\nonumber \\
&\le \frac{\eta }{\sqrt{m}} \sum_{r\in S_i^\perp} \norm{
	\frac{\partial L(\mat{W}(k))}{\partial \vect{w}_r(k)}
}  \nonumber\\
\Rightarrow   \quad  &\norm{\vE_2} \le  \frac{\eta \abs{S_i^\perp} \sqrt{n}\norm{\vect{u}(k)-\vect{y}}}{ m} 
 \leq \frac{\eta n^{3/2} R}{\delta} \norm{\vect{u}(k)-\vect{y}}\label{eq: I_2_upper}
\end{align}
Observe the maximum eigenvalue of matrix $\vH(k)$ upto iteration $k$ is bounded because
\begin{align}
\|\vH(k)-\vH(0)\| \leq \frac{4n^2R}{\sqrt{2\pi}\delta} \quad \text{with probability } 1-\delta \quad \Rightarrow \quad  \|\vH(k)\| \leq \|\vH(0)\| + \frac{4n^2R}{\sqrt{2\pi}\delta} \label{eq:upperH}
\end{align} 
Further, Lemma \ref{lem:H0_close_Hinft} \footnote{For more details, please see Lemma 3.1 in \cite{du2018gradient}} implies that 
\begin{align} 
\bigg|\|\vH^{\infty}\|-\|\vH(0)\|\bigg| = O\left( \frac{n^2\log(n/\delta)}{m} \right).
 \label{eq:upperH1}
\end{align}
That is, we could almost ignore the distance between $\|\vH^{\infty}\|$ and $\|\vH(0)\|$ for $m=\Omega\left(\frac{n^8}{\lambda_0^4\delta^3}\right)$.

With these estimates at hand, we are ready to prove the induction hypothesis.
\begin{align}
\norm{\vect{y}-\vect{u}(k+1)}^2 = 
&\norm{\vect{y}- \left(\vect{u}(k) + \vE_1 +\vE_2 \right)}^2 \nonumber\\
= &\norm{\vect{y}-\vect{u}(k)}^2 -  2\left(\vect{y}-\vect{u}(k)\right)^\top \left(\vE_1 +\vE_2\right) +\norm{\vE_1}^2 +\norm{\vE_2}^2+2\langle{\vE_1,\vE_2 \rangle} \nonumber\\
= &\norm{\vect{y}-\vect{u}(k)}^2 -   2\eta\left(\vect{y}-\vect{u}(k)\right)^\top \left(\mat{H}(k)-\mat{H}(k)^\perp\right) \left(\vect{y}-\vect{u}(k)\right) - 2 \left(\vect{y}-\vect{u}(k)\right)^\top\vect{E}_2 \nonumber \\
&+\eta^2\left(\vect{y}-\vect{u}(k)\right)^\top \left(\mat{H}(k)-\mat{H}(k)^\perp\right) ^2 \left(\vect{y}-\vect{u}(k)\right)  + \|\vect{E}_2\|^2 +2\langle{\vE_1,\vE_2 \rangle}\nonumber\\
= &\underbrace{\norm{\vect{y}-\vect{u}(k)}^2  -   2\eta\left(\vect{y}-\vect{u}(k)\right)^\top\left(\mathbf{I}- \frac{\eta}{2}\mat{H}(k)\right) \mat{H}(k) \left(\vect{y}-\vect{u}(k)\right)}_{{Term 1}} \nonumber\\
&+  2\eta\underbrace{\left(\vect{y}-\vect{u}(k)\right)^\top\left(  \frac{\eta}{2}\mat{H}^\perp(k) - \mathbf{I} \right) \mat{H}^\perp(k) \left(\vect{y}-\vect{u}(k)\right)}_{{Term 2}} \nonumber\\
&+  \underbrace{ \langle{2\vect{E_1} +\vect{E_2}-2 \left(\vect{y}-\vect{u}(k)\right) , \vect{E}_2 \rangle} }_{{Term 3}}+ 2\eta^2 \underbrace{\left(\vect{y}-\vect{u}(k)\right)^\top\mat{H}(k)\mat{H}^\perp(k) \left(\vect{y}-\vect{u}(k)\right)}_{{Term 4}}  
\label{eq:3terms}
\end{align}
Note that  $Term1$ dominates the increase or decrease of the term $\norm{\vect{y}-\vect{u}(k+1)}^2$ and other terms are very small for significant large $m$. 

First, given the strict positiveness of matrix $\vH(k)$ and  the range of stepsize such that $\eta \leq \frac{1}{\| \mat{H}(k) \| }$,  we have
\begin{align*}
Term 1& \leq \left( 1-2\eta\left(1-\frac{ \eta}{2} ||\mat{H}(k)\| \right) \lambda_{min}(\mat{H}(k)) \right)\norm{\vect{y}-\vect{u}(k)}^2 
\end{align*}
Due to \eqref{eq:norm_H_perp}, we could bound  $Term2$ 
\begin{align*}
Term 2 \leq \left(\frac{ \eta}{2} \|\mat{H}^\perp(k)\| ^2+\|\mat{H}^\perp(k)\| \right) \norm{\vect{y}-\vect{u}(k)}^2 \leq  \left( \frac{\eta n^2R}{2\delta}  + 1  \right) \frac{n^2R }{\delta}  \norm{\vect{y}-\vect{u}(k)}^2
\end{align*}
Due to \eqref{eq: I_2_upper}, we could bound  $Term 3$
\begin{align*}
Term 3& \leq  \left(2\|\vect{E_1}\|+ \|\vect{E_2}\|  +2 \norm{\vect{y}-\vect{u}(k)} \right) \|\vect{E_2}\|\\
& \leq \left( 2\eta \left( \|\mat{H}(k) \| +\| \mat{H}^\perp(k)\| \right)  \|\vy-\vu(k)\|+ \|\vect{E_2}\| + 2 \|\vy-\vu(k)\| \right) \|\vect{E_2}\|\\
& \leq \left( 2 \eta \| \mat{H}(k)\|+\frac{2\eta n^2R }{\delta}+1+ \frac{2\eta  n^{3/2} R}{\delta}  \right) \frac{\eta n^{3/2} R}{\delta} \norm{\vect{u}(k)-\vect{y}}^2 
\end{align*}
Finally, for $Term 4$
\begin{align*}
Term 4 \leq \|\mat{H}(k) \| \| \mat{H}^\perp(k)\|\norm{\vect{u}(k)-\vect{y}}^2 \leq \frac{n^2R}{\delta}  \|\mat{H}(k) \| \norm{\vect{u}(k)-\vect{y}}^2
\end{align*}

Putting $Term1$,  $Term2 $,  $Term3$  and  $Term4$ back to inequality \eqref{eq:3terms}, we have with probability $1-\delta$
\begin{align}
&\norm{\vect{y}-\vect{u}(k+1)}^2 -\norm{\vect{y}-\vect{u}(k)}^2 \nonumber\\
\leq & - 2\lambda_0 \eta \left( 1-   \frac{n^{3/2}\left( \sqrt{n}+2 \right)R }{2\lambda_0 \delta}- \frac{\eta}{2} \left( \frac{ \left(\lambda_0  \delta +2 n^{3/2}R+ n^2R\right)\|\vH(k)\|}{ \lambda_0  \delta} +\delta_1\right)\right)\norm{\vect{y}-\vect{u}(k)}^2
\label{eq:keyforadagrad1}
\end{align}
where $\delta_1= \left(\frac{n^{5/2}R }{2\delta}+\frac{2 n^2R }{\delta}+ \frac{2 n^{3/2} R}{\delta}  \right)\frac{n^{3/2} R}{\delta} $. Recall that $R= \frac{c\lambda_0 \delta}{n^3}$ for very small constant $c$; 
let  $C_1 = 1-   \frac{n^{3/2}\left( \sqrt{n}+2 \right)R }{2\lambda_0 \delta} = $ and 
 $C  = \frac{ { \left(\lambda_0  \delta +2 n^{3/2}R+ n^2R\right)\left( \|\vH^{\infty}\| + \frac{4n^2R}{\sqrt{2\pi}\delta} \right)}/{ \lambda_0  \delta} +\delta_1 }{ C_1 \| \vH^{\infty}\|} >0$. Since the upper bound of $\vH(k)$ is $  \|\vH^{\infty}\| + \frac{4n^2R}{\sqrt{2\pi}\delta}$ due to  \eqref{eq:upperH} and \eqref{eq:upperH1}, then we can re-write  \eqref{eq:keyforadagrad1} as follows
  \begin{align}
&\norm{\vect{y}-\vect{u}(k+1)}^2 -\norm{\vect{y}-\vect{u}(k)}^2 \nonumber\\
\leq &- 2\eta \lambda_0C_1 \left(1-\frac{\eta}{2} C{\|\vH^{\infty}\|}\right) \norm{\vect{y}-\vect{u}(k)}^2
\label{eq:keyforadagrad}
\end{align}
 
 We  have contractions for $\norm{\vect{y}-\vect{u}(k+1)}^2 , $  if the stepsize satisfy 
\begin{align*}
\eta \leq  \frac{2}{C{\|\vH^\infty\|}}  \quad \Rightarrow  \quad \eta   \leq \frac{1}{\|\vH(k)\|} . 
\end{align*}
 We could pick  $\eta = \frac{1}{C{\|\vH^\infty\|}} =\Theta\left( \frac{1}{\|\vH^{\infty}\|}\right) $ for  large $m$ such that
\begin{align}
\norm{\vect{y}-\vect{u}(k+1)}^2 
& \leq  \left(1-\frac{\lambda_0C_1}{C{\|\vH^\infty\|}}\right)\norm{\vect{y}-\vect{u}(k)}^2 .  \nonumber
\end{align}
Therefore Condition~\ref{cond:linear_converge} holds for $k'=k+1$. Now by induction, we prove Theorem~\ref{thm:main_gd} .

%% file: AproofGD_adagrad.tex
\subsubsection{Proof for Theorem \ref{thm:main_adagradloss}} 
\label{D}
As long as we can make sure $\vH^\infty$ is strictly positive, then we can borrow the proof of linear regression (Section \ref{sec:linear-proof}) by using careful induction for each step. We will omit repetitive steps and give a short sketch of proof and put the key lemmas after the sketch of the proof. We would like to point out that the proof is not a direct borrow from section \ref{sec:linear-proof} because 
\begin{itemize}
    \item the level of over-parameterization (see Lemma \ref{lem: small_w_for_allbsq1_sub} and Lemma \ref{lem: small_w_for_allbsq_sub});
    \item the maximum steps needed for $b_t$ greater than the critical threshold for the decreasing of $\|\vy -\vu(t)\|$ when $b_0$ is initialized with small values (see Lemma \ref{lem:increase_sub1}  and Lemma \ref{lem:increase_sub_a} ).
\end{itemize}

\paragraph{Similarity with Linear Regression}
We will discuss the similarity of the linear regression before we present the proof. Note that the equality  \eqref{eq:ineq2} is similar to \eqref{eq:3terms}. That is $\bar{\lambda}_1=C\|\vH^\infty\|$ and $\bar{\lambda}_n = \lambda_0C_1$. However, in order to have the similar result as linear regression, we need to get two  inequalities analogue to  \eqref{eq:ineq3} and \eqref{eq:ineq1} respectively. Indeed, we can get one similar to \eqref{eq:ineq3} based on
\begin{align}
\norm{ & (\vH^\infty )^{1/2}\left(\vect{y}-\vect{u}(k+1)\right)}^2 
= \underbrace{\norm{ (\vH^\infty )^{1/2} (\vect{y}-\vect{u}(k))}^2  -   2\eta \left(\vect{y}-\vect{u}(k)\right)^\top\vH^\infty\left(\mathbf{I}- \frac{\eta}{2}\mat{H}(k)\right) \mat{H}(k) \left(\vect{y}-\vect{u}(k)\right)}_{{Term 1}} \nonumber\\
&+  2\eta \underbrace{\left(\vect{y}-\vect{u}(k)\right)^\top  (\vH^\infty )^{1/2}\left(  \frac{\eta}{2}\mat{H}^\perp(k) - \mathbf{I} \right) \mat{H}^\perp(k) (\vH^\infty )^{1/2}\left(\vect{y}-\vect{u}(k)\right)}_{{Term 2}} \nonumber\\
&+  \underbrace{ \langle{2\vect{E_1} +\vect{E_2}-2 \left(\vect{y}-\vect{u}(k)\right) , \vH^\infty  \vect{E}_2 \rangle} }_{{Term 3}}+ 2\eta^2 \underbrace{ \left(\vect{y}-\vect{u}(k)\right)^\top\mat{H}(k) \mat{H}^\perp(k) \vH^\infty\left(\vect{y}-\vect{u}(k)\right)}_{{Term 4}}
\end{align}
 We list the following conditions that analogue to  inequality \eqref{eq:ineq2}, \eqref{eq:ineq3} and \eqref{eq:ineq1}.
\begin{condition}
	\label{con:adagrad_sub} At the $k'$-th iteration,\footnote{ For the convenience of induction proof, we define 
		$$ \norm{\vect{y}-\vect{u}(-1)}^2= \norm{\vect{y}-\vect{u}(0)}^2 /\left(1- \frac{\eta  \lambda_0 C_1}{b_{0}}\left(1-\frac{\eta C{\|\vH(0)\|}}{b_{0}} \right) \right). $$
			} we have for some small $\eta =\mathcal{O}\left(\frac{1}{C\|\vH^\infty\|} \right)$
		\begin{align}
\|\vy-\vu(k')\|^2
&\leq \left( 1-\frac{2\eta \lambda_0C_1}{b_{k'}} \left(1- \frac{\eta C\|\vH^\infty\|}{2b_{k'}}\right) \right)\|\vu(k'-1)-\vy\|^2  \label{eq:decrease1} 
\end{align}
\begin{align}
\| (\vH^\infty )^{1/2}\left(\vy-\vu(k')\right)\|^2 
&\leq \| (\vH^\infty )^{1/2}\left(\vy-\vu(k')\right)\|^2 \nonumber\\
& \quad- \frac{2\eta\widetilde{C}_1}{b_{k'}} \left( 1 - \frac{\eta \widetilde{C}\|\vH^\infty\|}{2b_{k'}} \right)\| \vH^\infty \left(\vy-\vu(k'-1)\right)\|^2  \label{eq:decrease2}
\end{align}
\begin{align}
\|\vy-\vu(k')\|^2
&\leq  \left( 1- \frac{2\eta \hat{C}_1\|\vH^\infty\|  \lambda_0}{ ( \lambda_0+C\|\vH^\infty\|)b_{k'}}\right)\|\vy-\vu(k'-1)\|^2 \nonumber\\
&\quad  +\frac{\hat{C}\eta}{b_{k'}} \left(\frac{\eta}{b_{k'}}- \frac{2}{C(\lambda_0+\|\vH^\infty\|)}\right)\|\vH^\infty \left(\vy-\vu(k'-1)\right)\|^2  \label{eq:decrease3}
\end{align}

where the constants $C_1, \hat{C}_1, \widetilde{C}_1$, $C, \hat{C}, \widetilde{C}$ are well-defined with order $O(1)$. We list the expression of $C_1$ and $\hat{C}_1$ in Table 1 as examples. But notice that $C_1, \hat{C}_1, \widetilde{C}_1$ are less than one and  $C, \hat{C}, \widetilde{C}$ are greater than one. One might compare  $C\|\vH^{\infty}\|$ and $C_1\lambda_0$ in \eqref{eq:decrease1} respectively with $\bar{\lambda}_1$  and $\bar{\lambda}_n$ in Theorem \ref{thm:linearRegression}.
\end{condition}

\subsubsection{Sketch of The Proof}

We prove by induction.
Our induction hypothesis is Condition \ref{con:adagrad_sub}.
Recall the key Gram matrix $\mat H(k')$  at $k'$-th iteration 
We prove two cases   $b_0/\eta\geq\frac{C\left(\lambda_0+ \|\vH^\infty\|\right)}{2}$  and  $b_0/\eta\leq \frac{C\left(\lambda_0+ \|\vH^\infty\|\right)}{2}$ separately.
\paragraph{\textbf{Case (1)}: $b_0/\eta\geq \frac{C\left(\lambda_0+ \|\vH^\infty\|\right)}{2}$}  
The base case $k'=0$ holds by the definition.  Now suppose for $k'=0,\ldots,k$, Condition~\ref{con:adagrad_sub} holds and we want to show Condition~\ref{con:adagrad_sub} holds for $k'=k+1$.  
Because $b_0/\eta\geq C{\|\vH^{\infty}\|}$,  by Lemma \ref{lem: small_w_for_allbsq_sub} we have
\begin{align}
\|\vw_r(k)-\vw_r(0)\|
&\leq \frac{4 \sqrt{n} }{\sqrt{m} \lambda_0 C_1} \left( \frac{1}{2\lambda_0}+ \frac{  \lambda_0}{2\hat{C}_1\|\vH^\infty\|} \right)\norm{\vy-\vu(0)}.
 \label{eq:case1}
\end{align}
Next, plugging in $m =\Omega \left(\frac{n^8}{\lambda_0^4\delta^3} \right)$, we have $ \|\vw_r(k)-\vw_r(0)\|\leq \frac{c\lambda_0\delta }{n^3}$. Then by Lemma~\ref{lem:H0_close_Hinft} and           \ref{lem:close_to_init_small_perturbation}, the matrix  $\vH(k)$ is positive such that the smallest eigenvalue of $\vH(k)$ is greater than $\frac{\lambda_0}{2}$. Consequently, we have Condition~\ref{con:adagrad_sub} holds for $k'=k+1$. 

Now we have proved the induction part and Condition~\ref{con:adagrad_sub} holds for any $t \in \mathbb{Z}^+$. Since $b_0>\frac{C\left(\lambda_0+ \|\vH^\infty\|\right)}{2}$ and $b_t$ is increasing, then we have 
$  \frac{b_{t}}{\eta} \in \left[\frac{C\left(\lambda_0+ \|\vH^\infty\|\right)}{2}, C\|\vH^\infty\| \right]$, $\forall t < \widetilde{T}$ for some index $\widetilde{T}>0$. As a result,  $\forall t < \widetilde{T}$
\begin{align*}
\norm{\vy-\vu(\widetilde{T})}^2 
   &\leq \Pi_{t=0}^{\widetilde{T}-1}\left( 1-\frac{2\eta \hat{C}_1\|\vH^\infty\|  \lambda_0}{ ( \lambda_0+C\|\vH^\infty\|)b_{t}}\right)\norm{\vy-\vu(0)}^2\\
  &\leq \exp\left( -\widetilde{T} \frac{\eta \lambda_0\hat{C}_1}{2\eta C^2\| \vH^\infty\|}  \right)\norm{\vy-\vu(0)}^2  \quad \text{ take } b_t = C\|\vH^\infty\|
\end{align*}
For $t>\widetilde{T}$, $b_{t} \geq \eta C\|\vH^{\infty}\|$, we have 
\begin{align}
\norm{\vy-\vu(T)}^2 
   &\leq \Pi_{t=\widetilde{T}}^{T}\left( 1-\frac{\eta \lambda_0C_1}{2 b_{t}}\right)\norm{\vy-\vu(\widetilde{T})}^2\\
  &\leq \exp\left( -(T-\widetilde{T}) \frac{\eta \lambda_0C_1}{2b_{\infty}}  \right)\norm{\vy-\vu(\widetilde{T})}^2\\
   &\leq \exp\left( -T D \right)\norm{\vy-\vu(0)}^2 \label{eq:case1-b0-contraction}
\end{align}
where $D = \max\{ \frac{2 C^2\| \vH^\infty\|}{ \lambda_0C_1}, \frac{b_\infty}{\eta \lambda_0 C_1}\} $ and $ b_\infty = b_0 + \alpha^2 \sqrt{n}\frac{ ( \lambda_0+C\|\vH^\infty\|)}{2\eta\hat{C}_1\|\vH^\infty\|  \lambda_0}\norm{\vy-\vu(0)}^2$ (c.f. Lemma~\ref{lem: small_w_for_allbsq_sub}).
This implies the convergence rate of \textbf{Case (1)}.

\paragraph{\textbf{Case (2)}: $b_0/\eta\geq\frac{C\left(\lambda_0+ \|\vH^\infty\|\right)}{2}$}  
We define 
\[
\hat{T} = \arg\min_{k} \frac{b_k}{\eta} \ge \frac{C\left(\lambda_0+ \|\vH^\infty\|\right)}{2}.
\]
Note this represents the number of iterations to make \textbf{Case (2)} reduce to \textbf{Case (1)}.
We first give an upper bound $T_0$ of $\hat{T}$.  Applying Lemma \ref{lem:increase_sub1} with $\hat{\Delta}:=[\vV\vy]_1-[\vV\vu(t)]_1$, the dimension that corresponds to the largest eigenvalue of $\vH^{\infty}$, we have the upper bound 
\[
T_0 = O\left( \frac{\lambda_0}{\alpha^2\sqrt{n} \epsilon} +\frac{1}{\alpha^2 \epsilon}\right) + 2 \log  \left(  1+
\frac{ \eta ^2\left({C\|\vH^\infty\|}+\lambda_0 \right)^2 - 4(b_0)^2 }{\eta^2 \bar{\lambda}_1^2 }
\right) / \log \left (1+   \frac{ 4\hat{\Delta}}{ \eta^2 \|\vH^\infty\| }\right)
\]
We have after $T_0$ step, 
\begin{align*}
\text{either } \quad \min_{k\in [{T}_0]} \|\vy-\vu(k)\|^2 \leq  \varepsilon, \quad  \text{or}\quad  \quad b_{{T}_0+1} \ge  \frac{\eta C( \lambda_0+\|\vH^{\infty} \|)}{2} .
\end{align*}
If  $\min_{k\in[T_0]}\norm{\vy-\vu(k)}^2\leq \varepsilon$, we are done. 
Note this bound $T_0$ incurs the first term of iteration complexity of the \textbf{Case (2)} in Theorem~\ref{thm:main_adagradloss}.

Similar to \textbf{Case (1)}, we use induction for the proof.
Again the base case $k'=0$ holds by the definition.  Now suppose for $k'=0,\ldots,k$, Condition~\ref{con:adagrad_sub} holds and we will show it also holds for $k'=k+1$. 
There are two scenarios.

For $k\leq T_0$, Lemma \ref{lem: small_w_for_allbsq1_sub} implies that $\|\vw_r(k)-\vw_r(0)\| $ is upper bounded.
Now plugging in our choice on $m$ and using Lemma~\ref{lem:H0_close_Hinft} and~\ref{lem:close_to_init_small_perturbation}, we know $\lambda_{\min}\left(\vH(k)\right)\ge \lambda_0/2$ and $\norm{\mat{H}(k)} \le C\norm{\mat H^\infty}$.
These two bounds on $\mat H(k)$ imply Condition~\ref{con:adagrad_sub}.

When  $k\geq T_0$, we have  contraction bound as in \textbf{Case (1)} and then same argument  follows but with the different initial values $\vW(T_0)$ and $\|\vy-\vu (T_0)\|$. 
We first analyze $\vW(T_0)$ and $\|\vy-\vu (T_0)\|$.
By Lemma~\ref{lem:log_growth_sub}, we know  $\|\vy-\vu (T_0)\|^2$  only increases an additive $ O\left(\left(  \eta C \| \vH^{\infty}\| \right)^2\right)$ factor from $\|\vy-\vu (0)\|^2$.
Furthermore, by Lemma~\ref{lem: small_w_for_allbsq1_sub}, we know for $r \in [m]$\begin{align*}
\norm{\vect{w}_r(T_0) - \vect{w}_r(0)} \le \frac{4\eta^2 C\norm{\mat{H}^\infty}}{\alpha^2\sqrt{m}}.
\end{align*}
Now we consider $k$-th iteration.
Applying Lemma~\ref{lem: small_w_for_allbsq_sub}, we have 
\begin{align*}
\|\vw_r(k)-\vw_r(0)\|
\leq &\|\vw_r(k)-\vw_r(T_0)\|+ \|\vw_r(T_0)-\vw_r(0)\|\\
\leq & \frac{4 \sqrt{n} }{\sqrt{m} \lambda_0 C_1} \left( 
\norm{\vy-\vu(T_0)} +\hat{R}
 \right)\\
 \leq  &{c\lambda_0\delta } /{n^3}
\end{align*}
where the last inequality we have used our choice of $m$.
Using Lemma~\ref{lem:H0_close_Hinft} and~\ref{lem:close_to_init_small_perturbation} again, we can show  $\lambda_{\min}\left(\vH(k)\right)\ge \lambda_0/2$ and $\norm{\mat{H}(k)} \le C\norm{\mat H^\infty}$.
These two bounds on $\mat H(k)$ imply Condition~\ref{con:adagrad_sub}.

Now we have proved the induction.
The last step is to use Condition~\ref{con:adagrad_sub} to prove the convergence rate.
Observe that for any $T \ge T_0$, we have 
\begin{align*}
& \norm{\vy-\vu(T-1)}^2 
\leq   \exp\left( - (T-T_0) \frac{\eta \lambda_0C_1}{2\widetilde{D}}  \right) \norm{\vy-\vu(T_0)}^2
\end{align*}
where we have used Lemma~\ref{lem: small_w_for_allbsq_sub} and Lemma~\ref{lem:log_growth_sub} to derive\begin{align*}
&\widetilde{D} = \max \left\{\bar{b}_\infty,\frac{2 C^2\| \vH^\infty\|}{ \lambda_0C_1}    \right\}\\
\bar{b}_\infty=  \eta C\|\vH^{\infty}\|+&\frac{4\alpha^2 \sqrt{n}}{\eta {\lambda_0}C_1}\left( \norm{\vy-\vu(0)}^2  + \frac{  2\eta^2 (C\|\vH^\infty\|)^2}{ \alpha^2  \sqrt{n}} \log \left(\frac{ \eta C\|\vH^{\infty}\|}{b_0}  \right) \right).
\end{align*}
With some algebra, one can show this bound corresponds to the second and the third term of iteration complexity of the \textbf{Case (2)} in Theorem~\ref{thm:main_adagradloss}.


\subsection{Lemmas for Theorem \ref{thm:main_adagradloss}}
\begin{lem}\label{lem: small_w_for_allbsq1_sub}
Let $L =  \frac{C\left(\lambda_0+ \|\vH^\infty\|\right)}{2} $  and $T_0\geq 1$ be the first index such that $b_{T_0+1} \geq  L $ and  $b_{T_0} <  L $.  Consider AdaLoss: $b_{k+1}^2 = b_{k}^2 +  \alpha^2 \sqrt{n}\|\vy-\vu(k) \| ^2$.
Then for every $r \in [m]$, we have for all $k<T_0,$
\begin{align*}
\|\vw_r(T_0)-\vw_r(0)\|_2
&\leq 
  \frac{\eta\sqrt{2(T_0+1)}}{\alpha \sqrt{m}}\sqrt{  1+ 2\log \left( \frac{C\left(\lambda_0+ \|\vH^\infty\|\right)}{2b_0} \right)}\triangleq \hat{R}.
\end{align*}
\end{lem}

\begin{proof}
For the upper bound of $\|\vw_r(k+1)-\vw_r(0)\|_2$ when $ t = 0,1, \cdots, k$ and $k\leq T_0-2$, we have
\begin{align*}
 \sum_{t=0}^{k}  \frac{\norm{\vy-\vu(t)}_2^2}{b_{t+1}^2}  
 &\leq   \frac{1}{\alpha^2 }\sum_{t=0}^{k}  \frac{\alpha^2 \sqrt{n}  \norm{\vy-\vu(t)}_2^2/b_0^2 }{\alpha^2 \sqrt{n} \sum_{\ell=0}^{t} \| \vy-\vu(\ell)\|_2^2/b_0^2 +1} \nonumber \\
&\leq    \frac{1}{\alpha^2 \sqrt{n}  } \left( 1+ \log \left( \alpha^2 \sqrt{n}\sum_{t=0}^{k} \| \vy-\vu(t)\|_2^2/b_0^2 +1\right)\right) \\
&\leq    \frac{1}{\alpha^2 \sqrt{n}  }\left( 1+ 2\log \left(b_{T_0}/b_0\right)\right) 
\end{align*}
where the second inequality use Lemma 6 in \cite{ward2018adagrad}.Thus
\begin{align*}
\|\vw_r(k+1)-\vw_r(0)\|_2 
&\leq  \frac{\eta \sqrt{n}}{\sqrt{m}}\sqrt{(k+1) \sum_{t=0}^{k}  \frac{\norm{\vy-\vu(t)}_2^2}{b_{t+1}^2}  } \\
&\leq   \frac{\eta\sqrt{2(k+1)}}{\alpha\sqrt{m}}\sqrt{  1+ 2\log \left( \frac{C\left(\lambda_0+ \|\vH^\infty\|\right)}{2b_0} \right)} . 
\end{align*}
\end{proof}

\begin{lem}
\label{lem:increase_sub1}  \textbf{ (Exponential Increase for  $b_0 < \frac{ \eta \|\vH^\infty\|}{2}-\epsilon$)}
Write the eigen-decomposition $\mat H^\infty = \vU \vect{\Sigma} \vV$, where $\vect v_1, \ldots, \vect v_n \in \R^n$ are orthonormal eigenvectors of $\mat H^\infty$ and $\lambda_1\geq\ldots\geq\lambda_n = \lambda_0$ are corresponding eigenvalues. Suppose we start with small initialization: $0<b_0< \eta \|\vH^{\infty}\|/2$. Let $\hat{\Delta}:=[\vV\vy]_1-[\vV\vu(0)]_1$, the dimension that corresponds to the largest eigenvalue of $\vH^{\infty}$.  Then after 
\begin{align*}
\hat{T}=\left\lceil2 \log  \left(  1+
\frac{\left(\eta\|\vH^\infty\|\right)^2 - 4(b_0)^2 }{\eta^2 \|\vH^\infty\|^2}
\right) / \log \left (1+   \frac{ 4\hat{\Delta}}{ \eta^2 \|\vH^\infty\| }\right)  \right\rceil +O\left( \frac{\lambda_0+\sqrt{n}}{\alpha^2\sqrt{n} \epsilon}\right)
\end{align*}
 either  $b_{\hat{T}+1}\geq \left(\|\vH^\infty\|+\lambda_n\right)/2$ or  $\|\vy-\vu(\hat{T})\|^2\leq \epsilon.$
\end{lem}
\textbf{Sketch proof of Lemma \ref{lem:increase_sub1}}
From Lemma  \ref{thm:convergence_rate}, we can explicitly express the dynamics of $\{[\vV\vy]_i-[\vV\vu(t)]_i\}_{t>0}$  for each dimension $i$. By Theorem \ref{thm:main_gd} we know $\{[\vV\vy]_1-[\vV\vu(t)]_1\}_{t>0}$  is monotone increasing up to $\frac{b_{t}}{\eta} \leq \frac{ \|\vH^\infty\|}{2}-\epsilon$.  Then, we could follows the exactly the same argument as in Lemma \ref{lem:increase_sub} and have the first term. However, for iteration $t\in \left\{t,  \frac{ \|\vH^\infty\|}{2}-\epsilon\leq \frac{b_{t}}{\eta}\leq \frac{ (C\|\vH^\infty\|+\lambda_0)}{2} \right\} $, the dynamic of $s_t = [\vV\vy]_1-[\vV\vu(t)]_1$ is not clear. For this case, we use the worse case analysis Lemma \ref{lem:increase_sub_a}. Thus, after 
\begin{align}
\hat{T} = \left\lceil2 \log  \left(  1+
\frac{\left(\eta\|\vH^\infty\|\right)^2 - 4(b_0)^2 }{\eta^2 \|\vH^\infty\|^2}
\right) / \log \left (1+   \frac{ 4\hat{\Delta}}{ \eta^2 \|\vH^\infty\| }\right)+ \left(\frac{\lambda_0+(C-1)\|\vH^\infty\|+\epsilon}{2\alpha^2 \sqrt{n}\epsilon}\right)  \right\rceil,\label{eq:eitheror-twolayer}    
\end{align}
 either  $b_{\hat{T}+1}\geq \left(\|\vH^\infty\|+\lambda_n\right)/2$ or  $\|\vy-\vu(\hat{T})\|^2\leq \epsilon.$ Note 
 that $C = O(1+\frac{1}{\sqrt{n}})$ (see Table 1 for the expression of $C$) and $\|\vH^\infty\| =O(n)$.
 We have  $$\hat{T}= O\left( \frac{\lambda_0+\sqrt{n}}{\alpha^2\sqrt{n} \epsilon}\right)$$

\begin{lem}\label{thm:convergence_rate}\cite{arora2019fine}
Write the eigen-decomposition $\mat H^\infty = \vU \vSi\vV$ where $\vect{v}_1, \ldots, \vect{v}_n \in \R^n$ are orthonormal eigenvectors of $\mat H^\infty$.
Suppose $\lambda_0 = \lambda_{\min}(\mat H^\infty) >0$,
$\kappa = O\left( \frac{\epsilon \delta}{\sqrt n} \right)$,
 $m = \Omega\left( \frac{n^8}{\lambda_0^4 \kappa^2 \delta^4 \epsilon^2} \right)$ and $\eta = O\left( \frac{\lambda_0}{n^2} \right)$. 
Then with probability at least $1-\delta$ over the random initialization, for all $k=0, 1, 2, \ldots$ we have: 
$$ \label{eqn:u(k)-y_size}
\norm{\vect{y}-\vect{u}(k)}_2
= \sqrt{\sum_{i=1}^{n}(1-\eta\lambda_i)^{2k} \left(\vect{v}_i^\top \vect{y}\right)^2} \pm \epsilon.
$$
\end{lem}

 \begin{lem}
\label{lem: small_w_for_allbsq_sub}
Suppose Condition~\ref{con:adagrad_sub} holds for $k'=0,\ldots,k$ and   $ b_k$  is updated by Algorithm 1. Let $T_0\geq 1$ be the first index such that $b_{T_0+1} \geq  \eta \frac{C\left(\lambda_0+ \|\vH^\infty\|\right)}{2} $.  Then for every $r \in [m]$ and $k = 0,1,\cdots$,
\begin{align*}
 & b_{k}\leq b_{T_0}+\alpha^2\sqrt {n}\frac{ ( \lambda_0+C\|\vH^\infty\|)}{2\eta \hat{C}_1\|\vH^\infty\|  \lambda_0}\norm{\vy-\vu(T_0)}^2 ;\\
\|\vw_r(k+T_0)-&\vw_r(T_0)\| 
\leq \frac{4\sqrt{n}}{\sqrt{m}}\frac{ ( \lambda_0+C\|\vH^\infty\|)}{2\hat{C}_1\|\vH^\infty\|  \lambda_0}\norm{\vy-\vu(T_0)}  \triangleq \widetilde{R} .
\end{align*}
\end{lem}

\textbf{Proof of Lemma \ref{lem: small_w_for_allbsq_sub}}
When $b_{T_0}/\eta \geq \frac{C\left(\lambda_0+ \|\vH^\infty\|\right)}{2}$ at some $T_0\geq 1$, thanks to the key fact that Condition~\ref{con:adagrad_sub} holds $k'=0,\ldots,k$, we have from inequality \eqref{eq:decrease3}
\begin{align}
\|\vy-\vu(t+T_0+1)\|^2
&\leq  \|\vy-\vu(t+T_0)\|^2 - \frac{2\eta \hat{C}_1\|\vH^\infty\|  \lambda_0}{ ( \lambda_0+C\|\vH^\infty\|)b_{t+T_0+1}}\|\vy-\vu(t+T_0)\|^2 \nonumber\\
 & \leq \norm{\vy-\vu(T_0)}^2-\sum_{\ell=0}^{t}\frac{2\eta \hat{C}_1\|\vH^\infty\|  \lambda_0}{ ( \lambda_0+C\|\vH^\infty\|)b_{\ell+T_0+1}}\norm{\vy-\vu(\ell+T_0)}^2 \label{eq:repeat}\\
 \Rightarrow  \quad \frac{2\eta \hat{C}_1\|\vH^\infty\|  \lambda_0}{ ( \lambda_0+C\|\vH^\infty\|)}& \sum_{\ell=0}^{t}\frac{\norm{\vy-\vu(\ell+T_0)}^2}{b_{\ell+1+T_0}}  \leq \norm{\vy-\vu(T_0)}^2 . 
 \end{align}
Thus, the upper bound for $b_k$,
\begin{align*}
b_{k+1+T_0} 
\leq& b_{k+T_0}  + \frac{\alpha^2\sqrt {n} }{b_{k+T_0}} \|\vy - \vu(k+T_0)\|^2 \\
\leq&b_{T_0} +\sum_{t=0}^{k}\frac{\alpha^2\sqrt {n} }{b_{t+T_0}} \|\vy - \vu(t+T_0)\|^2 \\
\leq &b_{T_0} + 
\alpha^2\sqrt {n}\frac{ ( \lambda_0+C\|\vH^\infty\|)}{2\eta \hat{C}_1\|\vH^\infty\|  \lambda_0}\norm{\vy-\vu(T_0)}^2.
\end{align*}
As for the  upper bound of $\|\vw_r(k+T_0)-\vw_r(T_0-1)\|,$ we repeat \eqref{eq:repeat} but without square in the norm. That is
\begin{align*}
\|\vy-\vu(t+T_0+1)\|
&\leq  \|\vy-\vu(t+T_0)\| - \frac{\eta \hat{C}_1\|\vH^\infty\|  \lambda_0}{ ( \lambda_0+C\|\vH^\infty\|)b_{t+T_0+1}}\|\vy-\vu(t+T_0)\| \nonumber\\
 & \leq \norm{\vy-\vu(T_0)}-\sum_{\ell=0}^{t}\frac{2\eta \hat{C}_1\|\vH^\infty\|  \lambda_0}{ ( \lambda_0+C\|\vH^\infty\|)b_{\ell+T_0+1}}\norm{\vy-\vu(\ell+T_0)} \\
 \Rightarrow  \quad \frac{2\eta \hat{C}_1\|\vH^\infty\|  \lambda_0}{ ( \lambda_0+C\|\vH^\infty\|)}& \sum_{\ell=0}^{t}\frac{\norm{\vy-\vu(\ell+T_0)}}{b_{\ell+1+T_0}}  \leq \norm{\vy-\vu(T_0)} .
 \end{align*}
Thus we have
\begin{align*}
 \|\vw_r(k+T_0)-\vw_r(T_0-1)\|
\leq& \sum_{t=0}^{k}\frac{\eta }{b_{t+T_0}} \| \frac{\partial L(\vW{(t+T_0-1)})}{\partial \vw_r } \| \\
 \leq& \frac{4\sqrt{n}}{\sqrt{m}}\frac{ ( \lambda_0+C\|\vH^\infty\|)}{2\eta \hat{C}_1\|\vH^\infty\|  \lambda_0}\norm{\vy-\vu(T_0)}.  
\end{align*}


%
\begin{lem}
\label{lem:log_growth_sub}
Let $T_0\geq 1$ be the first index such that $b_{T_0+1} \geq  \frac{C\left(\lambda_0+ \|\vH^\infty\|\right)}{2} $ from $b_{T_0} <  \frac{C\left(\lambda_0+ \|\vH^\infty\|\right)}{2} $. Suppose Condition \ref{con:adagrad_sub} holds for $k'=0,1,\ldots,k$. Then
\begin{align*}
\norm{\vy-\vu(T_0)}^2
 \leq & \norm{\vy-\vu(0)}^2 +  \frac{  2\eta^2C{\|\vH^\infty\|}^2}{ \alpha^2  \sqrt{n}} \log \left(\frac{C\left(\lambda_0+ \|\vH^\infty\|\right)}{2b_0}   \right).
\end{align*}
\end{lem}
\textbf{Proof of Lemma \ref{lem:log_growth_sub}} Using the inequality \eqref{eq:decrease1}
\begin{align}
\norm{\vy-\vu(T_0)}^2  &\leq   \norm{\vy-\vu(T_0-1)}^2+\frac{\eta^2 \left(C{\|\vH^\infty\|}\right)^2 }{b^2_{T_0}} \norm{\vy-\vu(T_0-1)}^2 \nonumber\\
    &\leq \norm{\vy-\vu(0)}^2 + \eta^2 \left(C{\|\vH^\infty\|}\right)^2\sum_{t=0}^{T_0}\frac{ \norm{\vy-\vu(t)}^2  }{b^2_{t+1}}  \nonumber \\
        &\leq \norm{\vy-\vu(0)}^2 +  \frac{  2\eta^2 \left(C\|\vH^\infty\|\right)^2}{ \alpha^2  \sqrt{n}} \log \left(\frac{C\left(\lambda_0+ \|\vH^\infty\|\right)}{2b_0}  \right) \nonumber
\end{align}

%% file: basic.tex
\begin{lem}
\label{lem:increase_sub_a}
Fix $\varepsilon \in (0,1]$, $L > 0$, $\gamma>0$.   For any non-negative $a_0, a_1, \dots, $ the dynamical system
$$
b_0 > 0; \quad  \quad b_{j+1}^2 = b_j^2 +  \gamma a_j
$$
has the property that after 
{$N = \lceil{ \frac{ L^2-b_0^2}{ \gamma\sqrt{\varepsilon}} \rceil}+1$}
iterations, either $\min_{k=0:N-1}  a_k  \leq  \sqrt{\varepsilon}$, or $b_{N} \geq  L$. 
\end{lem}

%

\begin{prop} \label{prop:a1}
Under Assumption \ref{asmp:norm1} and Assumption \ref{asmp:lambda_0} and $\lambda_{min} (\vH)\geq \frac{\lambda_0}{2}$, we have
$$\frac{\sqrt{\lambda_0}}{ \sqrt{2m}} \norm{\vy-\vu
	 }_2\leq \max_{r\in[m]}\norm{\frac{\partial L(\vW{})}{\partial \vw_r}}_{2} 
	\le \frac{\sqrt{n}}{\sqrt{m}} \norm{\vy-\vu
	 }_2 .$$
\end{prop}

\begin{proof}
For $a_r \sim \text{unif}(\{-1,1\})$ , we have
  \begin{align}
\max_{r\in[m]}\|\frac{\partial L(\vW)}{\partial \vw_r}\|^2
	    & =  \frac{1}{m}\max_{r\in[m]}\norm{
	\sum_{i=1}^n(y_i-u_i)a_r\vx_i\mathbb{I}_{\{\vw_r^\top \vx_i \ge 0\} } }^2
	 \nonumber\\
&= \frac{1}{m} \max_{r\in[m]}\left( \sum_{i,j}^n (u_i-y_i)(u_j-y_j)\langle{\vx_i,\vx_j\rangle} \mathbb{I}_{\{\vw_r^T\vx_i\geq0,\vw_r^T\vx_j\geq0\}}\right) \nonumber\\
&\geq \frac{1}{m} \left( \sum_{i,j}^n (u_i-y_i)(u_j-y_j)\langle{\vx_i,\vx_j\rangle} \frac{1}{m} \sum_{r=1}^m\mathbb{I}_{\{\vw_r^T\vx_i\geq0,\vw_r^T\vx_j\geq0\}}\right) \nonumber\\
&=  \frac{1}{m}(\vu-\vy)^\top\vH(\vu-\vy) 
  \end{align}
where the last inequality use the condition that 
$
    \lambda_{\min}(\vH)\geq\frac{\lambda_0}{2}. 
$
As for the upper bound, we have
	 $$\max_{r\in[m]}\norm{\frac{\partial L(\vW{})}{\partial \vw_r}}
	\le \frac{1}{\sqrt{m}}\sqrt{\sum_{i=1}^{n} |y_i-u_i|^2 }\sqrt{\sum_{i=1}^{n} \|\vx_i\|^2 }  \le \frac{\sqrt{n}}{\sqrt{m}}\norm{\vy-\vu
	 }$$
\end{proof}

Observe that at initialization, we have following proposition
\begin{prop}  \label{prop:a2}
Under Assumption \ref{asmp:norm1} and  \ref{asmp:lambda_0}, with probability $1-\delta$ over the random initialization,
\begin{align*}
  \norm{\vect{y}-\vect{u}(0)}^2 \leq \frac{n}{\delta}.
 \end{align*} 
\end{prop}
We get above statement by Markov's Inequality  with following
\begin{align*}
& \expect\left[\norm{\vect{y}-\vect{u}(0)}^2\right]
 =\sum_{i=1}^{n} (y_i^2 + 2y_i \expect \left[f(\mat{W}(0),\vect{a},\vect{x}_i)\right] + \expect\left[f^2(\mat{W}(0),\vect{a},\vect{x}_i)\right]) 
 =\sum_{i=1}^n (y_i^2 + 1) = O(n).
\end{align*}
Finally, we analyze  the upper bound of the maximum eigenvalues of Gram matrix that plays the most crucial role in our analysis.  Observe that
\begin{align*}
\norm{\mat{H}^{\infty}} 
 &=\sup_{\|\vv\|_2=1} \sum_{i,j} v_i v_j\langle{ \vx_i,\vx_j \rangle} \frac{1}{m} \sum_{r=1}^m \indict_{\left\{\vect{w}_r(0)^\top \vect{x}_{i} \ge 0, \vect{w}_r(0)^\top \vect{x}_{j}\ge 0 \right\}} \leq\sqrt{ \sum_{i\neq j}|\langle{ \vx_i,\vx_j \rangle}|^2} +1
\end{align*}
If the data points are pairwise uncorrelated (orthogonal), i.e., $\langle{\vx_i,\vx_j\rangle} = 0, i\neq j$, then  the maximum eigenvalues  is close to 1, i.e., $ \norm{\mat{H}^{\infty}} \leq 1$.  In contrast, we could have  
$\norm{\mat{H}^{\infty}}\le n$ if data points are pairwise highly correlated (parallel), i.e., $\langle{\vx_i,\vx_j\rangle} = 1, i\neq j$.

\input{table.tex}

%% file: table.tex
\begin{table}[H]
\caption{Some notations of parameters to facilitate understanding  the proofs in Appendix B and C}
\label{table}
\centering
\begin{tabular}{ c|c |l }
\hline \toprule
 {Expression}&{Order} &{First Appear} \\ \hline 
 \toprule
 $ s\leq 2 \frac{ \log  \left(  1+
\frac{ \left({\eta \lambda_1} \right)^2 - 4(b_0)^2 }{\eta^2 }
\right) }{ \log \left (1+   \frac{ 4}{ \eta^2 } \left([\hat{\vw}_{0}]_1 -[\hat{\vw}^{*}]_1 \right)^2 \right)}$ & $ O\left( 1\right)$    & {Theorem \ref{thm:linearRegression}
} 
       \\
        \midrule
         $ \widetilde{s}\leq 2 \frac{ \log  \left(  1+
\frac{ \left({\eta \lambda_1} \right)^2 - 4(b_0)^2 }{\eta^2 \lambda_1^2 }
\right) }{ \log \left (1+   \frac{ 4}{ \eta^2 \lambda_1 } \left([\hat{\vw}_{0}]_1 -[\hat{\vw}^{*}]_1 \right)^2 \right)}$ & $ O\left( 1\right)$    & {Corollary \ref{cor:linearRegression2}
} 
       \\
        \midrule
  $\delta_s = \frac{\eta^2 \left(\bar{\lambda}_1+\lambda_n\right)^4\left( 1+\frac{b_0}{ ([V^\top {\vw}_{t}]_i -[V^\top {\vw}^{*}]_i  )^2  } \right)} 
  {\lambda_1 \left(\bar{\lambda}_1-\lambda_n\right)^2 \left( \eta\lambda_1-\sqrt{b_0^2+([V^\top {\vw}_{t}]_i -[V^\top {\vw}^{*}]_i  )^2 } \right)^2 } $ & $ O\left( 1\right)$    & {Theorem \ref{thm:linearRegression}
} 
       \\
       \midrule
 $c$   is a small value, say less than $0.1$ & $ O\left( 1\right)$    & {Lemma \ref{lem:close_to_init_small_perturbation}
} 
       \\
       \midrule
$R=\frac{c\lambda_0\delta }{n^3}  $& $ O\left( \frac{\lambda_0\delta}{n^3}\right)$        & {Lemma \ref{lem:close_to_init_small_perturbation}
} 
       \\
       \midrule
$R'= \frac{4\sqrt{n} \norm{\vect{y}-\vect{u}(0)}}{ \sqrt{m} \lambda_0} $& $O\left( \frac{n}{\sqrt{m\delta} \lambda_0}\right) $         &{Lemma \ref{cor:dist_from_init}}
 \\   
\midrule
 $\delta_1 = \left(\frac{n^{5/2}R }{2\delta}+\frac{2 n^2R }{\delta}+ \frac{2 n^{3/2} R}{\delta}  \right)\frac{n^{3/2} R}{\delta} $  &   $ \Theta\left(1\right)$     &{Equation \eqref{eq:keyforadagrad1}, Condition \ref{con:adagrad_sub}}      
  \\ 
  \midrule
 $C_1 = 1-   \frac{n^{3/2}\left( \sqrt{n}+2 \right)R }{2\lambda_0 \delta}  $& $   O\left(1-\frac{1}{n}\right)$       &{Equation \eqref{eq:keyforadagrad}, Condition \ref{con:adagrad_sub}}      
  \\ 
  \midrule
$C=\frac{ { \left(\lambda_0  \delta +2 n^{3/2}R+ n^2R\right)\left( \|\vH^{\infty}\| + \frac{4n^2R}{\sqrt{2\pi}\delta} \right)}/{ \lambda_0  \delta} +\delta_1 }{ C_1 \| \vH^{\infty}\|}$& $  O\left(1+\frac{1}{\sqrt{n}}\right)$        &{Equation \eqref{eq:keyforadagrad}, Condition \ref{con:adagrad_sub}}      \\
\hline
\end{tabular}
\end{table}


%% file: exp-details.tex
\paragraph{Linear regression in Figure \ref{fig:adaloss1}}
We simulate $\vX\in \mathbb{R}^{1000\times20}$ and $\vw^{*}\in \mathbb{R}^{20}$ such that each entry of $\vX$ and $\vw^{*}$ is an i.i.d. standard Gaussian random variable and $\vy=\vX\vw^*$, i.e. the noiseless case. Our goal is to optimize the least square loss $F(\vw) = \frac{1}{2n} \|\vX\vw - \vy\|^2$ with different stepsize schedules: AdaLoss (updates as $b^2_{t+1}= b^2_{t}+\|\vX\vw_t-\vy\|^2$) for deterministic setting and (\ref{alg:linear_adas}) for stochastic setting), AdaGrad-Norm (updates in (\ref{eq:adalinear}) for deterministic setting and (\ref{alg:linear_adas}) for stochastic setting), (stochastic) GD-Constant with $\eta_t = \frac{1}{b_0}$ and (stochastic) GD-DecaySqrt with $\eta_t = \frac{1}{b_0 + 0.2\sqrt{t}}$. The plots show the change of the error $\|\vw_t - \vw^*\|^2$.

\paragraph{Linear regression in Figure \ref{fig:adaloss2}}

 We simulate $\vX\in \mathbb{R}^{1000\times2000}$ and $\vw^{*}\in \mathbb{R}^{2000}$ such that each entry of $\vX$ and $\vw^{*}$ is an i.i.d. standard Gaussian. Let $\vy=\vX\vw^*$. Note that we do not use one sample  for each iteration but use  a small sample of size $n_b=20$ independently drawn from the whole data sets. The loss function can be expressed as
$$
F(\vw) = \frac{1}{2m} \| \vX\vw - \vy \|^2 = \frac{1}{(m/n_b)} \sum_{\ell=1}^{m/n_b}\frac{1}{2n_b} \| \vX_{\xi_\ell}\vw - \vy_{\xi_\ell} \|^2, \quad \vX_{\xi_\ell} \in \mathbb{R}^{20\times 2000}
$$
We can think of the small sample of size $n_b=20$ is a one mini-batched sample. 
The vector $\vw_0$, whose entries follow i.i.d. uniform in $[0,1]$, is the same for all the methods so as to eliminate the effect of random initialization in the weight vector. \textbf{AdaLoss} algorithm can be expressed as
\begin{align*}
\vw_{j+1} = \vw_{j} -  \frac{\eta \vX_{\xi_j}^\top\left(\vX_{\xi_j}\vw_j  - \vy_{\xi_j} \right)/n_b}{\sqrt{b_0^2+\sum_{\ell=0}^j \left(\|\left(\vX_{\xi_\ell}\vw_\ell  - \vy_{\xi_\ell} \right)\|/n_b \right)^2}}.
\end{align*}
 We vary the initialization $b_0 > 0$ as to compare with (a) AdaGrad-Norm 
\begin{align*}
\vw_{j+1} = \vw_{j} -  \frac{\eta \vX_{\xi_j}^{T}\left(\vX_{\xi_j}\vw_j  - \vy_{\xi_j} \right)/n_b}{\sqrt{b_0^2+\sum_{\ell=0}^j \left(\|\vX_{\xi_\ell}^{T}\left(\vX_{\xi_\ell}\vw_\ell  - \vy_{\xi_\ell} \right)\|/n_b \right)^2}},
\end{align*}
 and plain SGD using (b) SGD-Constant: fixed stepsize $\frac{1}{b_0}$, (c) SGD-DecaySqrt: decaying stepsize  $\eta_j = \frac{1}{b_0\sqrt{j}}$. Figure \ref{fig:adaloss} (right 6 figures) plot $\sum_{j=T-100}^{T}\|\left(\vX_{\xi_j}\vw_j  - \vy\right) \|^2/100$ (loss) and  the effective learning rates  at iterations $T=200$, $1000$, and $5000$, and as a function of $b_0$, for each of the four methods. The effective learning rates are $\frac{1}{b_j}$ (AdaGrad-Norm), $\frac{1}{b_0}$ (SGD-Constant), $\frac{1}{b_0\sqrt{j}}$ (SGD-DecaySqrt).

 \paragraph{Two-layer Neural Network in Figure \ref{fig:adaloss}}

 We simulate $\vX\in \mathbb{R}^{100\times1000}$ and $\vy\in \mathbb{R}^{1000}$ such that each entry of $\vX$ and $\vy$ is an i.i.d. standard Gaussian. The hidden layer $m=1000$. Note that we do not use one sample  for each iteration but use  a small sample of size $n_b=20$ independently drawn from the whole data sets. The expression can be expressed as
$$
L_{\xi_k}(\vW) = \frac{1}{2n_b}\sum_{i\in \mathcal{B} (\xi_k)}\left(\frac{1}{\sqrt{m}}\sum_{r=1}^{m}a_r
     \sigma(\langle{\vw_r(k), \vx_{i} \rangle})- \vy_{i} \right)^2  \text{ with the carnality } |\mathcal{B} (\xi_k)|=n_b= 20
$$
We vary the initialization $b_0 > 0$  for AdaLoss with $\eta = 100$ as to compare with (a) AdaGrad-Norm and plain SGD using (b) SGD-Constant: fixed stepsize $\frac{100}{b_0}$, (c) SGD-DecaySqrt: decaying stepsize  $\eta_j = \frac{100}{b_0\sqrt{j}}$. Figure \ref{fig:adaloss} (left 6 figures) plot loss $\sum_{j=T-100}L_{\xi_j}(\vW)/100$ and  the effective learning rates  at iterations $T=200$, $1000$, and $5000$, and as a function of $b_0$, for each of the four methods. The effective learning rates are $\frac{100}{b_j}$ (AdaGrad-Norm), $\frac{100}{b_0}$ (SGD-Constant), $\frac{100}{b_0\sqrt{j}}$ (SGD-DecaySqrt). Note we set $\alpha = 1$ for the fair comparison. 

\paragraph{Experimental Details for Figure \ref{fig:adaloss-lstm} and Figure \ref{fig:adaloss-rl}}
The LSTM Classifier is well-suited to classify text sequences of various lengths.
The text data is Fake/Real News \cite{McIntire2018}, which contains  5336 training articles and 1000 testing articles, and the training set includes both fake and real news with a 1:1 ratio (not an imbalanced data). We modify the code from \cite{lei2019discrete} and construct a one-layer LSTM with 512 hidden nodes.

The actor-critic algorithm \cite{konda2000actor} is one of the popular algorithms to solve the classical control problem - inverted pendulum swingup.  According to the based code \footnote{\url{https://nbviewer.jupyter.org/github/MrSyee/pg-is-all-you-need/blob/master/02.PPO.ipynb}}, the model for the actor-network is ``two fully connected hidden layer with ReLU branched out two fully connected output layers for the mean and standard deviation of Gaussian distribution. The critic network  has three fully connected layers as two hidden layers (ReLU) and an output layer."  Both networks need to be optimized by gradient descent methods such as Adam and AdamLoss.

The algorithms are:\\
(1) \textbf{Adam}. We use the default setup in \cite{kingma2014adam} implemented in PyTorch (see Algorithm 2). We vary the parameter $\eta=\frac{1}{b_0}$ in the algorithm. \\
(2) \textbf{AdamLoss}. Based on Adam, we apply AdaLoss to the parameter $\eta$ and call it AdamLoss. The update of $\eta:=\eta_t$ becomes\vspace{-0.1cm} $\eta_t = \frac{1}{\sqrt{b_0^2+\alpha\sum_{\ell=0}^tf_{\xi_\ell}\left( \vw_\ell\right)}}$.\\
(3)  \textbf{AdamSqrt}. Based on Adam, we let the parameter $\eta$ decay sublinearly with respect to the iteration $t$ with the form $\eta_t = \frac{1}{\sqrt{b_0^2+ t}}.$

For the text classification, we vary the initialization $b_0$ to be $0.1$, $1$, $10$, $400$ and $1000$ and run the algorithms with the same initialized weight parameters.  For the control problem, we vary the initialization $b_0$ to be $0.01$, $0.1$, $10$, and $200$ and run the algorithms four times and plot the average performance. The results in Figure \ref{fig:adaloss-lstm} and Figure \ref{fig:adaloss-rl} use $\alpha=1$. See Table \ref{table2} and Figure \ref{fig} for $\alpha \in \{0.1,0.5, 5, 10\}$ .
 \begin{small}
\begin{table}
  \caption{Compare with different $\alpha$ and $\eta$ in the AdamLoss algorithm. The reported results are averaged from epoch 10 to 15. All experiments start with the same initialization of weight parameters.}\label{table2}
  \label{table:extra}
  \centering
  \begin{tabular}{lllllll}
    \toprule
    \small{\textbf{Tasks}} &\multicolumn{5}{c}{\small{\textbf{Text Classification (Test Error)}}} \\
   \cmidrule(r){1-1} \cmidrule(r){2-6}  
    \small{$b_0$} 
      &\small{  $0.1$} &\small{$1$} &\small{$10$} & \small{$400$ } & \small{$1000$ } \\
      \cmidrule(r){1-1} \cmidrule(r){2-6}

    \small{$\alpha=0$ \textbf{(Adam)}} 
    &\small{$0.401$} &\small{$0.264$} &\small{$0.168$} &\small{$0.079$} & \small{$0.060$}
\\
      \small{$\alpha=0.1$} &\small{$0.108$} &\small{$0.143$} &\small{$0.121$} &\small{$0.065$} &\small{$\mathbf{0.034}$} \\
  \small{$\alpha=1$}  &\small{$0.128$} &\small{$0.094$} &\small{$0.096$} &\small{$0.050$} &\small{$0.059$} \\      
    \small{$\alpha=5$}           &\small{$0.108$} &\small{$0.189$} &\small{$\mathbf{0.087}$} &\small{$\mathbf{0.042}$}&\small{$0.050$}
     \\
    \small{$\alpha=10$} 
     &\small{$\mathbf{0.078}$} &\small{$\mathbf{0.084}$} &\small{$0.121$} &\small{$0.053$} &\small{$0.045$}
\\

    \bottomrule
  \end{tabular}
\end{table}
\end{small}

\begin{figure*}[ht]\label{fig}
    \centering
\includegraphics[width=0.9\linewidth]{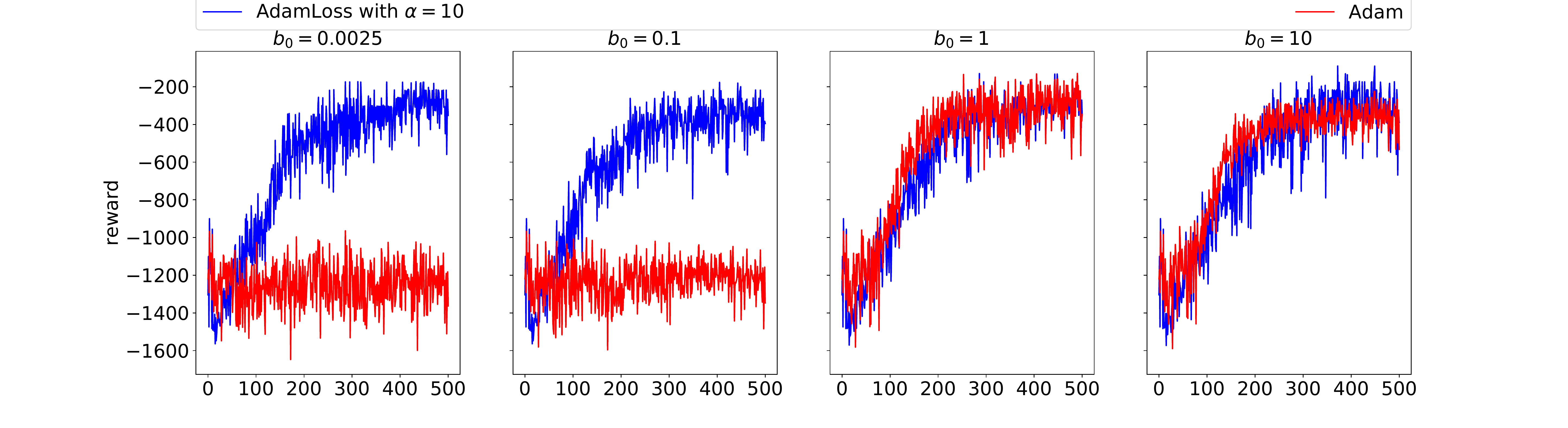}    
\includegraphics[width=0.9\linewidth]{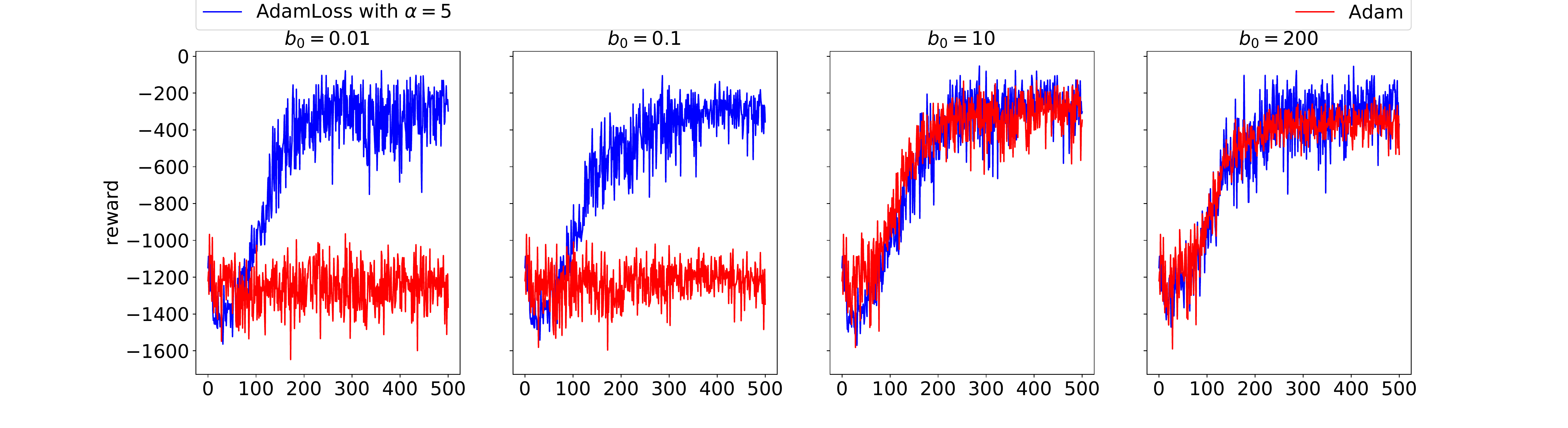}
\includegraphics[width=0.9\linewidth]{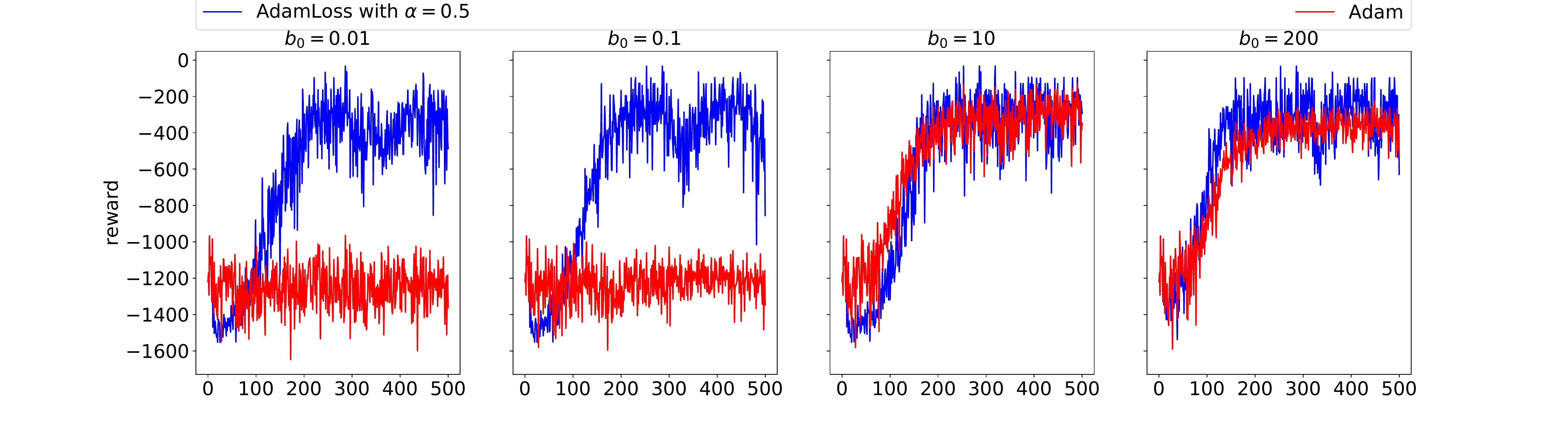}
\includegraphics[width=0.9\linewidth]{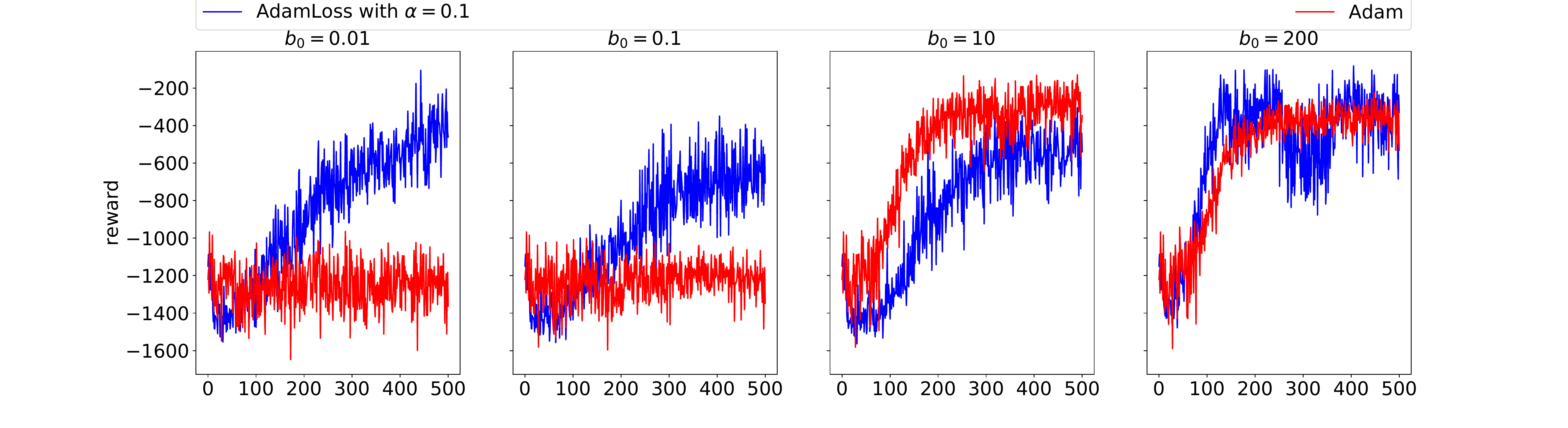}
      \caption{ Inverted Pendulum Swing-up  with Actor-critic Algorithm.  Each plot is the rewards (scores) w.r.t. number of frames $10000$ with roll-out length $2048$. The legend describes the value of $\alpha$: the 1st, 2nd, 3rd and 4th rows are respectively for $\alpha=10$, $\alpha=5$, $\alpha=0.5$ and $\alpha=0.1$.  }
      \vspace{-0.45cm}
\end{figure*}